\pdfoutput=1
\documentclass{informs4TD}
\usepackage{lmodern}
\usepackage{microtype}
\usepackage{makecell}
\usepackage{booktabs}
\usepackage{multirow}
\usepackage{hyperref}
\usepackage{xurl}
\usepackage[normalem]{ulem}
\OneAndAHalfSpacedXI

\usepackage{tablefootnote}
\usepackage[para,online,flushleft]{threeparttable}
\usepackage[top=1in, bottom=1in,left=1in,right=1in]{geometry}
\usepackage[nameinlink]{cleveref}
\crefname{assumption}{Assumption}{Assumptions}
\crefname{lemma}{Lemma}{Lemmas}
\crefname{corollary}{Corollary}{Corollaries}
\crefname{proposition}{Proposition}{Propositions}
\crefname{example}{Example}{Examples}
\crefname{remark}{Remark}{Remarks}
\crefname{definition}{Definition}{Definitions}
\crefname{figure}{Figure}{Figures}
\crefname{section}{Section}{Sections}
\crefname{appendix}{Appendix}{Appendices}
\crefname{table}{Table}{Tables}
\crefformat{equation}{(#2#1#3)}
\Crefformat{equation}{(#2#1#3)}
\usepackage{bbm}
\usepackage[normalem]{ulem}
\usepackage{tikz}
\usepackage{graphicx}
\usepackage{comment}
\usepackage{amsmath, amssymb}
\usepackage{enumitem}
\usepackage{epstopdf}
\usepackage{enumitem}
\usepackage{pdfpages} 
\usepackage{soul}
\usepackage{silence}
\usepackage{float}
\usepackage{placeins}
\usepackage{enumitem}
\usepackage[english]{babel}
\usepackage{natbib}
\usepackage{tabularx}
\usepackage{caption}
\usepackage{tikz}
\usepackage{booktabs, tabularx}
\usetikzlibrary{shapes,positioning,arrows.meta,decorations.pathreplacing,calc}
\usepackage{placeins}
\newtheorem{lemma}{Lemma}

\newtheorem{proposition}{Proposition}
\newtheorem{claim}{Claim}

\newtheorem{assumption}{Assumption}

\newtheorem{corollary}{Corollary}

\newtheorem{example}{Example}

\bibpunct[,]{(}{)}{,}{a}{}{,}%
\def\bibfont{\small}%
\def\bibsep{\smallskipamount}%

\definecolor{hopkins-blue}{RGB}{0,45,114}
\definecolor{columbia-blue}{RGB}{185, 217, 235}
\definecolor{chicago-maroon}{RGB}{128,0,0}
\definecolor{northwestern-purple}{RGB}{82,0,99}
\definecolor{cornell-red}{RGB}{179,27,27}
\definecolor{lawngreen}{RGB}{0,250,154}
\definecolor{gray}{RGB}{192,192,192}
\definecolor{AccentDeep}{RGB}{25, 55, 95}
\hypersetup{
  colorlinks=true,
  linkcolor=chicago-maroon,
  citecolor=chicago-maroon,
  urlcolor=chicago-maroon,
}

\usepackage{xcolor}

\ECRepeatTheorems

\MANUSCRIPTNO{}
	
	\RUNAUTHOR{Singh, Ghosh, and Dai}
	 \RUNTITLE{Skill Investment Under Improving AI and Worker Mobility}

\TITLE{Managing the Human Fallback: Skill Investment Under Improving AI and Worker Mobility
}

\ARTICLEAUTHORS{%
   \AUTHOR{Simrita Singh$^\ast$  \hspace{0.25in}  Naireet Ghosh$^\dagger$ \hspace{0.25in}  Tinglong Dai$^\ddagger$}
   \smallskip
   \AFF{$^\ast$Leavey School of Business, Santa Clara University, Santa Clara, California, 95053 \href{mailto:ssingh17@scu.edu}{\tt ssingh17@scu.edu}\\$^\dagger$Naveen Jindal School of Management, University of Texas at Dallas, Richardson, Texas, 75080 \href{mailto:naireet.ghosh@utdallas.edu}{\tt naireet.ghosh@utdallas.edu}\\ $^\ddagger$Carey Business School, Johns Hopkins University, Baltimore, Maryland 21202, \href{mailto:dai@jhu.edu}{\tt dai@jhu.edu}}
}

\ABSTRACT{%
When firms deploy autonomous AI, they must decide how much work to leave to the system and how much to keep workers engaged. This decision affects current output and future human capital. We develop a parsimonious two-period model in which AI may outperform the worker when it functions, but may fail with positive probability. A firm chooses worker engagement; engagement lowers current output for below-benchmark workers, but changes future skill through learning and erosion. We distinguish two dimensions of AI progress: capability, the system's output when it works, and reliability, the probability that it works. In a single-firm benchmark, engagement is valuable only as fallback investment. The firm engages the least-skilled workers most, because they have the largest skill gaps and are least costly to bring toward a useful fallback level. With worker mobility, engagement also affects labor-market sorting: workers prefer jobs that build more valuable skill trajectories. This sorting motive targets higher-skill workers near the AI frontier, where skill gains are more valuable and engagement is less costly. Mobility can therefore reverse the engagement pattern, shifting investment from the least-skilled toward the most-skilled workers below the AI benchmark. Mobility also reshapes how AI progress affects engagement: greater capability raises engagement by increasing the value of the skill trajectory a firm offers, whereas greater reliability can raise or lower it because it reduces fallback need while also changing learning opportunities. Under worker mobility, human--AI work design becomes a problem of human-capital investment, in which allocating work today shapes future skill.
}


\begin{document}
\maketitle

\section{Introduction}\label{sec:intro}
As of 2026, about three-quarters of new code at Google is generated by AI and approved by engineers, up from about half in late 2025 \citep{pichai2026cloudnext}. AI writes the code; a person reviews it and answers for it \citep{linuxkernel2026coding}. When the code is wrong or insecure, when a request falls outside what the model has seen, or when the assistant is unavailable, an engineer must step in and do the work unaided. A firm that deploys an improving AI must decide how much human skill to keep in reserve for the moments AI falls short.

The same division of labor extends to other settings in which a capable but imperfect AI shares a task with a person. Waymo's vehicles complete rides with no driver, at crash rates below human benchmarks, yet a vehicle that meets an ambiguous road calls a remote specialist for guidance \citep{kusano2024waymo,waymo2024fleetresponse}. A radiologist adjudicates what an algorithm has flagged, and an underwriter signs the risk a model has scored. In each case, AI absorbs more of the routine work, and human skill is most valuable exactly where AI reaches the limits of its competence, limits that arise less and less often as that absorption proceeds.

It is tempting to view this division of labor as a transitional arrangement on the way to production with little human involvement. Popular accounts take that view: Yuval Noah Harari warns of a coming ``useless class'' \citep{harari2017useless}, and Sam Altman predicts that ``whole classes of jobs'' may disappear \citep{altman2025singularity}. For the firm itself, the decision is subtler. Human skill remains valuable because AI is imperfect: the worker is the fallback when AI fails, faces a case outside its domain, or is unavailable. But the skill is portable, so part of the return to preserving it accrues to the worker or another employer; and engaging the worker is costly, because AI often does this work better than the worker would. Reducing engagement thus raises current output but erodes the skill the firm may later need. The same progress that makes AI less costly to lean on also makes the skill it displaces harder to sustain. We study how this friction shapes worker engagement: the extent to which a firm keeps the worker actively involved in AI-assisted production.\footnote{Institutions raise engagement through concrete instruments: AI-free settings, a minimum number of cases handled without AI before access is granted, staged access, and protocols that have the worker commit to a judgment before consulting AI \citep{keren2026expertise}. Shell, for example, requires early-career employees to frame a problem before using AI \citep{bcg2026deskilling}.} Engaging the worker lowers output today but preserves skill valued both inside the firm and in the labor market.

A worker's human capital is not fixed; sustained reliance on AI can weaken it. In software, programmers who leaned on AI to learn a new library gained little speed on average yet ended up worse at reading and debugging code unaided \citep{shen2026skillformation}. In education, high school students given an unguarded AI tutor solved more problems while they had it, but once it was withdrawn they scored below classmates who had never used it \citep{bastani2025genai}. In medicine, after routine exposure to AI-assisted colonoscopy, experienced endoscopists' unassisted adenoma-detection rate fell from 28.4\% to 22.4\% \citep{budzyn2025endoscopist}. In a survey of business leaders, about half already observe deskilling, and more than 60\% expect it to become a real threat within three to five years \citep{bcg2026deskilling}. Engagement is costly today, but disengagement erodes the human capital future production may need.

Workers, in turn, value the skill a job builds. Career advancement, learning, and skill development have become central to whether they take, keep, or leave a job \citep{gallup2024upskilling,pew2022quit}. A job that routes most of its work to AI builds little human capital, and a mobile worker has reason to leave it for one that builds more.

This matters most where firms cannot differentiate on pay. When wages for comparable roles are standardized by salary bands, internal-equity norms, and external market benchmarking \citep{baker1988compensation,mas2017transparency}, a firm cannot simply outbid a rival for a worker who could leave; what it can offer instead is the skill trajectory the job builds. Engagement, the same lever that governs current output, then becomes the instrument through which a firm competes for mobile workers.

We develop a parsimonious two-period model of skill investment when an improving AI shares a task with a mobile worker. A firm chooses how actively to involve a worker in an AI-assisted task; engagement shapes current output and, through learning-by-doing and erosion, the worker's future skill. We model two attributes of AI along which it improves: \emph{capability}, how good its output is when it functions, and \emph{reliability}, how often it functions as intended. Both are visible in the settings above: a more capable system handles more of the road, or more of the codebase, on its own, whereas a more reliable one fails or defers to a person less often. We focus on below-benchmark workers, whose unaided output falls short of what AI delivers; for them engagement lowers current output, because AI is faster or more accurate, but it builds skill the firm may later need. Models of automation allocate tasks between people and machines but take the skill distribution as fixed \citep{Acemoglu2018,Autor2003}; here engagement reshapes it.

A firm engages such a worker, despite the current cost, because engagement builds the worker's skill. Engagement today sets the skill the worker carries forward, so the firm weighs current output against a future in which that skill may be its fallback, its draw in the labor market, or both. Without mobility, that skill is worth preserving as a fallback for when AI fails. Under mobility, two ex ante identical firms set engagement policies and workers sort across them by the skill trajectories those policies imply; engagement also attracts workers, whose skill the firm cannot fully keep.

Our analysis yields three findings. First, the two motives target opposite ends of the skill range. Without mobility, the fallback motive drives engagement toward the lowest-skilled workers: they sit furthest below AI, so they have the most room to grow, and a unit of engagement builds the most skill. Mobility adds a sorting motive, strongest for higher-skilled workers, for whom added skill is worth the most in the labor market and so does the most to draw them. The same skill is a fallback where workers are weakest and a draw where they are strongest, so which workers a firm engages most can reverse, from the bottom of the skill range to the top, once workers are mobile. This is the engagement reversal: the firm's lever does not change, but the workers it engages most can shift from the least- to the most-skilled as the labor market opens.

Second, the two dimensions of AI progress can move engagement in opposite directions. Reliability sets how often human skill is called on. Greater reliability makes the fallback less likely to be needed, weakening one reason to engage; but workers build skill best with a system that mostly works yet occasionally fails, so engagement can be highest at intermediate reliability. A system that never fails leaves the worker nothing to stay ready for, while one that fails often is barely worth deploying; the room to build skill lies in between. Capability does not by itself change how often workers are needed for fallback, but under mobility a more capable AI widens the gap from which workers learn, strengthening the trajectory a firm can offer. Treating AI progress as a single index conflates channels that can point in opposite directions.

Third, mobility reshapes engagement relative to the no-mobility benchmark. Because the skill a firm builds is general, it cannot capture the full return, so mobility should depress investment, as \citet{Becker1962} first argued.\footnote{In Becker's analysis of general (fully portable) training, a competitive firm bears none of its cost: because a trained worker can quit and command the now-higher marginal product at any employer, wages rise to match it and the firm recoups nothing. As Becker puts it, ``the cost as well as the return from general training would be borne by trainees, not by firms'' (\citealp{Becker1962}, p.~13). The worker therefore finances such training through lower wages while training and captures its return afterward; only firm-specific training is shared between firm and worker.} We prove that this holds when workers respond weakly to differences in skill trajectories: engagement falls below the single-firm benchmark. When they respond strongly, the sorting motive can instead reverse the timing of engagement, so that mobility need not uniformly discourage investment in general skill. Mobility can also break the symmetry between ex ante identical firms, with one investing in skill while the other draws on the shared labor pool.

To our knowledge, this is among the first analytical models in which a firm's engagement with an improving AI shapes the skill of mobile workers. It identifies a sorting motive for skill investment, absent from existing work on AI-assisted production, by which firms differentiate through the skill a job builds when they cannot compete on pay \citep{bastani2025contracting,caosun2026augmentation,dai2025enterpriseAI,lu2025tomlin,siderius2026useit,xu2025orgstructure}.

Together, these results recast the design of human--AI work as a human-capital problem as much as an operational one. Wherever an improving but imperfect AI shares a task with a mobile worker, whether in autonomous driving, software, radiology, underwriting, contact centers, or legal review, the firm's engagement choice governs not only today's output but the skill it will have on hand tomorrow and the workers it can keep.

The remainder of the paper is organized as follows. \cref{sec:literature} reviews the related literature. \cref{sec:model} presents the model. \cref{sec:monopoly} analyzes the single-firm benchmark. \cref{sec:competition} studies how worker mobility reshapes engagement. \cref{sec:discussion} discusses implications and concludes.

\section{Related Literature}\label{sec:literature}

Our paper connects two questions that are often treated separately. The first is how automation changes the use and preservation of human skill. The second is how firms invest in human capital. AI-assisted work brings these questions together. When the system outperforms the worker, engagement is costly today. Yet engagement may also preserve skill that the firm needs as fallback and that the worker values in the labor market. Our contribution is to model this tradeoff and show how worker mobility changes both the level of engagement and the workers to whom engagement is directed.

The first question, how automation changes the use and preservation of human skill, has been studied from several angles. The most direct is AI-induced \emph{deskilling}: more reliable systems erode the human capabilities needed when those systems fail. \citet{bainbridge1983} identifies the ``ironies of automation,'' in which more reliable systems leave operators less able to intervene when failures occur, and \citet{parasuraman1997} formalize the related taxonomy of automation misuse, disuse, and abuse. Modern AI generates the same pattern: \citet{bastani2025genai} find unrestricted generative-AI access improves high-school students' performance during practice but lowers exam scores once access is removed; \citet{poulidis2025selfregulated} show learners over-request help even when they understand its long-term cost; and \citet{dellacqua2023} find AI raises consultants' productivity within its frontier but degrades it outside, because workers anchor to AI output. Clinical evidence points the same way: endoscopists' unassisted adenoma-detection declined after routine AI-assisted colonoscopy \citep{budzyn2025endoscopist}, and \citet{abdulnour2025nejm}, \citet{natali2024deskilling}, and \citet{shen2026skillformation} document deskilling, never-skilling, and mis-skilling in medical training and software. This literature establishes that reliance on automation can erode the skill needed in failure states, but it treats that erosion as a byproduct of automation rather than as something a firm controls. We make erosion and learning \emph{endogenous} to the firm's engagement policy and ask which workers a profit-maximizing firm keeps~engaged.

A second angle is human--AI co-production, in which a central finding is that adding humans to AI need not improve output. \citet{Vaccaro2024combinations}, in a meta-analysis of experiments, find human--AI combinations often underperform the better of human or AI alone, especially when AI is stronger. \citet{devericourt2023machine} provide an analytical explanation: when decision-makers supervise AI output, verification bias can keep them from learning whether the machine outperforms their own judgment. \citet{jamanetwork2025} document a related pattern in a randomized clinical trial: physician access to an LLM did not improve diagnostic reasoning even though the LLM alone performed well, suggesting that integrating AI advice into expert judgment is itself difficult. \citet{combn_human_ai} show a similar friction in radiology: AI predictions alone did not improve performance on average because radiologists underweighted the AI signal and did not combine it correctly with their own information, whereas contextual information improved performance. \citet{singh2025pathfinder} show that when physicians exhibit anchoring bias, AI should serve as a gatekeeper for low-risk patients and a second opinion for high-risk patients.

AI assistance also appears most helpful for less-experienced workers: \citet{brynjolfsson2025genaiwork} find generative AI raises customer-support productivity on average, with gains concentrated among less-experienced agents and small quality declines among the most skilled, and \citet{ni2024ecommerce} find a similar skill-dependent pattern in a field experiment at Alibaba. This literature studies whether and how to pair people with AI; we study the dynamic skill consequences of keeping a worker engaged. Our engagement variable is therefore distinct: not whether a worker has access to AI, but whether the worker stays actively involved in production when the AI baseline is already stronger. Keeping a below-benchmark worker engaged is operationally costly, pulling output away from the AI baseline, even as it builds skill through learning-by-doing.

Closest to our model is a growing literature on the design of human--AI work. \citet{siderius2026useit} develop a principal-agent model in which workers exert costly effort to verify imperfect AI output and show profit-maximizing compensation can be non-monotonic in AI quality. \citet{caosun2026augmentation} develop a continuous-time model in which a decision-maker chooses AI usage intensity, showing forward-looking adoption can rationally lead to long-run skill loss, an ``augmentation trap.'' \citet{bastani2025contracting} identify a ``human--AI contracting paradox'': as AI becomes more reliable, motivating human vigilance becomes prohibitively costly, so firms may prefer less reliable AI. \citet{dai2025enterpriseAI} study a principal-agent model in which firms jointly choose AI temperature and effort-contingent pay, showing endogenizing AI design can reverse classical delegation results by making temperature and worker effort complementary levers. \citet{lu2025tomlin} show fully disclosing an AI demand forecast can reduce managerial effort below what preserves the machine's value, and characterize partial disclosure as the optimal design response. \citet{xu2025orgstructure} show AI's effect on entry-level skill requirements and organizational span of control depends on whether firms deploy it as automation or augmentation. These models study verification, contracting, disclosure, or AI-use intensity, each within a single firm, with AI progress as a single dimension and the worker tied to one employer even as skill changes over the long run. We ask how these forces interact, and our model brings them together: engagement trades current production against future skill; AI improves along two distinct dimensions, capability and reliability, which affect engagement through different channels; and workers sort across firms by the skill trajectories engagement creates.

The second question, how firms invest in skill when workers are mobile, belongs to the human-capital literature. The economics of human capital, following \citet{Becker1962} and \citet{mincer1974}, establishes that skills accumulate through education and on-the-job training. \citet{acemoglu1999structure} show labor-market frictions, particularly compressed wage structures, enable firm-sponsored general training that would unravel in a frictionless market. The operations-management literature has modeled closely related workforce-learning and turnover tradeoffs with dynamic staffing and recruitment models: \citet{gans2002learningturnover} on staffing when employees learn and turn over, \citet{arlotto2014hiringretention} on hiring and retention for heterogeneous workers who learn, and \citet{whitt2006retention} on retention and contact-center performance. Workers may also sort across employers by the skill trajectories their jobs create. We contribute to this literature by identifying a sorting channel for skill investment under wage standardization: when firms cannot differentiate through pay, they differentiate through the skill trajectories their engagement policies imply. Unlike the classical prediction that mobility weakens incentives to invest in general skill, this channel can strengthen them, shifting engagement toward higher-skill workers, the opposite end from the fallback motive, and it can lead ex ante identical firms to specialize.
\section{Model}\label{sec:model}

We study a firm that, over two periods, chooses worker engagement with an improving AI system. The worker is below the AI benchmark (the worker's effective throughput is lower than that of AI when it functions), so engagement lowers current output. The value of engagement is instead dynamic: it slows erosion and can build skill that matters later as fallback capacity and, when workers are mobile, as a labor-market draw. The model therefore has three primitives: AI capability and failure risk, skill dynamics under engagement, and a wage schedule that values portable skill.

\subsection{Environment and Technology}

The model spans $T$ periods, $t=1,2,\ldots,T$. A firm employs a continuum of workers who do not interact in production, so we describe its problem for a single worker of skill $s_t\in[0,1]$, measured as the worker's unaided effective throughput on the task. In software development, for example, $s_t$ is the worker's independent rate of correct routine code completion; lower skill means slower work, more mistakes, or more rework. Production is assisted by an AI system indexed by $A_t$. The technology level determines two parameters: capability $\alpha_t \triangleq \alpha(A_t) \in (0,1)$, measured on the same normalized effective-throughput scale as worker skill $s_t$, and failure probability $\pi_t \triangleq \pi(A_t) \in (0,1)$, the probability that AI cannot be used as the autonomous producer. The path $\{A_t\}_{t=1}^T$ is exogenous and known to all agents. In each period, the firm chooses an engagement level $h_t \in [0,1]$, which sets how much of the workflow is routed through the worker rather than left to the AI alone. A higher $h_t$ means deeper human review, more hands-on work or verification, or more frequent escalation to the worker. \cref{fig:model_timing} shows the sequence of events within a period.

\begin{figure}[t]
\centering
\definecolor{timingink}{HTML}{2A2A2A}
\begin{tikzpicture}[font=\small, >={Stealth[length=2.2mm]},
   box/.style={draw=timingink!80, line width=0.5pt, fill=timingink!3,
               align=center, text width=22.5mm, minimum height=11mm,
               font=\scriptsize, inner sep=2.5pt}]

\foreach \x/\lab in {0/n1, 5/n2, 10/n3}
  \node[circle, draw=timingink, fill=white, line width=0.7pt, inner sep=1.8pt] (\lab) at (\x,0) {};
\node[below=1.5pt of n1] {$s_1$};
\node[below=1.5pt of n2] {$s_2$};
\node[below=1.5pt of n3] {$s_3$};
\draw[->, line width=0.9pt, timingink] (n1) -- node[above=1pt, font=\footnotesize]{$g(s_1,h_1;A_1)$} (n2);
\draw[->, line width=0.9pt, timingink] (n2) -- node[above=1pt, font=\footnotesize]{$g(s_2,h_2;A_2)$} (n3);
\draw[->, line width=0.9pt, dashed, timingink] (n3) -- ++(2.1,0) node[right, font=\footnotesize]{$B(s_3)$};
\node[font=\scriptsize, timingink, above=2.5pt] at (11.05,0) {post-horizon};

\draw[decorate, decoration={brace, amplitude=4pt, mirror, raise=12pt}, timingink!60]
     (n1.south) -- (n2.south) node[midway, below=17pt, font=\footnotesize]{Period 1};
\draw[decorate, decoration={brace, amplitude=4pt, mirror, raise=12pt}, timingink!60]
     (n2.south) -- (n3.south) node[midway, below=17pt, font=\footnotesize]{Period 2};

\node[box, anchor=west] (a) at (-1.35,-1.95) {AI technology\\ $A_t$ realized};
\node[box, right=3mm of a] (b) {Firm chooses\\ engagement $h_t$};
\node[box, right=3mm of b] (c) {Workers sort\\ (logit share $\sigma_t^{\,j}$)};
\node[box, right=3mm of c] (d) {Output, wages,\\ and profits};
\node[box, right=3mm of d] (e) {Skills update\\ via $g(\cdot)$};
\foreach \i/\j in {a/b,b/c,c/d,d/e} \draw[->, timingink!85, line width=0.7pt] (\i) -- (\j);
\end{tikzpicture}
\caption{Timing of the two-period model. Worker skill is carried forward
$s_1\!\to\!s_2\!\to\!s_3$; the terminal skill $s_3$ is valued post-horizon by
$B(\cdot)$, and within each period skill evolves according to
$s_{t+1}=g(s_t,h_t;A_t)$. The lower row lists the order of events within a period. The sorting step is active under worker mobility and degenerate in the single-firm benchmark.}
\label{fig:model_timing}
\end{figure}
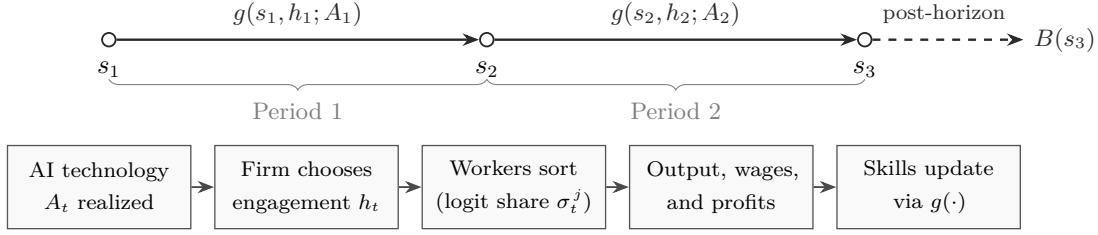
Throughout the paper, we focus on the two-period case $T = 2$,
which is sufficient to capture the key tradeoffs (between current output and future skill, between the fallback and sorting motives, and between capability and reliability) while keeping the analysis tractable and the results interpretable.

\begin{assumption}\label{ass:tech}
The capability function $\alpha(\cdot)$ is strictly increasing and the
failure probability $\pi(\cdot)$ is strictly decreasing. The technology
path satisfies $A_2 \ge A_1$, so $\alpha_2 \ge \alpha_1$ and
$\pi_2 \le \pi_1$.
\end{assumption}

This assumption captures technological improvement along both
dimensions: more advanced AI systems are both more capable when they
function and less likely to fail. We keep capability and
reliability as separate parameters because they enter the model through
distinct channels: capability $\alpha_t$ governs the value of AI-assisted output and the size of the performance gap $\alpha_t-s_t$, whereas reliability $1-\pi_t$ determines how often AI functions and how often workers must perform the task on their own.

\subsection{AI-Assisted Production and the Engagement-Output Tradeoff}

Production depends on whether the AI system functions or fails. Let $q\in[0,1]$ denote effective throughput: the fraction of a standardized task load completed correctly in a period, measured on the same scale as capability $\alpha_t$ and worker skill $s_t$. When AI functions (is on), which occurs with probability $1-\pi_t$, the AI-alone effective throughput is $\alpha_t$, and worker engagement modifies this throughput according to
\begin{align*}
q^{\mathrm{on}}(s_t,h_t;A_t)
=
\alpha_t+\delta h_t(s_t-\alpha_t),
\end{align*}
where $\delta\in(0,1)$ measures the influence of worker intervention on AI-on production. When $s_t > \alpha_t$, engagement raises effective throughput by pulling it toward the worker's skill; when $s_t < \alpha_t$, engagement lowers effective throughput for the same reason.

When AI fails (is off), which occurs with probability $\pi_t$, the worker performs the task independently, so effective throughput equals the worker's own skill, $q^{\mathrm{off}}(s_t)=s_t$.

Let $\lambda>0$ denote the standardized task load per worker per period, and let $R>0$ denote revenue per correctly completed task. The firm's expected operational revenue per period therefore equals
\begin{align*}
S(s_t,h_t;A_t)
=
\lambda R
\left[
(1-\pi_t)\big(\alpha_t+\delta h_t(s_t-\alpha_t)\big)
+
\pi_t s_t
\right].
\end{align*}

The production effect of engagement is $\partial S(s_t,h_t;A_t)/\partial h_t = \lambda R\delta(1-\pi_t)(s_t-\alpha_t)$, so for every below-benchmark worker ($s_t<\alpha_t$) engagement strictly lowers current operational revenue. The analysis below focuses on this region, where engagement is costly today but can preserve or build skill.

\subsection{Skill Dynamics: Learning and Erosion}

Engagement changes future skill through a tradeoff between learning and erosion. When the firm leaves the task to AI and AI functions, the worker does less of the task and loses skill through passive reliance. When the firm keeps the worker engaged, the worker learns by staying involved in AI-assisted cases and by practicing independent execution in the states in which AI cannot be relied on. For a below-benchmark worker, $s_t < \alpha_t$, we capture these two forces with an erosion term, $\gamma(1-h_t)(1-\pi_t)(\alpha_t-s_t)$, and a learning term, $\phi h_t \pi_t(1-\pi_t)(\alpha_t-s_t)$. Both are proportional to the gap $\alpha_t-s_t$: workers with more room to grow learn more from engagement, but they also lose more when AI routinely does the work. The product $\pi_t(1-\pi_t)$ implies that learning is strongest at intermediate reliability. If AI never works, the worker has little useful output to learn from; if it never fails, the worker gets little fallback practice. The product form is a parsimonious reduced form for this exposure--practice complementarity; for the results most exposed to this specification, we report robustness checks that replace $\pi_t(1-\pi_t)$ with a generic learning profile $\ell(\pi_t)$.\footnote{We consider profiles such as $\ell(\pi)=\pi^a(1-\pi)^c$, $\ell(\pi)=\kappa+\pi(1-\pi)$,
$\ell(\pi)=\pi$, and $\ell(\pi)=1-\pi$, with $a,c>0$ and $\kappa\ge 0$. These alternatives capture asymmetric exposure--practice complementarity, baseline learning independent of failures, failure-practice learning, and
AI-exposure learning, respectively.} The parameters $\gamma>0$ and
$\phi>0$ measure the rates of erosion and learning, respectively.

Combining these forces yields the skill transition
\begin{align*}
s_{t+1}
=
g(s_t,h_t;A_t)
=
\begin{cases}
\max\{s_t+\Gamma(h_t;A_t)(\alpha_t-s_t),\,0\}
& \text{if } s_t<\alpha_t,\\[4pt]
s_t
& \text{if } s_t\ge\alpha_t,
\end{cases}
\end{align*}
where
\begin{align*}
\Gamma(h_t;A_t)
=
(1-\pi_t)\big[h_t(\phi\pi_t+\gamma)-\gamma\big].
\end{align*}
For workers with $s_t < \alpha_t$, $\Gamma(h_t;A_t)$ represents the net
fraction of the skill gap $\alpha_t - s_t$ that the worker closes in one period (when $\Gamma$ is positive) or the fraction by which the gap widens (when $\Gamma$ is negative). For workers with
$s_t \ge \alpha_t$, we set $s_{t+1} = s_t$, that is, we assume workers already at or above AI's capability level maintain their current skill.

Solving $\Gamma(h_t;A_t) = 0$ yields the engagement level that
preserves worker skill,
\begin{align*}
h^c(A_t) \triangleq \frac{\gamma}{\phi\pi_t + \gamma} \in (0,1).
\end{align*}
For workers with $s_t < \alpha_t$, skill declines when $h_t < h^c(A_t)$
and improves when $h_t > h^c(A_t)$. Engagement therefore functions as an investment in worker skill: for below-benchmark workers, it lowers the current effective throughput of the AI-assisted workflow, but raises future worker skill and hence future earning capacity. The threshold $h^c(A_t)$ is strictly decreasing in
$\pi_t$, and therefore rises as AI becomes more reliable; that is, more
engagement is required to maintain worker skill when failures become
rarer. This is the \emph{automation paradox}
\citep{bainbridge1983,parasuraman1997}: $\pi_t$ governs both how often
human skill is needed and how often it is exercised, so gains in AI
reliability simultaneously reduce the need for human skill and the
opportunities to preserve it. More generally, with learning $\phi h_t\ell(\pi_t)(\alpha_t-s_t)$, the threshold is $h^c_\ell(A_t)=\gamma(1-\pi_t)/[\phi\ell(\pi_t)+\gamma(1-\pi_t)]$; it rises with reliability (i.e., falls in $\pi_t$) whenever $\ell(\pi)/(1-\pi)$ is increasing in $\pi$, including $\ell(\pi)=\pi^a(1-\pi)^c$ on ranges where $a(1-\pi)>(c-1)\pi$.

\begin{assumption}\label{ass:skill}
For every AI technology level $A$,
$\phi\pi(A)(1-\pi(A)) < 1$.
\end{assumption}

\cref{ass:skill} rules out one-period leapfrogging: a worker who starts
below the AI benchmark remains below it after one period, even under full
engagement. The point of the assumption is economic as well as
technical. It keeps the analysis in the region in which engagement is
costly in current production, so any positive engagement for a
below-benchmark worker must be justified by future skill rather than by
an immediate operating gain. The assumption also rules out an
implausibly fast learning path in which one period of engagement turns a
below-benchmark worker into the new frontier. The largest fraction of
the skill gap closed in one period occurs under full engagement, where it
equals $\phi\pi(A)(1-\pi(A))$; lower engagement closes a weakly smaller
fraction because passive reliance erodes skill. Because
$\pi(1-\pi) \le 1/4$, the simpler restriction $\phi < 4$ is a
technology-independent sufficient condition. For the general learning profile, the analogous no-leapfrogging condition is $\phi\ell(\pi(A))<1$ for every AI technology level $A$.

\begin{assumption}\label{ass:skill_domain}
The initial skill level satisfies $s_1 \in (\underline{s},\,\alpha_1)$,
where
\begin{align*}
\underline{s}
\triangleq
\frac{\gamma(1-\pi_2)\alpha_2 + \gamma(1-\pi_1)\alpha_1
\bigl[1+\gamma(1-\pi_2)\bigr]}
{\bigl[1+\gamma(1-\pi_2)\bigr]\bigl[1+\gamma(1-\pi_1)\bigr]}.
\end{align*}
\end{assumption}

\cref{ass:skill_domain} keeps the zero-skill boundary from binding on
the relevant domain. Even under complete disengagement in both periods,
worker skill remains strictly positive. Because engagement only builds
skill for a below-benchmark worker, disengagement is the binding case:
the worst case is $h_1 = h_2 = 0$, and $\underline{s}$ is constructed so
that this path still stays in the interior. In particular, for every
$h_1\in[0,1]$,
\begin{align*}
s_2=g(s_1,h_1;A_1)>\underline{s}_2
\triangleq
\frac{\gamma(1-\pi_2)\alpha_2}{1+\gamma(1-\pi_2)},
\end{align*}
and hence $g(s_2,h_2;A_2)>0$ for every $h_2\in[0,1]$. As a result, the
main analysis can work with the smooth part of the skill transition
rather than carrying the $\max\{\cdot,0\}$ operator through every
comparative static.

\subsection{Worker Payoff}\label{sec:worker_utility}

Workers earn a market-clearing wage that depends on their skill
level. Because skill is general-purpose and observable, an outside labor
market pins each worker's compensation to a common wage schedule: any
firm employing a worker of skill $s_t$ pays $W(s_t)$. Standardized wages
capture settings with salary bands, internal equity constraints, and
market benchmarking \citep{baker1988compensation,mas2017transparency}, in
which firms have limited ability to offer idiosyncratic premiums. Because
firms cannot bid for workers with wages, their main lever for attracting
workers is job design, specifically how actively workers engage with
AI-assisted tasks: a policy that builds more skill leads to higher future
wages, making the firm more attractive to workers.

We maintain standardized wages throughout the main
analysis. Holding pay fixed isolates job design, and engagement in
particular, as the lever firms use to attract workers, which is the
mechanism we study. The restriction is not essential to our conclusions:
in \cref{app:bonus} we let firms attract workers with pay as
well, offering a bonus above the market wage alongside their
engagement choice, and a numerical study there shows our main
results carry over.

The firm's per-worker margin equals operational revenue minus wages,
\begin{align*}
M(s_t,h_t;A_t) = S(s_t,h_t;A_t) - W(s_t).
\end{align*}
Because the wage in period $t$ depends only on entering skill $s_t$, not on the engagement chosen within the period, engagement affects the margin only through operational revenue. Since $S(s_t,h_t;A_t)$ is affine in $h_t$
with slope $\lambda R\delta(1-\pi_t)(s_t-\alpha_t)$, the margin is
strictly decreasing in $h_t$ whenever $s_t < \alpha_t$: for
below-benchmark workers, engagement reduces current output and lowers the
firm's current margin.

Fix a realized engagement path and abstract from the idiosyncratic taste shocks introduced in the mobility model. The deterministic portion of a worker's payoff then has three components:
wages earned in period~1, discounted wages earned in period~2, and a
discounted terminal value capturing the labor-market worth of the skill
the worker carries beyond the model horizon. A worker with initial skill
$s_1$ who experiences engagement levels $h_1$ and $h_2$ has deterministic payoff
\begin{align*}
W(s_1) + \beta\, W(s_2) + \beta^2\, B(s_3),
\end{align*}
where $s_2 = g(s_1, h_1; A_1)$ and $s_3 = g(s_2, h_2; A_2)$. Here
$B:[0,1]\to\mathbb{R}_+$ is the post-horizon value of skill, which we take
to be $C^2$ on $(0,1]$, strictly increasing, and strictly convex; in our comparative
statics and numerical examples, we set $B=W$, valuing post-horizon skill
by the same wage schedule used within the model.

Because wages are determined by the market and do not depend on the
firm's engagement policy, this deterministic payoff differs across firms only
through the skill trajectory each firm's policy creates. In the single-firm
benchmark of \cref{sec:monopoly}, a single employer operates, so the
worker's payoff does not enter the firm's engagement decision. When
workers are mobile (\cref{sec:competition}), the terminal
value drives sorting: a firm whose engagement policy yields higher future
skill, and hence higher future wages and terminal value, attracts more
workers.

\begin{assumption}\label{ass:wages_revenue}
The wage function $W : [0,1] \to \mathbb{R}_+$ is $C^2$ on $(0,1]$,
strictly increasing, strictly convex, with $W(0) = 0$. The revenue scale
satisfies $\lambda R > \max_{t=1,2} W(\alpha_t)/\alpha_t$.
\end{assumption}

Convexity of $W$ is consistent with the Mincerian earnings function
\citep{mincer1974} and with evidence on convex returns to education
\citep{psacharopoulos2018}; the power wage $W(s) = s^b$ with $b > 1$, used
in our numerical examples, satisfies these conditions. The revenue
condition comes from setting $h_t = 0$ and $s_t = \alpha_t$ in the revenue
expression: $S(\alpha_t, 0; A_t) = \lambda R \alpha_t$, so $M = \lambda R
\alpha_t - W(\alpha_t)$, and $M \ge 0$ requires $\lambda R >
W(\alpha_t)/\alpha_t$. This is the tightest case; the margin is positive
for all other $(s_t, h_t)$ in the relevant domain. Employing a
below-benchmark worker is therefore always profitable, so the analysis
concerns how much to engage the worker, not whether to employ one.

The mobility variables $\eta$ and $\sigma_t^j$ are introduced formally in \cref{sec:competition}. The model primitives have direct empirical counterparts in task-level workflow data, and their measurement and calibration are detailed in \cref{app:measurement}. A summary of the notation appears in \cref{tab:notation} of the online appendix.

\section{The Single-Firm Benchmark}\label{sec:monopoly}

We begin with the single-firm benchmark, which isolates the fallback motive. Without worker mobility, engagement matters only because it changes the skill the firm will have on hand later when AI fails. This benchmark makes clear what worker mobility adds in \cref{sec:competition}: sorting and imperfect appropriability. We first study a one-period benchmark and then specialize
to the two-period setting.

\subsection{Single Period: No Skill Investment}
\label{sec:one_period_monopoly}

Consider a single period with AI technology level $A$. For a worker of
skill $s$, the single firm chooses engagement $h\in[0,1]$ to maximize
the per-worker margin $M(s,h;A) = S(s,h;A) - W(s)$. Because the wage
$W(s)$ does not depend on $h$, maximizing the margin is equivalent to
maximizing operational revenue.

That revenue is affine in $h$ with slope
$\lambda R\delta(1-\pi(A))(s-\alpha(A))$, so engagement raises it only
when the worker's skill exceeds AI capability. The single firm therefore
engages a worker fully when $s>\alpha(A)$ and not at all when
$s<\alpha(A)$, and is indifferent at $s=\alpha(A)$. With no future skill
to protect, it disengages every below-benchmark worker, those with
$s<\alpha(A)$. Engagement here is purely static: the firm builds no
skill, because no future state remains in which that skill could pay
off.

\subsection{Two Periods: The Fallback Motive}
\label{sec:two_period_monopoly}

We now set the horizon to two periods, $t \in \{1,2\}$, with AI technology
path $(A_1, A_2)$. The single firm chooses $(h_1^{\mathrm{sf}}, h_2^{\mathrm{sf}}) \in [0,1]^2$
to maximize total discounted profit. We focus on workers with initial
skill $s_1 < \alpha_1$, because these are the workers for whom engagement involves a tradeoff: it lowers current production but preserves future skill. By the no-leapfrogging condition,
$g(s_1, h_1; A_1) < \alpha_1 \le \alpha_2$ for every $h_1 \in [0,1]$,
so engagement is operationally costly in both periods, and any positive first-period engagement must reflect the firm's goal of preserving fallback skill.
In the terminal period, the single firm has no future to protect, so the
static logic above applies directly: it disengages every below-benchmark
worker, setting $h_2^{\mathrm{sf}} = 0$ when $s_2 < \alpha_2$ (and is indifferent at
$s_2 = \alpha_2$). The period-1 problem is therefore
\begin{align*}
\max_{h_1^{\mathrm{sf}} \in [0,1]}\;
M(s_1, h_1^{\mathrm{sf}}; A_1)
+ \beta\, M\!\bigl(g(s_1, h_1^{\mathrm{sf}}; A_1),\, 0;\, A_2\bigr).
\end{align*}
The objective is strictly concave in $h_1^{\mathrm{sf}}$, so the optimum is unique. 
The first-order condition balances the benefit of building worker skill against its cost. The benefit is the extra operational revenue a more skilled worker generates as a fallback when AI fails; the cost is the current production forgone plus the higher wage the firm must later pay. At an interior optimum, this gives $W'(g(s_1,h_1;A_1))=\lambda R\pi_2-\lambda R\delta/[\beta(\phi\pi_1+\gamma)]$.
When the right-hand side lies in the range of marginal wages,
\begin{align}\label{eq:interior_target_condition}
\lambda R \pi_2 - \frac{\lambda R \delta}{\beta(\phi\pi_1 + \gamma)}
\in \bigl(W'(0),\, W'(\alpha_1)\bigr),
\end{align}
the right-hand side has a unique preimage under $W'$, the \emph{target skill} $s^*$:
\begin{align}\label{eq:s_star}
s^* \triangleq
(W')^{-1}\!\left(
\lambda R \pi_2 - \frac{\lambda R \delta}{\beta(\phi\pi_1 + \gamma)}
\right) \in (0, \alpha_1),
\end{align}
the period-2 skill the single firm would most like the worker to reach. If the right-hand side is weakly below $W'(0)$, then the objective is decreasing in $h_1^{\mathrm{sf}}$, so the firm sets $h_1^{\mathrm{sf}}(s_1)=0$ for every below-benchmark worker. If the right-hand side is weakly above $W'(\alpha_1)$, then the objective is increasing in $h_1^{\mathrm{sf}}$, so the firm sets $h_1^{\mathrm{sf}}(s_1)=1$ for every below-benchmark worker. Condition~\cref{eq:interior_target_condition} rules out these two boundary cases and focuses attention on the case in which the target skill $s^*$ is well defined. Even then, the target may lie outside a particular worker's attainable interval $[g(s_1,0;A_1),g(s_1,1;A_1)]$, giving the worker-specific corner cases characterized below.
\begin{proposition}
\label{prop:monopoly_period1}
For each $s_1 \in (\underline{s}, \alpha_1)$, the firm's period-1
engagement $h_1^{\mathrm{sf}}(s_1)$ is unique and drives period-2 skill toward the
target $s^*$, found by comparing $s^*$ with the worker's attainable
interval $[g(s_1,0;A_1),\,g(s_1,1;A_1)]$:
\begin{enumerate}
  \item[\textup{(i)}] if $s^*\ge g(s_1,1;A_1)$, so even full engagement
        falls short of the target, the firm engages fully,
        $h_1^{\mathrm{sf}}(s_1)=1$;
  \item[\textup{(ii)}] if $g(s_1,0;A_1)<s^*<g(s_1,1;A_1)$, the target is
        attainable and engagement is interior, setting period-2 skill
        $g(s_1,h_1^{\mathrm{sf}}(s_1);A_1)=s^*$,
        \begin{align}\label{eq:h_interior_closed}
        h_1^{\mathrm{sf}}(s_1) =
        \frac{s^* - g(s_1, 0; A_1)}
        {(1-\pi_1)(\phi\pi_1 + \gamma)(\alpha_1 - s_1)};
        \end{align}
  \item[\textup{(iii)}] if $s^*\le g(s_1,0;A_1)$, so even full atrophy
        leaves the worker above the target, the firm does not engage,
        $h_1^{\mathrm{sf}}(s_1)=0$.
\end{enumerate}
In the interior region, $h_1^{\mathrm{sf}}(s_1)$ is strictly decreasing in $s_1$.
\end{proposition}
\cref{prop:monopoly_period1} shows that the firm follows a skill-targeting rule: a single target $s^*$ summarizes the period-2 skill it wants each worker to reach, and it engages just enough to get there. Whether a worker is engaged fully, partially, or not at all depends only on where the target sits relative to the skill the worker can attain.

This characterization can equivalently be stated in terms of
initial-skill thresholds. Define
\begin{align}\label{eq:s_L_s_H}
s_L \triangleq \frac{s^* - \phi\pi_1(1-\pi_1)\alpha_1}
{1 - \phi\pi_1(1-\pi_1)},
\qquad
s_H \triangleq \frac{s^* + \gamma(1-\pi_1)\alpha_1}
{1 + \gamma(1-\pi_1)}.
\end{align}
Full engagement arises for $s_1 \le s_L$, no engagement for
$s_1 \ge s_H$, and interior engagement for $s_1 \in (s_L, s_H)$,
intersected with the feasible domain
$s_1 \in (\underline{s}, \alpha_1)$. Equivalently, the policy can be
summarized by the thresholds $s_L$ and $s_H$; a tabular summary is
reported in \cref{app:skill_transition_table}.
\begin{corollary}
\label{cor:reliability_vs_capability}
When the target skill $s^*$ is interior:

\medskip

\noindent\textup{(i)} The target skill satisfies
\begin{align*}
\frac{\partial s^*}{\partial \pi_2}
= \frac{\lambda R}{W''(s^*)} > 0,
\qquad
\frac{\partial s^*}{\partial \alpha_2} = 0,
\end{align*}
and the thresholds $s_L,\, s_H$ in~\cref{eq:s_L_s_H} are strictly
increasing in $\pi_2$ and independent of $\alpha_2$.

\medskip

\noindent\textup{(ii)} For any fixed $s_1 \in (\underline{s}, \alpha_1)$,
the period-1 engagement $h_1^{\mathrm{sf}}(s_1)$ characterized in
\cref{prop:monopoly_period1} is weakly increasing in $\pi_2$
on the feasible domain $(0, \pi_1]$, with interior engagement strictly
increasing in $\pi_2$. As $\pi_2$ falls from $\pi_1$ toward zero, the
worker's period-1 engagement either remains constant or decreases.
\end{corollary}
As AI becomes more reliable, the value of preserving fallback skill declines, and the firm withdraws engagement: if AI will rarely fail in the future, workers will rarely need to step in, so the firm has less reason to maintain worker skill today. Future capability, by contrast, has no such effect. Mathematically, $\alpha_2$ enters the continuation payoff only as an additive constant, so it drops out of the first-order condition for period-1 engagement. The reason is intuitive: because the firm disengages below-benchmark workers in the terminal period, worker skill affects period-2 production only when AI fails. A more capable AI generates higher effective throughput when it functions, but that output does not depend on the worker's skill. The marginal value of worker skill in period 2 is therefore governed by how often AI fails ($\pi_2$), not by AI's performance level when it functions ($\alpha_2$). \cref{fig:skill_targeting} plots this policy: the next-period skill it produces against the worker's current skill.

\begin{figure}[ht]
\centering
\includegraphics[width=0.6\textwidth]{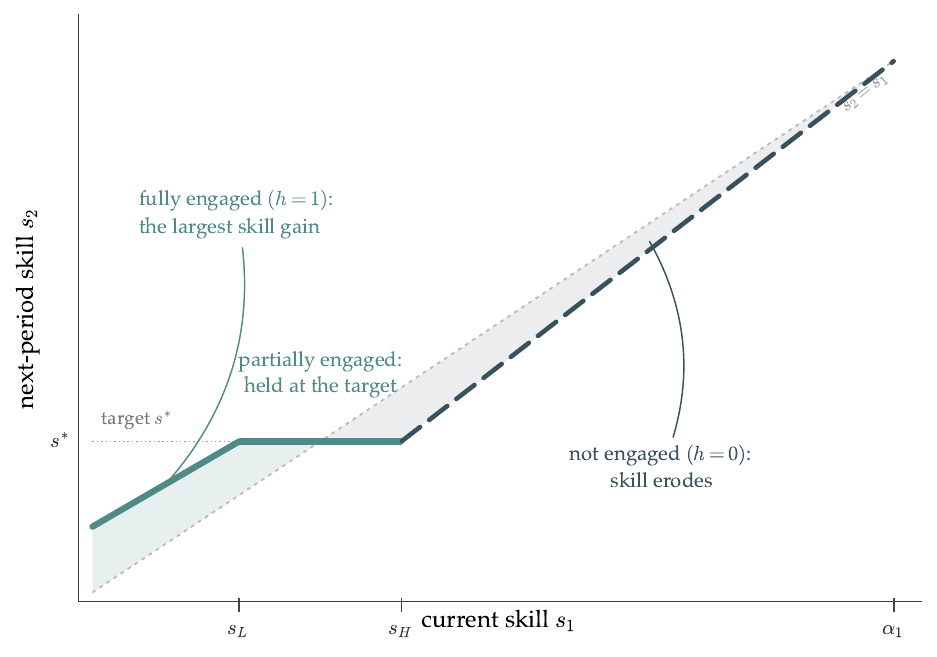}
\caption{Single-firm engagement targets the least-skilled workers. The figure traces the next-period skill $s_2$ produced by optimal period-1 engagement against current skill $s_1$. The firm engages the lowest-skilled workers fully ($h_1^{\mathrm{sf}}=1$), building skill toward the target $s^*$; intermediate workers partially, holding skill at $s^*$; and the highest-skilled not at all ($h_1^{\mathrm{sf}}=0$), letting skill erode toward the $45^\circ$ line. The shaded band is the skill built by engagement, widest for the least-skilled. Parameter values: $W(s)=s^{1.2}$, $\alpha_1=0.5$, $\alpha_2=0.6$, $\pi_1=0.45$, $\pi_2=0.35$, $\phi=0.5$, $\gamma=0.3$, $\delta=0.05$, $\lambda R=3.65$, $\beta=0.95$.}
\label{fig:skill_targeting}
\end{figure}
This targeting of the least-skilled is not special to the parameters of \cref{fig:skill_targeting}.

\begin{proposition}
\label{prop:fallback_targets_bottom}
In the absence of worker mobility, the firm's engagement is nonincreasing in initial skill: it allocates the most engagement to the lowest-skilled workers in the feasible domain, partial engagement to intermediate workers so as to bring next-period skill exactly to $s^*$, and no engagement to workers whose future skill remains above the target even under atrophy.
\end{proposition}
Intuitively, the firm engages only to secure a fallback for when AI is unavailable or unreliable, and that fallback is least expensive to build from below. Two forces pull the same way. The skill a unit of engagement adds is proportional to the gap $\alpha_1 - s_1$, so the firm gains more skill per unit of engagement from a worker who starts further down. Because wages are convex, the marginal wage it pays to build that skill is also lower the lower the worker starts. An already-skilled worker offers neither saving, so the firm withholds engagement.

The firm engages to secure adequate fallback capability at minimum cost, not to maximize skill. Because AI handles most of the work, building a worker far above the target earns little return at a higher wage. The firm aims only for the level at which a worker can take over when AI is unavailable or unreliable, and no higher.

\section{Worker Mobility}\label{sec:competition}

Worker mobility adds a second motive to the firm's engagement choice. With two ex ante identical firms and a common market wage, workers choose between jobs on the basis of the skill trajectories those jobs build. Engagement now affects profits in two ways. It lowers the current margin on a below-benchmark worker, as in the single-firm benchmark, but it also makes the firm more attractive to workers who value the skill the job creates. This sorting motive, absent without mobility because a worker has nowhere else to go, comes at a price: it weakens the firm's hold on the skill it builds, because that skill is portable and moves with the worker. The sorting motive operates even in the terminal period, in which a single firm disengages every below-benchmark worker, and it targets the higher-skilled workers, the opposite end of the spectrum from the least-skilled workers a single firm favors. Under mobility, AI capability and reliability move engagement in different directions. In period~1, the sorting motive interacts with the fallback motive. When workers respond weakly to the skill a job builds, mobility drains the firm's return on portable skill and pulls engagement below the single firm's level. The analysis also illustrates three further possibilities: engagement can rise above that level, its timing can reverse so that a worker is engaged only after re-sorting, and two ex ante identical firms can split apart.

In each period $t\in\{1,2\}$, the two firms simultaneously choose engagement levels $h_t^j,h_t^{-j}\in[0,1]$, and workers then sort between them. A worker of skill $s_t$ who joins firm $j$ is engaged at level $h_t^j$ and enters the next period with skill $g(s_t,h_t^j;A_t)$. Let $V_t^j(s_t,h_t^j)$ denote the worker's economic value from joining firm $j$. Because both firms pay the common market wage $W(s_t)$, differences in $V_t^j$ come only from the skill trajectory each engagement policy creates. A worker's realized valuation of firm $j$ is this value plus a private taste shock, $U_t^j(s_t,h_t^j)=V_t^j(s_t,h_t^j)+\eta\varepsilon_t^j$, where the shocks $\varepsilon_t^j$ are drawn independently across workers and firms from a type-I extreme-value distribution. The scale $\eta>0$ governs how strongly workers respond to the skill trajectories firms offer: as $\eta$ falls, they track those trajectories ever more closely, and as it rises, private taste dominates. This structure yields the familiar \emph{logit share} \citep{aksoypierson2013,basuroy1998,mcfadden1974}: a worker of skill $s_t$ joins firm $j$ with probability
\begin{align*}
\sigma_t^j(h_t^j,h_t^{-j};s_t,A_t)
=
\frac{1}{1+\exp\!\left(\bigl[V_t^{-j}(s_t,h_t^{-j})-V_t^j(s_t,h_t^j)\bigr]/\eta\right)},
\end{align*}
and, before the taste shock, the worker's expected value
$\mathbb{E}\!\left[\max\{U_t^1,\,U_t^2\}\right]$ takes the log-sum form
\begin{align}\label{eq:worker_expected_value}
V_t(s_t,h_t^1,h_t^2)
=
\eta \log\!\left[
\exp\!\left(V_t^1(s_t,h_t^1)/\eta\right)
+
\exp\!\left(V_t^2(s_t,h_t^2)/\eta\right)
\right]
+\eta\kappa,
\end{align}
the standard type-I extreme-value result \citep{mcfadden1980}, with
$\kappa$ the Euler--Mascheroni constant. 
The economic value $V_t^j$ takes a different form in each period. We begin with the terminal period, in which the firm's own fallback motive is absent, so any engagement it offers can serve only to attract workers, isolating the sorting motive. We take up the period-1 problem, in which current output, future skill, and worker sorting all interact, in~\cref{sec:period1_duopoly}.

\subsection{Terminal Period: The Sorting Motive}\label{sec:terminal_duopoly}

In the terminal period, a firm without worker mobility has no reason to engage a below-benchmark worker: with no future left to protect, engagement only sacrifices current output, so it disengages entirely. Worker mobility reverses this. Once a worker can choose between firms, a firm engages even though doing so earns it no fallback value of its own, because engagement raises the worker's post-horizon skill, and a worker who anticipates a stronger skill trajectory is more likely to join that firm. Engagement thus becomes a tool to attract workers. Two questions then organize the terminal analysis: whether a firm engages a below-benchmark worker at all, settled here by a single threshold on the sorting friction, and which workers it engages most, taken up in \cref{sec:targeting_duopoly}.

The worker's economic value $V_2^j$ is the current
wage plus the discounted post-horizon value of the skill carried
forward,
\begin{align*}
V_2^j(s_2,h_2^j)
=
W(s_2)
+
\beta B\!\bigl(g(s_2,h_2^j;A_2)\bigr),
\end{align*}
where $B(\cdot)$ is the post-horizon skill-value function introduced in \cref{sec:worker_utility}. Because both firms pay the common wage, it
cancels out in the logit sorting probabilities, so a firm's terminal-period logit share depends only on the skill its engagement promises:
\begin{align*}
\sigma_2^j(h_2^j,h_2^{-j};s_2,A_2)
=
\frac{1}{1+\exp\!\left(\beta\bigl[B(g(s_2,h_2^{-j};A_2))-B(g(s_2,h_2^j;A_2))\bigr]/\eta\right)}.
\end{align*}
The firm that builds more skill draws the larger logit share, increasingly so as $\eta$ falls.

For a below-benchmark worker ($s_2<\alpha_2$), firm $j\in\{1,2\}$ chooses $h_2^j\in[0,1]$ to maximize expected profit, the product of its worker share and its per-worker margin,
\begin{align*}
\Pi_2^j(h_2^j, h_2^{-j}; s_2, A_2)
=
\sigma_2^j(h_2^j, h_2^{-j}; s_2, A_2)\,
M(s_2, h_2^j; A_2).
\end{align*}
The firm is pulled in two directions. A higher $h_2^j$ lifts the
worker's future skill, and with it the firm's share $\sigma_2^j$; but
because the worker sits below the benchmark, the same $h_2^j$ erodes
the current margin $M(s_2,h_2^j;A_2)$. The best response weighs
attracting workers against current output.

Two ex ante identical firms could, in principle, support an asymmetric
pure-strategy equilibrium in the terminal-period game, with one firm
engaging more than the other. The next lemma rules~that~out.

\begin{lemma}
\label{lem:terminal_symmetry}
In the terminal period, any pure-strategy Nash equilibrium of the game in
which firms $1$ and $2$ simultaneously choose $h_2^1,h_2^2\in[0,1]$ to
maximize $\Pi_2^1(h_2^1,h_2^2;s_2,A_2)$ and $\Pi_2^2(h_2^2,h_2^1;s_2,A_2)$
is symmetric: $h_2^{1*}=h_2^{2*}$.
\end{lemma}
\cref{lem:terminal_symmetry} reflects a balance of two forces. A firm
that engages more attracts a larger share of workers but earns a thinner
margin on each; a firm that engages less faces the reverse. These two
pulls offset, so neither firm can gain by engaging differently, and any gap between them unravels. The balance holds only in the
terminal period. In period~1, by contrast, skill moves with the worker, so once workers re-sort the other firm can draw on it without having paid to build it, and this spillover can break the symmetry between them (\cref{sec:period1_duopoly}).

Because \cref{lem:terminal_symmetry} rules out asymmetric pure-strategy
equilibria in the terminal period, it suffices to characterize symmetric
candidates. At a symmetric profile \(h_2^j=h_2^{-j}=h\), the own-action
marginal payoff $\partial \Pi_2^j(h_2^j, h_2^{-j}; s_2, A_2)/\partial h_2^j$ of firm \(j\) has the same sign as
\(F_2(h;s_2,A_2)-\lambda R\delta\), where
\begin{align*}
F_2(h;s_2,A_2)
\triangleq
\frac{\beta(\phi\pi_2+\gamma)}{2\eta}\,
B'\!\bigl(g(s_2,h;A_2)\bigr)\,
M(s_2,h;A_2).
\end{align*}
Thus an interior symmetric equilibrium solves
\(F_2(h;s_2,A_2)=\lambda R\delta\). The term \(F_2(h;s_2,A_2)\) is the marginal
engagement gain: higher engagement builds more
post-horizon skill and makes the firm more attractive to workers. The
term \(\lambda R\delta\) is the normalized marginal operating cost of
engagement for a below-benchmark worker.

We impose a sufficient single-crossing restriction that makes this
comparison monotone. There exists a cutoff
\(\bar{\bar{s}}_2\in[\underline{s}_2,\alpha_2)\) such that, with
\(D_2\equiv(\bar{\bar{s}}_2,\alpha_2)\), for every \(s_2\in D_2\),
\begin{align}\label{eq:monotone_condition}
\max_{h \in [0,1]}
\left\{
\frac{B''(g(s_2, h; A_2))}{B'(g(s_2, h; A_2))}\,
M(s_2, h; A_2)
\right\}
<\frac{\lambda R \delta}{\phi\pi_2 + \gamma}.
\end{align}
We call \cref{eq:monotone_condition} the terminal single-crossing
condition. Engagement raises future skill, and convex \(B\)
makes that skill increasingly valuable; but engagement also lowers the
firm's current margin. The condition ensures the margin decline dominates, so the marginal engagement gain
falls with engagement and meets the cost \(\lambda R\delta\) at most once. Throughout the rest of \cref{sec:competition}, terminal-period statements are understood for
\(s_2\in D_2\).

\begin{proposition}
\label{prop:duopoly_terminal}
Suppose the terminal single-crossing condition in
\cref{eq:monotone_condition} holds. Then the
engagement gain \(F_2(h;s_2,A_2)\) is strictly decreasing in \(h\), and the
terminal-period two-firm game has a unique pure-strategy Nash equilibrium,
symmetric at \(h_2^*(s_2)\), found by comparing the engagement gain at the
boundaries with the operating cost \(\lambda R\delta\):
\begin{enumerate}
  \item[\textup{(i)}] if \(\lambda R\delta\le F_2(1;s_2,A_2)\), both firms
        engage fully, \(h_2^*=1\);
  \item[\textup{(ii)}] if \(F_2(1;s_2,A_2)<\lambda R\delta<F_2(0;s_2,A_2)\),
        both engage at the interior level \(h_2^*\in(0,1)\) solving
        \(F_2(h_2^*;s_2,A_2)=\lambda R\delta\);
  \item[\textup{(iii)}] if \(F_2(0;s_2,A_2)\le\lambda R\delta\), neither firm
        engages, \(h_2^*=0\).
\end{enumerate}
\end{proposition}
\cref{prop:duopoly_terminal} shows worker mobility creates a
terminal-period sorting motive even though engagement lowers the
current margin. At a symmetric profile, engaging more wins workers but
loses margin, whereas engaging less saves margin but loses workers.

For the power wage \(B(s)=W(s)=s^b\), \(b>1\), which we use in several results below, the single-crossing condition becomes explicit and fixes the domain \(D_2\). The left-hand side of \cref{eq:monotone_condition} is then largest at \(h=0\) and, under \cref{ass:wages_revenue}, strictly decreasing in \(s_2\): it diverges as \(s_2\downarrow\underline{s}_2\) and falls toward the benchmark, so the condition fails for the least-skilled workers and holds for the most-skilled. The domain \(D_2=(\bar{\bar{s}}_2,\alpha_2)\) is therefore an upper-tail band of below-benchmark skills just under the AI frontier, where the cutoff \(\bar{\bar{s}}_2\) is the unique skill at which the single-crossing condition holds with equality; \Cref{app:power_single_crossing_cutoff} in the online appendix gives its closed-form characterization.\footnote{The restriction is an upper-tail modest-convexity condition, consistent with Mincerian wage profiles, which are convex but typically only modestly so \citep{mincer1974,psacharopoulos2018}. In a power-wage calibration, \(D_2\) covers about \(73\%\) of the below-benchmark interval (\cref{app:d2_size}). Outside \(D_2\), \(F_2\) need not be monotone and the terminal equilibrium need not be unique; \(D_2\) is the region on which single-crossing yields a unique equilibrium.}

\cref{prop:duopoly_terminal} states the no-engagement boundary in general
form as \(F_2(0;s_2,A_2)\le\lambda R\delta\). Under the power wage,
\(F_2(0;s_2,A_2)\) is proportional to \(1/\eta\), so this boundary can be
written as a threshold condition on the sorting friction. Define
\begin{align}\label{eq:eta_cutoff_terminal}
\bar\eta(s_2)
\triangleq
\frac{\beta b(\phi\pi_2+\gamma)}{2\lambda R\delta}
\bigl(g(s_2,0;A_2)\bigr)^{b-1}
M(s_2,0;A_2).
\end{align}
Equivalently,
\(F_2(0;s_2,A_2)=\lambda R\delta\,\bar\eta(s_2)/\eta\).

\begin{corollary}
\label{cor:terminal_cutoff}
Suppose \(B(s)=W(s)=s^b\) with \(b>1\). For \(s_2\in D_2\), \(h_2^*(s_2)>0\) if and only if \(\eta<\bar\eta(s_2)\).
\end{corollary}

Here \(\bar\eta(s_2)\) is the critical sorting-friction threshold for
engaging a worker of skill \(s_2\): the firm engages for
\(\eta<\bar\eta(s_2)\) and disengages for \(\eta\ge\bar\eta(s_2)\). Thus
\(\bar\eta(s_2)\) is the engagement gain at \(h=0\), expressed in
units of the sorting friction. As workers grow more responsive, that is,
as \(\eta\) falls, engagement becomes profitable for a wider set of
skills within \(D_2\).

A single firm's workers cannot move to a rival, so it has no sorting motive; it engages a below-benchmark
worker in period~1 only when the fallback motive, the prospect of a
more productive worker in period~2, justifies the cost. With no future
left, even that motive is absent, so a single firm sets \(h_2^{\mathrm{sf}}=0\).
Under mobility, by contrast, a firm engages a worker with
\(s_2\in D_2\) whenever the engagement gain covers the cost,
\(F_2(0;s_2,A_2)>\lambda R\delta\). Engagement is then no longer a
purely operational choice. The question is not whether a firm engages,
but whom.

\subsection{The Engagement Reversal}
\label{sec:targeting_duopoly}
Mobility reverses which workers a firm engages. Without it, the fallback motive engages the least-skilled workers (\cref{sec:monopoly}): far below the AI benchmark, a unit of engagement closes the widest skill gap and buys fallback capability at the lowest wage. In the terminal period the fallback motive is gone, and the sorting motive stands alone. It points the other way, toward the high-skill workers near the AI frontier, for two reasons. First, engaging a below-benchmark worker forgoes operational revenue in proportion to the skill gap \(\alpha_2-s_2\), so workers nearer the benchmark are less costly to engage, a force that points upward even when wages are linear. Second, convex wages reinforce the pull: the marginal wage \(B'(s)\) rises with skill, so the portable skill that engagement builds is worth more to a higher-skill worker. A margin force works against both: higher-skill workers also command higher current wages, compressing the firm's per-worker margin and dampening the gain from attracting more of them. How strongly the net pull translates into engagement depends on the sorting friction \(\eta\): the lower \(\eta\), the more responsive workers are to the skill trajectories firms offer. Mobility therefore shifts engagement from the bottom of the skill range toward an upper band that, for low enough \(\eta\), can reach the frontier~(\cref{fig:reversal}).

\begin{figure}[t]
  \centering
  \includegraphics[width=0.55\textwidth]{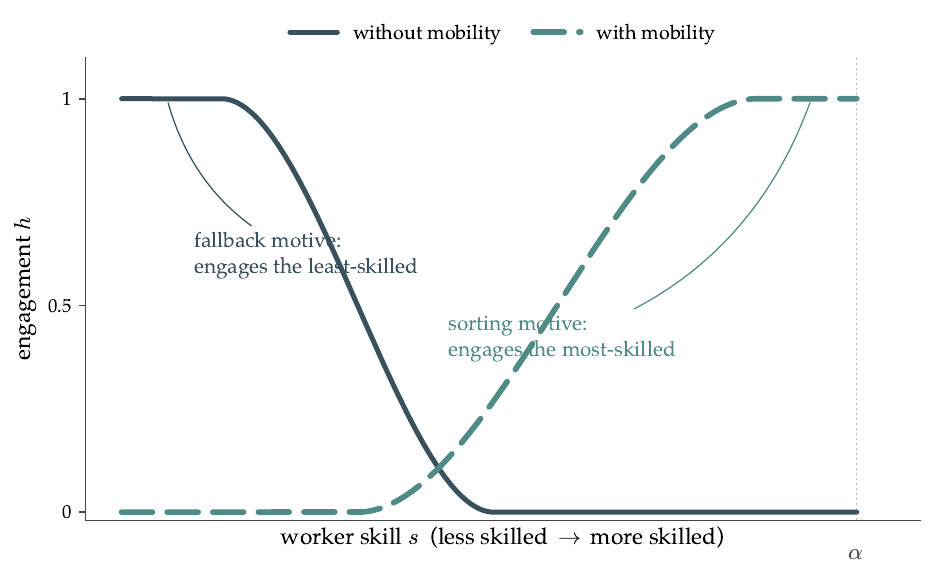}
  \caption{The engagement reversal. Without mobility, the fallback motive engages the least-skilled workers most; under mobility, the sorting motive engages the most skilled below-benchmark workers. The figure is schematic.}
  \label{fig:reversal}
\end{figure}

Under the power wage \(B(s)=W(s)=s^b\), a firm engages a worker of skill \(s_2\in D_2\) when the sorting friction falls below a skill-specific cutoff, \(\eta<\bar\eta(s_2)\) (\cref{cor:terminal_cutoff}). For a given \(\eta\), the engaged workers are those whose cutoff \(\bar\eta(s_2)\) exceeds \(\eta\), so which workers a firm engages is determined by how that cutoff varies with skill. The cutoff and the engagement gain \(F_2(0;s_2,A_2)\) are proportional, because \(F_2(0;s_2,A_2)=\lambda R\delta\,\bar\eta(s_2)/\eta\) and the factor \(\lambda R\delta/\eta\) does not depend on \(s_2\). The two therefore rise, fall, and peak together, so the targeting problem reduces to a single object: how the engagement gain at \(h=0\) varies with skill. That variation settles both which workers are engaged and how intensely.

Whether the workers a firm engages most reach the AI frontier or sit just below it is governed by a single comparison, summarized by the scalar \(\Delta_2\). Engagement requires \(\eta<\bar\eta(s_2)\), so at a given \(\eta\) the engaged workers are those with the highest cutoffs, an interval at the top of the admissible range. Where this interval sits, at the very top or in a band just below the frontier, depends on the shape of the cutoff \(\bar\eta(s_2)\), and that shape is determined by the sign of \(\Delta_2\):
\begin{align}\label{eq:delta2_def}
\Delta_2 \triangleq (b-1)[1+\gamma(1-\pi_2)](\lambda R\alpha_2-\alpha_2^b)+\alpha_2(\lambda R\pi_2-b\alpha_2^{b-1}).
\end{align}

Intuitively, \(\Delta_2\) weighs the current-period value of a worker's skill near the AI frontier against the wage it commands there. When that value keeps pace with the rising wage up to the frontier, the engaged interval reaches it; when the wage overtakes it first, the most-engaged workers sit in a band just below. The following proposition makes this precise.

\begin{proposition}
\label{prop:terminal_targeting}
Suppose \(B(s)=W(s)=s^b\) with \(b>1\). Fix \(\eta\). On the admissible domain \(D_2\), the engaged set \(E_2(\eta)\equiv\{s_2\in D_2:\eta<\bar\eta(s_2)\}\) is an interval that widens as \(\eta\) falls, beginning from the workers with the highest cutoff \(\bar\eta(s_2)\). When \(\Delta_2\ge0\), \(\bar\eta(s_2)\) rises monotonically over \((\underline{s}_2,\alpha_2)\) and is largest at the benchmark \(\alpha_2\); when \(\Delta_2<0\), \(\bar\eta(s_2)\) is single-peaked, reaching an interior maximum at some \(s_2^\dagger\in(\underline{s}_2,\alpha_2)\).
\end{proposition}
The sign of \(\Delta_2\) determines where engagement is most intense. When \(\Delta_2\ge0\), the cutoff is highest at the AI frontier, so the most-engaged workers sit at the top of the admissible range; when \(\Delta_2<0\), the cutoff peaks at an interior skill \(s_2^\dagger\), and they sit in a band just below the frontier. \Cref{app:terminal_targeting_detail} gives \(s_2^\dagger\) and the endpoints of the engaged interval in each case.

Under one further condition, \(\lambda R\pi_2\ge b\alpha_2^{b-1}\), engagement rises monotonically with skill across \(D_2\): the firm engages the most-skilled admissible workers most. The condition rules out an offsetting channel. A higher skill \(s_2\) raises the engagement gain \(F_2\) by lifting the skill trajectory engagement builds, but it also raises the wage and so could erode the firm's current margin. The condition ensures the first effect dominates: even at the benchmark, where the wage slope is steepest, the extra current revenue a more skilled worker earns when AI fails covers the higher wage. The gain \(F_2(h;s_2,A_2)\) is then increasing in \(s_2\) at every engagement level \(h\). Combined with the single-crossing property, which makes \(F_2\) decreasing in \(h\), a higher skill shifts the whole marginal-gain curve upward without disturbing its single crossing with the cost \(\lambda R\delta\). Equilibrium engagement therefore rises weakly with skill: zero for the least-skilled admissible workers, then interior, then full near the frontier.

\begin{corollary}
\label{cor:terminal_monotone}
Suppose \(B(s)=W(s)=s^b\) with \(b>1\), and suppose \(\lambda R\pi_2\ge b\alpha_2^{b-1}\). Then, on \(D_2\), the gain to attracting workers, \(F_2(h;s_2,A_2)\), is decreasing in engagement and increasing in skill, and equilibrium engagement \(h_2^*(s_2)\) is continuous and nondecreasing in skill. Two thresholds \(\bar{\bar{s}}_2\le\hat{s}_2^D\le\overline{s}_2^D\le\alpha_2\) partition the admissible workers into three consecutive bands: engagement is zero for \(s_2\le\hat{s}_2^D\), interior for \(\hat{s}_2^D<s_2<\overline{s}_2^D\), and full for \(s_2\ge\overline{s}_2^D\), rising with skill across the interior band. The lower threshold \(\hat{s}_2^D\) is the skill at which the engagement gain at zero engagement first covers the cost \(\lambda R\delta\), so below it no worker is engaged; the upper threshold \(\overline{s}_2^D\) is the skill at which the gain at full engagement covers the cost, so above it every worker is engaged fully.
\end{corollary}
The precise threshold definitions, and the continuity and differentiability of \(h_2^*(s_2)\), are given in the proof of \cref{cor:terminal_monotone}.

This is the engagement reversal. In the single-firm benchmark, engagement falls with skill because the firm is securing fallback capacity at the lowest cost (\cref{prop:fallback_targets_bottom}). Under mobility, engagement rises weakly with skill on \(D_2\) because the firm is competing through the skill trajectory the job offers. The same choice variable serves a different economic purpose.

\subsection{Worker Mobility Separates the Effects of AI Capability and Reliability}
\label{sec:ai_dimensions_duopoly}

Worker mobility also separates the effects of AI capability and AI reliability. Without mobility, the two dimensions act differently: capability does not affect engagement at all, because a single firm values a worker's skill only as a fallback for AI failures, although higher reliability reduces engagement monotonically (\cref{cor:reliability_vs_capability}). Mobility changes both. Because a firm now also engages to attract workers, capability matters: it raises the future skill a unit of engagement builds, which makes the job more attractive. Reliability, by contrast, moves two margins at once. It changes how much skill a unit of engagement builds, but it also changes the firm's current operating margin on a below-benchmark worker, and these two forces can point in opposite directions. The net effect on engagement can therefore be non-monotone. We state the comparative statics for the power wage \(B(s)=W(s)=s^b\) with \(b>1\) and a worker of skill \(s_2\in D_2\) at which the terminal-period equilibrium is interior, \(h_2^*(s_2)\in(0,1)\).

Capability and reliability diverge under mobility: a more capable AI moves engagement in one direction, although a more reliable AI can move it either way. The next result makes the split precise for the power wage. The proof of \cref{prop:alpha_effect_competition} in the online appendix defines threshold values \(\bar\pi^+\) and \(\bar\pi^-\) that separate the range in which a marginal increase in unreliability raises engagement from the range in which it lowers engagement.

\begin{proposition}
\label{prop:alpha_effect_competition}
Suppose the terminal single-crossing condition in
\cref{eq:monotone_condition} holds and
\(s_2\in D_2\) is such that the terminal-period equilibrium of
\cref{prop:duopoly_terminal} is interior, \(h_2^*(s_2)\in(0,1)\).
\begin{enumerate}
\item[\textup{(i)}] A higher AI capability raises engagement: for local
changes in \(\alpha_2\) that keep \(s_2\in D_2\) and keep the equilibrium
interior, \(\partial h_2^*(s_2)/\partial\alpha_2>0\).\footnote{Part~\textup{(i)} does not rely on the power wage; it holds for any \(B\) satisfying the terminal single-crossing condition.}

\item[\textup{(ii)}]
\label{prop:pi_comparative_statics}
Suppose additionally that \(B(s)=W(s)=s^b\) with \(b>1\). For local changes in \(\pi_2\) that keep \(s_2\in D_2\) and keep the
equilibrium interior, engagement rises with AI unreliability,
\(\partial h_2^{*}(s_2,\pi_2)/\partial\pi_2>0\), whenever
\(\pi_2<\bar\pi^+\) and \(\pi_2\le\tfrac12\), and falls,
\(\partial h_2^{*}(s_2,\pi_2)/\partial\pi_2<0\), whenever
\(\pi_2>\bar\pi^-\).
\end{enumerate}
\end{proposition}

Capability moves engagement in one direction: under mobility, a more capable AI raises engagement, because it improves the skill trajectory the firm offers. Reliability moves engagement in two directions at once, and its net effect splits into three regimes. When AI is highly reliable, a marginal increase in unreliability raises engagement: it makes a unit of engagement build more skill, strengthening the sorting value of the skill the worker carries. When AI is already unreliable, further unreliability lowers engagement: frequent failures erode the current margin by more than the added skill is worth, especially for workers far from the benchmark. Between these regimes the two channels offset. Reliability can therefore move engagement non-monotonically, in contrast to the monotone decline without mobility, producing the hump-shaped pattern in \cref{fig:h_star_vs_pi} on the interior of the curve, on which \(s_2\in D_2\) and \(h_2^*(s_2)\in(0,1)\).
\begin{figure}[ht]
\centering
\includegraphics[width=0.50\textwidth]{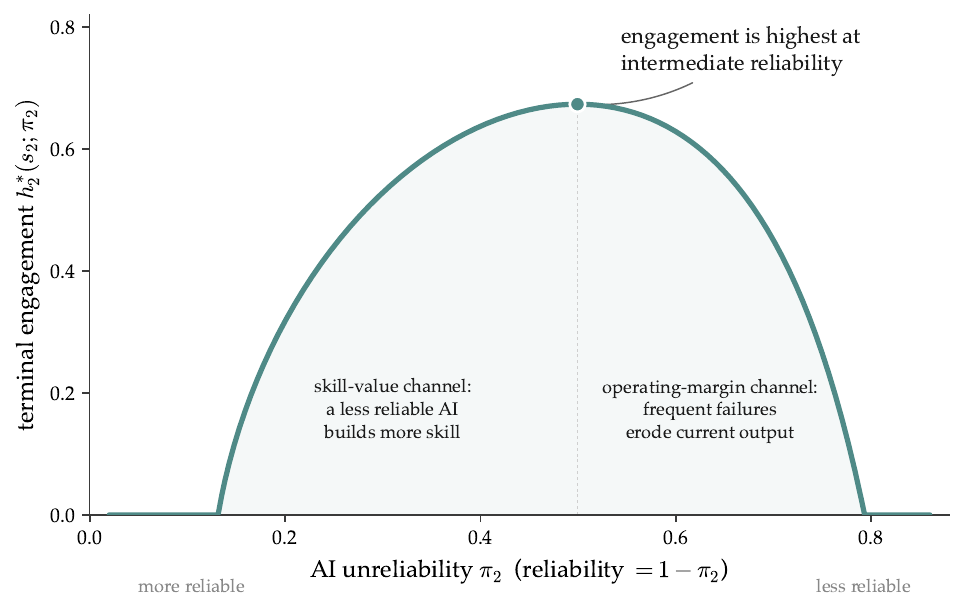}
\caption{Engagement can be highest at intermediate reliability. Equilibrium engagement \(h_2^*(s_2;\pi_2)\) under mobility, plotted against AI unreliability \(\pi_2\) (so reliability is \(1-\pi_2\)): as reliability falls from its highest level, a unit of engagement builds more skill and engagement rises, until frequent failures erode current output and engagement falls again, so it peaks near \(\pi_2\approx0.5\). Parameter values: \(B(s)=W(s)=s^{1.2}\), \(s_2=0.30\), \(\alpha_2=0.90\), \(\phi=0.5\), \(\gamma=0.3\), \(\delta=0.5\), \(\lambda R=3.0\), \(\beta=0.95\), \(\eta=0.21\).}
\label{fig:h_star_vs_pi}
\end{figure}

The hump shape relies on the nature of learning, not on a specific functional form. Building skill requires both working with a functioning AI and occasionally handling the fallback, so learning is strongest at intermediate reliability and weaker at either extreme. The hump follows from this two-sided pattern, not from the exact product \(\pi_t(1-\pi_t)\): \cref{app:learning_profile_robustness} reproduces it for a range of two-sided learning profiles and shows that it weakens or disappears when learning is one-sided, drawing on AI exposure or fallback practice alone.

\subsection{Period 1: Mobility and the Fallback Motive}
\label{sec:period1_duopoly}
Period~1 combines the two forces the earlier subsections separated. Engagement now lowers current output, builds fallback skill for the period-2 state in which AI fails, and shapes how workers sort across firms. Because that skill is portable, the other firm can draw on it after workers re-sort in period~2. Relative to the single firm, mobility therefore adds a sorting gain and an appropriability loss, and which force dominates depends on how readily workers respond to the skill trajectories firms offer.

For the main text it suffices to study initial skills \(s_1\in D_1\), the region for which every continuation state reachable under a feasible period-1 engagement lies in the terminal domain \(D_2=(\bar{\bar{s}}_2,\alpha_2)\), so that \cref{prop:duopoly_terminal} delivers a unique symmetric terminal equilibrium \(h_2^*(s_2)\) at every reachable state. The explicit cutoff \(\bar{\bar{s}}_1\) defining \(D_1\) is reported in \cref{app:period1_domain_size}, which also shows that this restriction is not narrow: across power-wage examples satisfying the model assumptions and terminal-domain conditions, \(D_1\) is nonempty in \(95.3\%\) of cases and covers a median \(81.0\%\) of the feasible period-1 skill interval \((\underline{s},\alpha_1)\).

A worker who joins firm $j$ enters period~2 with skill
$g(s_1,h_1^j;A_1)$ and is then engaged at the symmetric equilibrium level
$h_2^*\!\bigl(g(s_1,h_1^j;A_1)\bigr)$ of \cref{prop:duopoly_terminal}. The
firm's period-2 profit, half the margin at that equilibrium, is
\begin{align}\label{eq:Pi2_star_def}
\Pi_2^*\!\bigl(g(s_1,h_1^j;A_1),A_2\bigr)
\triangleq
\frac{1}{2}M\!\Bigl(g(s_1,h_1^j;A_1),\,
h_2^*\!\bigl(g(s_1,h_1^j;A_1)\bigr);A_2\Bigr),
\end{align}
and the worker's period-2 value, the log-sum
in~\cref{eq:worker_expected_value} at that equilibrium, is
\begin{align*}
\bar V_2\!\bigl(g(s_1,h_1^j;A_1),A_2\bigr)
=
W\!\bigl(g(s_1,h_1^j;A_1)\bigr)
+\beta B\!\bigl(\zeta(s_1,h_1^j;A_1,A_2)\bigr)
+\eta\log 2,
\end{align*}
up to the type-I extreme-value normalization constant, where
\begin{align*}
\zeta(s_1,h_1;A_1,A_2)
\triangleq
g\bigl(g(s_1,h_1;A_1),\,h_2^*(g(s_1,h_1;A_1));A_2\bigr)
\end{align*}
is the post-horizon skill the worker carries out of the model, as shaped
by period-1 engagement.

A worker choosing a firm in period~1 weighs the current wage against
the value expected from period~2 onward. The wage $W(s_1)$ is common
to both firms and cancels in the logit, exactly as in the terminal
period. The worker's continuation value differs across firms only through the skill carried into period~2. Because the worker re-sorts at the start of period~2 and the terminal equilibrium at any reachable skill is symmetric (\cref{prop:duopoly_terminal}), that continuation value is the log-sum of~\cref{eq:worker_expected_value} evaluated at the terminal equilibrium, namely $\bar V_2$ as defined above. The worker sorts in period~1 on these period-2 values, choosing firm $j$ with probability
\begin{align*}
\sigma_1^j(h_1^j,h_1^{-j};s_1,A_1,A_2)
=
\frac{1}{1+\exp\!\left(\beta\Bigl[
\bar V_2\!\bigl(g(s_1,h_1^{-j};A_1),A_2\bigr)
-\bar V_2\!\bigl(g(s_1,h_1^{j};A_1),A_2\bigr)\Bigr]/\eta\right)}.
\end{align*}
Regardless of which firm the worker chose in period~1, re-sorting at the start of period~2 means both firms engage the worker identically; each firm therefore earns the period-2 profit $\Pi_2^*$ of~\cref{eq:Pi2_star_def}, namely half the margin at the terminal equilibrium, irrespective of which firm built the skill the worker arrives with. Period-1
engagement thus determines who collects the current margin and what
skill the worker carries forward, but not how the period-2 profit on the worker
is divided. If the worker joins firm $j$, the firm earns the current
margin $M(s_1,h_1^j;A_1)$ and the discounted period-2 profit
$\beta\Pi_2^*\!\bigl(g(s_1,h_1^j;A_1),A_2\bigr)$; if the worker joins the
other firm, it earns no current margin but still collects
$\beta\Pi_2^*\!\bigl(g(s_1,h_1^{-j};A_1),A_2\bigr)$ on the skill the
other firm built. Firm $j$ therefore chooses $h_1^j$ to maximize
\begin{align}\label{eq:period1_objective}
\Phi_1^j(h_1^j,h_1^{-j};s_1,A_1,A_2)
\triangleq{}&
\sigma_1^j(h_1^j,h_1^{-j};s_1,A_1,A_2)
\Bigl[M(s_1,h_1^j;A_1)
+\beta\,\Pi_2^*\!\bigl(g(s_1,h_1^j;A_1),A_2\bigr)\Bigr] \notag\\
&+\bigl[1-\sigma_1^j(h_1^j,h_1^{-j};s_1,A_1,A_2)\bigr]\,
\beta\,\Pi_2^*\!\bigl(g(s_1,h_1^{-j};A_1),A_2\bigr).
\end{align}

Whereas the terminal period has a unique symmetric equilibrium (\cref{lem:terminal_symmetry}), period~1 admits both symmetric and asymmetric equilibria: because skill moves with the worker, a firm can let the other firm build it and draw on it after workers re-sort, which can pull the two firms' engagement apart. We treat the two cases in turn. This subsection characterizes the symmetric equilibrium and compares it to the single-firm benchmark; \cref{sec:asymmetry_duopoly} shows by example when ex ante identical firms instead specialize into a skill-builder and a free-rider.

Let \(F_1(h_1;s_1,A_1,A_2)\) denote the derivative of
\(\Phi_1^j(h_1^j,h_1^{-j};s_1,A_1,A_2)\) in its own action \(h_1^j\),
evaluated at the symmetric profile \(h_1^j=h_1^{-j}=h_1\). Unlike the
terminal-period gain \(F_2\), which was set against the separate
operating cost \(\lambda R\delta\), \(F_1\) is the entire marginal
payoff: it already embeds the operating cost of engaging a below-benchmark
worker, the sorting gain from a more attractive skill trajectory, and the
continuation-profit effect. Because these enter with different weights, no
common factor divides out to leave a constant cost on one side, as it did
in period~2. The interior condition is therefore \(F_1=0\), the period-1
counterpart of \(F_2=\lambda R\delta\), and a symmetric equilibrium must
leave neither firm a profitable deviation in its own engagement, which
characterizes the candidates by the sign of \(F_1\) at the boundaries and its roots
in the interior.

\begin{lemma}
\label{lem:period1_candidates}
Consider \(s_1\in D_1\), and let
\(s_2^*\equiv g(s_1,h_1^*;A_1)\). At any symmetric profile
\(h_1^j=h_1^{-j}=h_1^*\) at which \(\bar V_2(s_2,A_2)\) and
\(\Pi_2^*(s_2,A_2)\) are differentiable in \(s_2\) at \(s_2=s_2^*\), a symmetric equilibrium must satisfy \(F_1(h_1^*;s_1,A_1,A_2)(h_1-h_1^*)\le0\) for all
\(h_1\in[0,1]\); equivalently, \(h_1^*=0\) with
\(F_1(0;s_1,A_1,A_2)\le0\), or \(h_1^*=1\) with
\(F_1(1;s_1,A_1,A_2)\ge0\), or \(h_1^*\in(0,1)\) with
\(F_1(h_1^*;s_1,A_1,A_2)=0\).
\end{lemma}

The lemma gives conditions that any symmetric equilibrium must satisfy. These conditions are necessary but not sufficient: a profile meeting them need not be unique, and need not be an equilibrium. The proposition below
adds sufficient regularity conditions for both. First,
\(F_1(h_1;s_1,A_1,A_2)\) is strictly decreasing in \(h_1\), so the
candidate conditions in \cref{lem:period1_candidates} select at most one
symmetric candidate. Second, for each firm \(j\), the payoff
\(\Phi_1^j(h_1^j,h_1^{-j};s_1,A_1,A_2)\) is strictly concave in
\(h_1^j\) for every fixed \(h_1^{-j}\), so the first-order condition is
sufficient for a global best response. Together, these two properties
turn the necessary conditions in the lemma into a unique symmetric
equilibrium.

For tractability, we impose skill-domain restrictions that keep terminal
engagement fixed along every period-1 skill trajectory starting from
\(s_1\). The issue is that \(F_1\), \(\Pi_2^*\), and \(\bar V_2\) inherit
the shape of the terminal policy \(h_2^*(g(s_1,h_1;A_1))\). By \cref{cor:terminal_monotone}, terminal engagement is zero below a lower skill cutoff, interior between the two cutoffs, and full above an upper cutoff. If the reachable continuation states straddle one of these cutoffs, terminal engagement changes with \(h_1\), creating additional curvature and possible kinks. We therefore focus on values of \(s_1\) for which every
reachable continuation state lies entirely in a terminal corner region,
so \(h_2^*(g(s_1,h_1;A_1))\) is either always zero or always one,
independently of \(h_1\).

For the power-wage setting, let \(s_1^D\) be the period-1 threshold below which even full engagement keeps the worker below the terminal no-engagement cutoff, so terminal engagement is always zero, and let \(s_1^F\) be the threshold above which even zero engagement keeps the worker above the terminal full-engagement cutoff, so terminal engagement is always one. These thresholds pull the domain-restricted terminal cutoffs \(\hat{s}_2^D\) and \(\overline{s}_2^D\) of \cref{cor:terminal_monotone} back to period~1; their explicit formulas, and the requirement that each region be nonempty after intersection with \(D_1\), are given in \cref{app:period1_domain_size}, which also shows that the fixed-terminal regions \(s_1<s_1^D\) and \(s_1>s_1^F\) cover at least \(93.8\%\) of \(D_1\) on average, with median coverage \(100\%\), across power-wage examples satisfying the model assumptions and terminal-domain conditions.
\begin{proposition}
\label{prop:period1_regimes}
Suppose the conditions of \cref{cor:terminal_monotone} hold, \(s_1\in D_1\), and \(s_1<s_1^D\) or \(s_1>s_1^F\). There is a threshold \(\eta_1(s_1)>0\) such that, whenever \(\eta\ge\eta_1(s_1)\), the engagement gain \(F_1(h_1;s_1,A_1,A_2)\) is strictly decreasing in \(h_1\) and each firm's objective \(\Phi_1^j(h_1^j,h_1^{-j};s_1,A_1,A_2)\) is strictly concave in \(h_1^j\), so the period~1 game has a unique symmetric equilibrium \(h_1^*(s_1)\), found by comparing the engagement gain at the boundaries with zero:
\begin{enumerate}
  \item[\textup{(i)}] if \(F_1(1;s_1,A_1,A_2)\ge0\), both firms engage fully, \(h_1^*(s_1)=1\);
  \item[\textup{(ii)}] if \(F_1(1;s_1,A_1,A_2)<0<F_1(0;s_1,A_1,A_2)\), both engage at the interior level \(h_1^*(s_1)\in(0,1)\) solving \(F_1(h_1;s_1,A_1,A_2)=0\);
  \item[\textup{(iii)}] if \(F_1(0;s_1,A_1,A_2)\le0\), neither firm engages, \(h_1^*(s_1)=0\).
\end{enumerate}
Terminal engagement at the realized continuation state is then fixed: disengagement, \(h_2^*(g(s_1,h_1^*(s_1);A_1))=0\), when \(s_1<s_1^D\), and full engagement, \(h_2^*(g(s_1,h_1^*(s_1);A_1))=1\), when \(s_1>s_1^F\).
\end{proposition}

\cref{prop:period1_regimes} shows that, on the fixed-terminal regions, ex ante identical firms settle on a single common engagement level. Intuitively, a fixed terminal engagement removes the curvature that portable skill would otherwise inject, so the period-1 best responses are well behaved and the symmetric equilibrium is unique. The threshold \(\eta_1(s_1)\) is a regularity threshold for the
period-1 game, not a terminal-engagement condition. Terminal engagement
is fixed by the skill restrictions \(s_1<s_1^D\) and \(s_1>s_1^F\);
the role of \(\eta\ge\eta_1(s_1)\) is to make
\(F_1(h_1;s_1,A_1,A_2)\) strictly decreasing in \(h_1\) and
\(\Phi_1^j(h_1^j,h_1^{-j};s_1,A_1,A_2)\) strictly concave in \(h_1^j\).
If \(\eta<\eta_1(s_1)\), \cref{lem:period1_candidates} still gives
necessary conditions for symmetric equilibria, but uniqueness and
sufficiency are no longer guaranteed.

We now compare period-1 engagement under worker mobility with the single-firm benchmark of \cref{sec:monopoly}. We focus on the region $s_1<s_1^D$, in which the worker is disengaged in period~2 both under mobility and without it. The two period-1 problems then share the same continuation, so the comparison isolates the forces mobility introduces in period 1: the sorting gain and the skill spillover, rather than differences in anticipated period-2 engagement.
\begin{corollary}
\label{cor:mobility_vs_monopoly}
Suppose the conditions of \cref{prop:period1_regimes} hold with
\(s_1\in D_1\) and $s_1<s_1^D$, and let $h_1^{\mathrm{sf}}(s_1)$ denote the single firm's period-1
engagement (\cref{prop:monopoly_period1}). There exists a threshold
$\eta_2(s_1)\ge\eta_1(s_1)$ such that for all $\eta\ge\eta_2(s_1)$ the
unique symmetric equilibrium satisfies $h_1^*(s_1)\le h_1^{\mathrm{sf}}(s_1)$, with
strict inequality whenever $h_1^{\mathrm{sf}}(s_1)\in(0,1)$.\footnote{The restriction to $s_1<s_1^D$ does not tilt the comparison:
the single firm's policy bands do not depend on $\eta$, whereas $s_1^D$
rises with $\eta$, so at the threshold $\eta_2(s_1)$ the region
typically spans workers the single firm engages fully, partially, and
not at all. Its role is identification: it is the only region on which
the firm-under-mobility's terminal play matches the single firm's invariable
terminal disengagement, so the period-1 gap is attributable to
mobility alone.}
\end{corollary}

\cref{cor:mobility_vs_monopoly} isolates the region in which mobility works in the classical Becker direction. When workers respond weakly to differences in skill trajectories, the sorting gain is small but the portability of skill remains. A firm with mobile workers therefore invests no more, and typically less, than a single firm that captures the full return on the skill it builds.

Mobility can also shift engagement across time. In the single-firm benchmark of \cref{sec:monopoly}, a below-benchmark worker is never engaged in the terminal period; any engagement occurs only in period~1, as an investment in future fallback skill. With mobility, by contrast, terminal engagement can arise even when firms do not engage the worker in period~1, because it attracts workers at once rather than serving as an investment whose return may spill over to the other firm.
\begin{example}
\label{ex:reversals}
For the parameterization reported in \cref{app:timing_reversal_learning}, the unique symmetric equilibrium satisfies \(h_1^*(s_1)=0\) while the induced continuation state satisfies \(h_2^*(g(s_1,0;A_1))=1\). Mobility can therefore reverse the timing of engagement: a worker ignored before re-sorting is fully engaged after re-sorting.
\end{example}

Because the example depends on the skill transition, \cref{app:timing_reversal_learning} tests robustness by replacing the baseline learning term $\phi h_t\pi_t(1-\pi_t)(\alpha_t-s_t)$ with alternative terms $\phi h_t\ell(\pi_t)(\alpha_t-s_t)$. The timing reversal persists for a wide range of exposure--practice profiles, such as $\ell(\pi)=\pi^a(1-\pi)^c$ and $\ell(\pi)=(\kappa+\pi)(1-\pi)$, that preserve sufficient terminal skill-building at $\pi_2=0.35$; it disappears when the alternative profile makes terminal skill-building too weak.
\subsection{Asymmetric Specialization}
\label{sec:asymmetry_duopoly}

Ex ante identical firms need not behave identically. The reason is that skill, once built, is portable. When a firm invests in a worker's skill in period 1, the worker can switch firms in period 2 and carry that skill elsewhere. The firm bears the cost of building skill but cannot fully capture its return. No such force operates in the terminal period, because there is no period after it for workers to move into.

Whether each firm's engagement incentive rises or falls when the other firm engages more is measured by the cross-partial of firm \(j\)'s period-1 objective, \(\partial^2 \Phi_1^j/\partial h_1^j\,\partial h_1^{-j}\). When the other firm invests more heavily in skill, two opposing forces act on firm \(j\). First, more continuation skill enters the shared labor pool once workers re-sort, which weakens firm \(j\)'s own incentive to build skill. Second, the other firm's engagement draws workers away in period 1, so firm \(j\)'s engagement now applies to fewer workers, lowering its marginal cost of engaging and thus strengthening its incentive to engage. Evaluating this cross-partial at a symmetric profile (the algebra is in \cref{app:proofs}), engagement choices are local strategic complements if and only if
\begin{align}\label{eq:complements}
2\beta\,
\Pi_2^{*\prime}\!\bigl(g(s_1,h_1;A_1),A_2\bigr)
(\phi\pi_1+\gamma)
<
\lambda R\delta.
\end{align}

The left side of \cref{eq:complements} is the pooled-skill externality; the right side is the current-period cost effect. When the externality dominates and the condition fails, incentives are local strategic substitutes and asymmetric equilibria become possible, though the condition does not guarantee they exist.

\cref{fig:asym_eq} illustrates such a case for an admissible power-wage example. Condition~\cref{eq:complements} fails at the symmetric profile, so engagement choices are local strategic substitutes; each firm's best response slopes downward, with \(BR(0)=1\) and \(BR(1)=0\), and the two cross at one symmetric and two asymmetric equilibria. At the asymmetric equilibria one firm engages fully and the other not at all, and the disengaged firm free-rides on the skill the other builds, out-earning it.

\begin{figure}[t]
\centering
\includegraphics[width=0.5\linewidth]{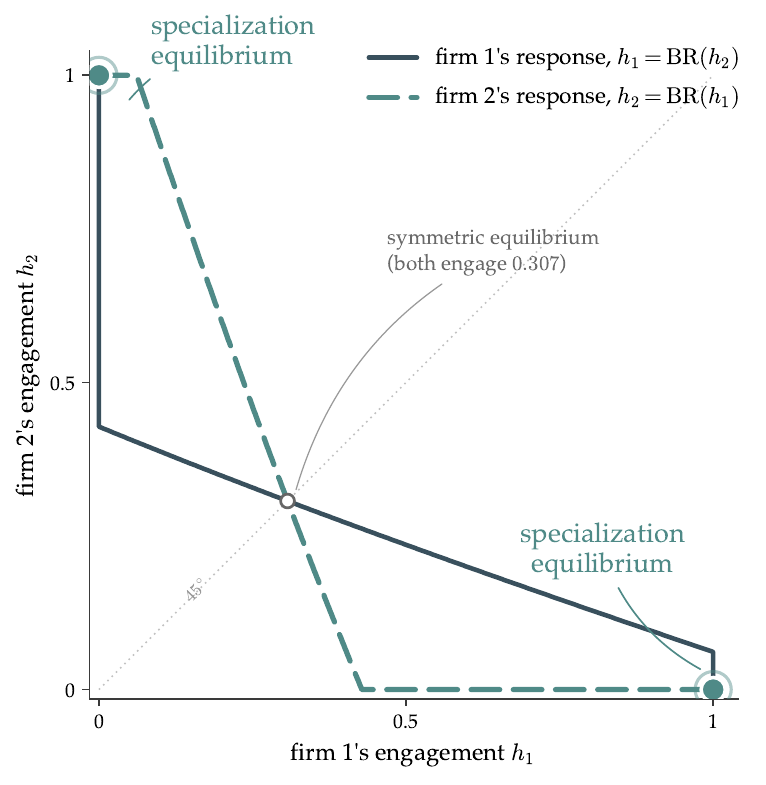}
\caption{Ex ante identical firms can specialize, and the free-rider out-earns the builder. Each firm's period-1 best response declines in the other firm's engagement, and the two cross at one symmetric equilibrium and a pair of asymmetric equilibria \((0,1)\) and \((1,0)\) in which one firm engages fully and the other not at all; at the asymmetric equilibria the disengaged firm earns \(7.732\) per worker against the engaged firm's \(7.465\), because it draws on the skill the other firm builds. Parameter values: \(B(s)=W(s)=s^{1.08}\), \(\alpha_1=0.76\), \(\alpha_2=0.82\), \(\pi_1=0.72\), \(\pi_2=0.48\), \(\phi=3.4\), \(\gamma=0.011\), \(\delta=0.29\), \(\lambda R=20\), \(\beta=0.5\), \(\eta=16\), \(s_1=0.36\).}
\label{fig:asym_eq}
\end{figure}

At the asymmetric profiles the firms specialize: one engages fully and builds skill, while the other disengages and free-rides on the portable skill that enters the shared pool after re-sorting. Mobility reverses the payoff ranking, so the free-rider out-earns the builder (\(7.732\) versus \(7.465\)) despite offering the weaker skill trajectory. Which firm plays which role depends on history or coordination, but in either case two ex ante identical firms split endogenously into a skill-builder and a free-rider.
\section{Concluding Remarks}\label{sec:discussion}

As AI handles more routine work, routing the rest to the machine is the obvious move. Yet keeping the worker engaged, though it lowers current output, preserves the fallback skill the firm may need when AI fails, and, once workers are mobile, shapes the skill trajectory that determines which firm they join.

The two motives point to different workers. The fallback motive is strongest at the bottom of the skill distribution, where engagement builds the most capability for future failure states; the sorting motive is strongest near the AI frontier, where engagement is least costly and the portable skill it builds is most valuable in the labor market. Worker mobility can therefore reverse who is engaged, depress investment by weakening the firm's hold on the skill it builds, or sustain asymmetric specialization in which one of two ex ante identical firms builds skill and the other draws on it. It also pulls the two dimensions of AI progress apart: under mobility a more capable AI raises engagement by strengthening the skill trajectory a job can offer, whereas a more reliable AI can raise or lower it, so engagement may be highest at intermediate reliability.

The framework extends in several directions. A longer horizon would let worker skill and AI capability co-evolve, so that engagement today shapes the whole skill path rather than a single fallback. Making AI progress itself respond to how intensively firms engage would close that loop. A richer labor market, in which the sorting motive operates through wages as well as job design, would let firms compete on pay and skill trajectory at once. Beyond these extensions, the central prediction is testable: because mobility can reverse engagement from the least-skilled to the most-skilled workers, it invites study wherever task-level human--AI allocation and worker movement can be observed together.

Human involvement in AI-assisted production is an investment decision about fallback capacity and portable human capital. Treating that decision as a cost to be minimized, routing each residual task to the machine, leaves a firm short of the fallback capacity it will need when the machine fails and short of the skill trajectory that keeps its strongest workers from leaving. A radiology group that sends every routine scan to its AI may find, the first time a hard case arrives or the system is unavailable, that its own unaided reading has eroded below what the case demands, and that the strong readers whose skills it chose not to build have already left for a group that invested in them.

\setlength{\bibsep}{0pt plus 0.3pt}
\bibliographystyle{informs2014-noperiods}
\bibliography{references}
\clearpage

\setcounter{page}{1}
\renewcommand{\thepage}{OA\arabic{page}}

\setcounter{equation}{0}
\renewcommand{\theequation}{OA\arabic{equation}}
\renewcommand{\theHequation}{OA.\arabic{equation}}
\setcounter{claim}{0}
\renewcommand{\theclaim}{OA\arabic{claim}}
\renewcommand{\thesection}{OA.\arabic{section}}
\renewcommand{\theHsection}{OA.\arabic{section}}
\setcounter{section}{0}
\setcounter{table}{0}
\renewcommand{\thetable}{OA\arabic{table}}
\renewcommand{\theHtable}{OA.\arabic{table}}
\setcounter{figure}{0}
\renewcommand{\thefigure}{OA\arabic{figure}}
\renewcommand{\theHfigure}{OA.\arabic{figure}}

\begingroup

\begin{center}
\textbf{\large Online Appendix to ``Managing the Human Fallback: Skill Investment Under Improving AI and Worker Mobility''}
\end{center}

\medskip

\noindent This online appendix supplements the main text. \cref{app:measurement} relates the model primitives to observable workflow data; \cref{app:proofs} proves all results in the order they appear in the paper; \cref{app:skill_transition_table} restates the single-firm engagement policy as a skill-transition table; and the remaining sections report the size of the terminal and period-1 skill domains (\cref{app:d2_size,app:period1_domain_size}), two robustness checks on the reliability pattern and the timing reversal (\cref{app:learning_profile_robustness,app:timing_reversal_learning}), and an extension that lets firms attract workers with a pay bonus (\cref{app:bonus}).

\smallskip

\begin{table}[H]
\caption{Summary of Notation}
\label{tab:notation}
\centering
\footnotesize
\renewcommand{\arraystretch}{1.0}
\begin{tabular}{@{}l l@{}}
\toprule
Symbol & Description \\
\midrule
\multicolumn{2}{@{}l}{\itshape Time and AI technology}\\
$t$ & Period index, $T=2$ \\
$A_t$ & AI technology level, period $t$ \\
$\alpha_t\triangleq\alpha(A_t)$ & AI capability \\
$\pi_t\triangleq\pi(A_t)$ & Failure probability ($1-\pi_t$ reliability) \\
\addlinespace[3pt]
\multicolumn{2}{@{}l}{\itshape Worker skill and engagement}\\
$s_t$ & Worker skill, $s_t\in[0,1]$ \\
$h_t$ & Engagement, $h_t\in[0,1]$ \\
$h^c(A_t)$ & Skill-preserving engagement \\
\addlinespace[3pt]
\multicolumn{2}{@{}l}{\itshape Skill dynamics}\\
$\delta$ & Effect of intervention on AI output \\
$\gamma$ & Skill-erosion rate, passive reliance \\
$\phi$ & Strength of learning from engagement \\
$\Gamma(h_t;A_t)$ & Net skill-adjustment coefficient \\
$g(s_t,h_t;A_t)$ & Skill transition function \\
\addlinespace[3pt]
\multicolumn{2}{@{}l}{\itshape Worker mobility}\\
$\eta$ & Sorting-friction parameter \\
$\sigma_t^j$ & Logit share of workers at firm $j$ \\
\addlinespace[3pt]
\multicolumn{2}{@{}l}{\itshape Production}\\
$q^{\mathrm{on}}(s_t,h_t;A_t)$ & Throughput, AI functioning \\
$q^{\mathrm{off}}(s_t)$ & Throughput, AI failed \\
$S(s_t,h_t;A_t)$ & Operational revenue per worker \\
\addlinespace[3pt]
\multicolumn{2}{@{}l}{\itshape Firm payoffs}\\
$M(s_t,h_t;A_t)$ & Per-worker margin, $S-W$ \\
$W(s)$ & Market wage schedule \\
$\lambda R$ & Task load $\lambda$ times revenue $R$ \\
$\beta$ & Discount factor \\
$B(s)$ & Post-horizon value of skill \\
\bottomrule
\end{tabular}
\end{table}

\medskip

\section{Notation and Empirical Counterparts}\label{app:notation}\label{app:measurement}

\cref{tab:notation} lists the notation used throughout the paper. Each primitive also has a direct empirical counterpart in task-level workflow data, as we describe below.

AI capability $\alpha_t$ can be measured from tasks completed by AI with little or no human involvement, conditional on AI functioning, whereas failure probability $\pi_t$ can be measured from cases in which AI cannot be used as the autonomous producer because it is unavailable, blocked, fails validation, or faces an edge case requiring human fallback. Worker skill $s_t$ can be measured from no-AI holdout tasks or fallback tasks performed independently by the worker. Variation in engagement across otherwise comparable AI-on tasks can be used to estimate $\delta$, by measuring how human involvement shifts realized throughput relative to the AI-alone and worker-alone benchmarks. Operational records provide $\lambda$ and $R$, whereas salary bands and external labor-market data can be used to calibrate $W(\cdot)$ and $\eta$.

The skill-dynamics parameters require repeated measures of worker performance. Declines in independent performance after periods of low engagement and high AI reliance can be used to estimate or calibrate the erosion rate $\gamma$, whereas improvements after hands-on engagement with AI-assisted tasks can be used to estimate or calibrate the learning parameter $\phi$. Existing empirical work provides guidance for these measurements: studies of AI-assisted work use variation in AI access and workflow design to estimate changes in productivity or output, which speaks to $\alpha_t$, $s_t$, and $\delta$ \citep{brynjolfsson2025genaiwork,dellacqua2023,ni2024ecommerce,Vaccaro2024combinations}; studies that measure later unaided performance after AI exposure speak to skill erosion and learning, and therefore to $\gamma$ and $\phi$ \citep{bastani2025genai,budzyn2025endoscopist,poulidis2025selfregulated,shen2026skillformation}.

\endgroup

\section{Proofs}\label{app:proofs}

We first record a fact used in several proofs below.
\begin{claim}\label{app:margin_positivity}
Under \cref{ass:wages_revenue}, for each $t \in \{1,2\}$,
all $s_t \in (0, \alpha_t)$, and all $h_t \in [0,1]$, \(M(s_t, h_t; A_t) > 0\).
\end{claim}

\medskip

\noindent \emph{Proof.}
Fix $t \in \{1,2\}$ and $h_t \in [0,1]$. Define
\begin{align*}
f(s) \triangleq M(s, h_t; A_t) = S(s, h_t; A_t) - W(s),
\qquad s \in [0, \alpha_t].
\end{align*}
For fixed $h_t$, $S(s, h_t; A_t)$ is affine in $s$, whereas $W$ is
strictly convex by \cref{ass:wages_revenue}. Hence $f$ is
strictly concave on $[0, \alpha_t]$. At the lower endpoint, using
$W(0) = 0$, \(f(0) = \lambda R(1-\pi_t)\alpha_t(1-\delta h_t) > 0\),
because $1-\pi_t > 0$, $\alpha_t > 0$, and $\delta h_t < 1$. At the
upper endpoint, \(f(\alpha_t) = \lambda R\, \alpha_t - W(\alpha_t) > 0\)
by revenue sufficiency in \cref{ass:wages_revenue}. Strict
concavity of $f$ on $[0, \alpha_t]$ with strictly positive endpoint
values implies $f(s) > 0$ for every $s \in (0, \alpha_t)$.
\hfill \emph{Q.E.D.}

\medskip

\noindent \emph{Proof of \cref{prop:monopoly_period1}.}\label{app:proof_monopoly_period1}
Fix $s_1 \in (\underline{s}, \alpha_1)$. Under
\cref{ass:skill_domain}, the zero floor in the skill
transition never binds for any $h_1 \in [0,1]$, so
\begin{align*}
g(s_1, h_1; A_1)
= s_1 + (1-\pi_1)\bigl(h_1(\phi\pi_1 + \gamma) - \gamma\bigr)(\alpha_1 - s_1).
\end{align*}
The period-1 single-firm objective is
\begin{align*}
V(h_1) = M(s_1, h_1; A_1) + \beta\, M\!\bigl(g(s_1, h_1; A_1),\, 0;\, A_2\bigr).
\end{align*}

The first term is affine in
$h_1$. The continuation term is the sum of an affine function of
$g(s_1, h_1; A_1)$ and $-W(g(s_1, h_1; A_1))$. Because $g(s_1, h_1; A_1)$
is affine in $h_1$ with strictly positive slope
\begin{align*}
\frac{\partial g(s_1, h_1; A_1)}{\partial h_1}
= (1-\pi_1)(\phi\pi_1 + \gamma)(\alpha_1 - s_1) > 0,
\end{align*}
and $W$ is strictly convex by \cref{ass:wages_revenue}, the
continuation term is strictly concave in $h_1$. Hence $V$ is strictly
concave and has a unique maximizer $h_1^{\mathrm{sf}}(s_1) \in [0,1]$.

Differentiating gives
\begin{align*}
V'(h_1)
&= \lambda R \delta(1-\pi_1)(s_1 - \alpha_1) \\
&\quad + \beta\bigl[\lambda R \pi_2 - W'(g(s_1, h_1; A_1))\bigr]
(1-\pi_1)(\phi\pi_1 + \gamma)(\alpha_1 - s_1) \\
&= (1-\pi_1)(\alpha_1 - s_1)
\left\{-\lambda R \delta + \beta(\phi\pi_1 + \gamma)
\bigl[\lambda R \pi_2 - W'(g(s_1, h_1; A_1))\bigr]\right\}.
\end{align*}
Strict concavity of $V$ implies $V'$ is strictly decreasing. The
optimum satisfies the standard Karush–Kuhn–Tucker conditions:
$h_1^{\mathrm{sf}}(s_1) = 0$ if and only if $V'(0) \le 0$; $h_1^{\mathrm{sf}}(s_1) = 1$ if
and only if $V'(1) \ge 0$; otherwise the unique interior optimum
solves $V'(h_1) = 0$, which is equivalent to
\begin{align*}
W'\!\bigl(g(s_1, h_1^{\mathrm{sf}}(s_1); A_1)\bigr)
= \lambda R \pi_2 - \frac{\lambda R \delta}{\beta(\phi\pi_1 + \gamma)}.
\end{align*}

Next suppose
\cref{eq:interior_target_condition} holds, and let $s^*$ be defined
by~\cref{eq:s_star}. Because $W$ is strictly convex,
$W'$ is strictly increasing, and the interior first-order condition is
equivalent to $g(s_1, h_1^{\mathrm{sf}}(s_1); A_1) = s^*$. The boundary conditions
$V'(0) \le 0$ and $V'(1) \ge 0$ translate into $g(s_1, 0; A_1) \ge s^*$
and $g(s_1, 1; A_1) \le s^*$, respectively. This gives the three cases
in the proposition.

Solving $g(s_1, h_1^{\mathrm{sf}}(s_1); A_1) = s^*$ for the interior optimizer yields
\begin{align*}
h_1^{\mathrm{sf}}(s_1)
= \frac{s^* - g(s_1, 0; A_1)}{(1-\pi_1)(\phi\pi_1 + \gamma)(\alpha_1 - s_1)},
\end{align*}
which is~\cref{eq:h_interior_closed}.

Finally, differentiating the interior
closed form gives
\begin{align*}
\frac{d h_1^{\mathrm{sf}}(s_1)}{d s_1}
= \frac{s^* - \alpha_1}{(1-\pi_1)(\phi\pi_1 + \gamma)(\alpha_1 - s_1)^2} < 0,
\end{align*}
because $s^* \in (0, \alpha_1)$.
\hfill \emph{Q.E.D.}

\medskip

\noindent \emph{Proof of \cref{cor:reliability_vs_capability}.}\label{app:proof_cor_reliability}
\textbf{Part (i).} Under~\cref{eq:interior_target_condition},
$s^*$ is defined implicitly by
\begin{align*}
W'(s^*) = \lambda R \pi_2 - \frac{\lambda R \delta}{\beta(\phi\pi_1 + \gamma)}.
\end{align*}
The right-hand side does not depend on $\alpha_2$, so
$\partial s^* / \partial \alpha_2 = 0$. Differentiating both sides in
$\pi_2$ and using $W''(s^*) > 0$ gives \(\frac{\partial s^*}{\partial \pi_2} = \frac{\lambda R}{W''(s^*)} > 0\).
The thresholds $s_L$ and $s_H$ defined in~\cref{eq:s_L_s_H} are
affine functions of $s^*$ with strictly positive slopes
\begin{align*}
\frac{1}{1 - \phi\pi_1(1-\pi_1)}
\quad \text{and} \quad
\frac{1}{1 + \gamma(1-\pi_1)},
\end{align*}
respectively. The first slope is positive by
\cref{ass:skill}. Hence both thresholds are strictly
increasing in $\pi_2$ and independent of $\alpha_2$.

\textbf{Part (ii).} Fix $s_1 \in (\underline{s}, \alpha_1)$. We show
that $h_1^{\mathrm{sf}}(s_1)$ is weakly increasing in $\pi_2$ on $(0, \pi_1]$.

\emph{Interior region.} On the open set where
$g(s_1, 0; A_1) < s^*(\pi_2) < g(s_1, 1; A_1)$, the closed
form~\cref{eq:h_interior_closed} gives
\begin{align*}
h_1^{\mathrm{sf}}(s_1)
= \frac{s^*(\pi_2) - g(s_1, 0; A_1)}
       {(1-\pi_1)(\phi\pi_1 + \gamma)(\alpha_1 - s_1)}.
\end{align*}
The denominator and $g(s_1, 0; A_1)$ do not depend on $\pi_2$.
Differentiating in $\pi_2$,
\begin{align*}
\frac{\partial h_1^{\mathrm{sf}}(s_1)}{\partial \pi_2}
= \frac{1}{(1-\pi_1)(\phi\pi_1 + \gamma)(\alpha_1 - s_1)} \cdot
\frac{\partial s^*}{\partial \pi_2} > 0
\end{align*}
by part (i).

\emph{Boundary regions.} The boundary $h_1^{\mathrm{sf}}(s_1) = 0$ holds when
$g(s_1, 0; A_1) \ge s^*(\pi_2)$, equivalently $s_1 \ge s_H(\pi_2)$;
the boundary $h_1^{\mathrm{sf}}(s_1) = 1$ holds when
$g(s_1, 1; A_1) \le s^*(\pi_2)$, equivalently $s_1 \le s_L(\pi_2)$.
Because $s^*(\pi_2)$, $s_L(\pi_2)$, and $s_H(\pi_2)$ are all strictly
increasing in $\pi_2$ by part (i), increasing $\pi_2$ can only move
the worker from the no-engagement region into the interior region
(once $s_H(\pi_2)$ rises past $s_1$), and from the interior region
into the full-engagement region (once $s_L(\pi_2)$ rises past $s_1$).
The reverse transitions are impossible. Hence $h_1^{\mathrm{sf}}(s_1)$ is weakly
increasing in $\pi_2$ on the feasible domain $(0, \pi_1]$, with
strict monotonicity in the interior region. Equivalently, as $\pi_2$
falls from $\pi_1$ toward zero, $h_1^{\mathrm{sf}}(s_1)$ either remains constant
or decreases.
\hfill \emph{Q.E.D.}

\medskip

\noindent \emph{Proof of \cref{prop:fallback_targets_bottom}.} From the
skill transition $g(s_1,h_1;A_1)=s_1+(1-\pi_1)\bigl(h_1(\phi\pi_1+\gamma)-\gamma\bigr)(\alpha_1-s_1)$,
\begin{align*}
\frac{\partial g}{\partial s_1}
= 1-(1-\pi_1)\bigl(h_1(\phi\pi_1+\gamma)-\gamma\bigr)
\ge 1-(1-\pi_1)\phi\pi_1 > 0,
\end{align*}
where the bound uses $h_1\le1$ and the no-leapfrogging \cref{ass:skill}.
Thus $g$ is strictly increasing in $s_1$ for every $h_1\in[0,1]$, so by
\cref{prop:monopoly_period1} the full-engagement set
$\{s_1:g(s_1,1;A_1)\le s^*\}$ is a lower interval, and the no-engagement
set $\{s_1:g(s_1,0;A_1)\ge s^*\}$ is an upper interval, with the interior
region between them. On the interior, \cref{prop:monopoly_period1} sets
$g(s_1,h_1^{\mathrm{sf}}(s_1);A_1)=s^*$; differentiating in $s_1$ and using
$\partial g/\partial h_1=(1-\pi_1)(\phi\pi_1+\gamma)(\alpha_1-s_1)>0$,
\begin{align*}
\frac{dh_1^{\mathrm{sf}}(s_1)}{ds_1}
= -\frac{\partial g/\partial s_1}{\partial g/\partial h_1} < 0.
\end{align*}
Hence $h_1^{\mathrm{sf}}$ equals $1$ on the lower interval, strictly decreases across
the interior region, and equals $0$ on the upper interval, so it is
nonincreasing in $s_1$.
\hfill \emph{Q.E.D.}

\medskip

\noindent \emph{Proof of \cref{lem:terminal_symmetry}.}\label{app:proof_terminal_symmetry}
Suppose, to the contrary, that $(h_2^{1*},h_2^{2*})$ is a
pure-strategy Nash equilibrium with $h_2^{1*}>h_2^{2*}$. For
$s_2<\alpha_2$, the transition $g(s_2,h_2;A_2)$ is strictly increasing
in $h_2$. Because $B$ is strictly increasing, the deterministic worker
value \(V_2^j(s_2,h_2^j)=W(s_2)+\beta B(g(s_2,h_2^j;A_2))\)
is also strictly increasing in $h_2^j$. Hence
$V_2^1(s_2,h_2^{1*})>V_2^2(s_2,h_2^{2*})$, and the logit sorting rule
implies \(\sigma_2^{1*}>1/2>\sigma_2^{2*}\).
Consider firm 1's option to deviate unilaterally to the other firm's
engagement level $h_2^{2*}$, while firm 2 continues to play
$h_2^{2*}$. The resulting action profile would be
$(h_2^{2*}, h_2^{2*})$, at which the logit sorting probability is
$\sigma_2^1 = 1/2$ by symmetry of the form, and firm 1's own margin
is $M(s_2, h_2^{2*}; A_2)$. The deviating payoff would therefore
equal $(1/2)\, M(s_2, h_2^{2*}; A_2)$, and Nash equilibrium requires
that firm 1's equilibrium payoff weakly exceeds this:
\begin{align*}
\sigma_2^{1*}\, M(s_2, h_2^{1*}; A_2)
\ge
\tfrac{1}{2}\, M(s_2, h_2^{2*}; A_2).
\end{align*}
Analogously, considering firm 2's option to deviate unilaterally to
$h_2^{1*}$ while firm 1 continues to play $h_2^{1*}$, Nash equilibrium
requires
\begin{align*}
\sigma_2^{2*}\, M(s_2, h_2^{2*}; A_2)
\ge
\tfrac{1}{2}\, M(s_2, h_2^{1*}; A_2).
\end{align*}
Both margins are strictly positive by revenue sufficiency in
\cref{ass:wages_revenue}. Multiplying the two no-deviation
inequalities and dividing through by the positive product
$M(s_2, h_2^{1*}; A_2)\, M(s_2, h_2^{2*}; A_2)$ gives \(\sigma_2^{1*}\, \sigma_2^{2*} \ge \tfrac{1}{4}\).
But $\sigma_2^{1*}, \sigma_2^{2*} > 0$ and $\sigma_2^{1*} + \sigma_2^{2*} = 1$
imply $\sigma_2^{1*}\, \sigma_2^{2*} \le 1/4$, with equality only at
$\sigma_2^{1*} = \sigma_2^{2*} = 1/2$. This contradicts
$\sigma_2^{1*} > 1/2 > \sigma_2^{2*}$, so no asymmetric pure-strategy
Nash equilibrium exists.
\hfill \emph{Q.E.D.}
\medskip

\noindent \emph{Proof of \cref{prop:duopoly_terminal}.}\label{app:proof_duopoly_terminal}
Fix \(s_2\in D_2\). Since \(D_2\subset(\underline{s}_2,\alpha_2)\), the zero floor
does not bind, so
\begin{align*}
g(s_2,h;A_2)
=
s_2+(1-\pi_2)\bigl[h(\phi\pi_2+\gamma)-\gamma\bigr](\alpha_2-s_2),
\end{align*}
and therefore
\begin{align*}
\frac{\partial g(s_2,h;A_2)}{\partial h}
=
(1-\pi_2)(\phi\pi_2+\gamma)(\alpha_2-s_2)>0.
\end{align*}
Moreover,
\begin{align*}
\frac{\partial M(s_2,h;A_2)}{\partial h}
=
-\lambda R\delta(1-\pi_2)(\alpha_2-s_2)<0,
\end{align*}
because \(s_2<\alpha_2\).

First consider the engagement-gain function
\begin{align*}
F_2(h;s_2,A_2)
=
\frac{\beta(\phi\pi_2+\gamma)}{2\eta}
B'\!\bigl(g(s_2,h;A_2)\bigr)
M(s_2,h;A_2).
\end{align*}
Differentiating with respect to $h$ gives
\begin{align*}
\frac{\partial F_2(h;s_2,A_2)}{\partial h}
&=
\frac{\beta(\phi\pi_2+\gamma)}{2\eta}
B'\!\bigl(g(s_2,h;A_2)\bigr)
(1-\pi_2)(\alpha_2-s_2)\\
&\quad\times\left[
(\phi\pi_2+\gamma)
\frac{B''(g(s_2,h;A_2))}{B'(g(s_2,h;A_2))}
M(s_2,h;A_2)
-\lambda R\delta
\right].
\end{align*}
All terms outside the square brackets are strictly positive. Therefore,
condition~\cref{eq:monotone_condition} implies the expression in
brackets is strictly negative for every $h\in[0,1]$. Hence
$F_2(h;s_2,A_2)$ is strictly decreasing on $[0,1]$.

Now consider firm $j$'s terminal-period payoff,
\begin{align*}
\Pi_2^j(h_2^j,h_2^{-j};s_2,A_2)
=
\sigma_2^j(h_2^j,h_2^{-j};s_2,A_2)
M(s_2,h_2^j;A_2).
\end{align*}
Differentiating with respect to $h_2^j$ and using the logit derivative,
we obtain
\begin{align*}
\frac{\partial \Pi_2^j}{\partial h_2^j}
=
\sigma_2^j(1-\pi_2)(\alpha_2-s_2)
\left[
(1-\sigma_2^j)
\frac{\beta}{\eta}
B'\!\bigl(g(s_2,h_2^j;A_2)\bigr)
(\phi\pi_2+\gamma)
M(s_2,h_2^j;A_2)
-\lambda R\delta
\right].
\end{align*}
Because the prefactor is strictly positive, the sign of
$\partial \Pi_2^j/\partial h_2^j$ is the sign of the bracketed term.

For any fixed action by the other firm $h_2^{-j}$, define
\begin{align*}
D(h_2^j;h_2^{-j})
\triangleq
(1-\sigma_2^j)
\frac{\beta}{\eta}
B'\!\bigl(g(s_2,h_2^j;A_2)\bigr)
(\phi\pi_2+\gamma)
M(s_2,h_2^j;A_2)
-\lambda R\delta.
\end{align*}
We next show $D(h_2^j;h_2^{-j})$ is strictly decreasing in
$h_2^j$. First, $\sigma_2^j$ is strictly increasing in $h_2^j$, because
$B$ and $g$ are strictly increasing in own engagement. Hence
$1-\sigma_2^j$ is strictly decreasing in $h_2^j$. Second,
\begin{align*}
\frac{\partial}{\partial h_2^j}
\left[
B'\!\bigl(g(s_2,h_2^j;A_2)\bigr)
M(s_2,h_2^j;A_2)
\right]
\end{align*}
equals
\begin{align*}
B'\!\bigl(g(s_2,h_2^j;A_2)\bigr)
(1-\pi_2)(\alpha_2-s_2)
\left[
(\phi\pi_2+\gamma)
\frac{B''(g(s_2,h_2^j;A_2))}
     {B'(g(s_2,h_2^j;A_2))}
M(s_2,h_2^j;A_2)
-\lambda R\delta
\right].
\end{align*}
Since \(s_2\in D_2\), condition~\cref{eq:monotone_condition} applies for
every own action \(h_2^j\in[0,1]\). Hence the expression in square
brackets is strictly negative for every own action $h_2^j\in[0,1]$,
regardless of the other firm's action. Therefore the derivative above is
strictly negative.

It follows that the product
\begin{align*}
(1-\sigma_2^j)
B'\!\bigl(g(s_2,h_2^j;A_2)\bigr)
M(s_2,h_2^j;A_2)
\end{align*}
is strictly decreasing in $h_2^j$, and so
$D(h_2^j;h_2^{-j})$ is strictly decreasing in firm $j$'s own engagement
for every fixed action by the other firm. Thus firm $j$'s payoff is strictly
quasiconcave in its own action.

At a symmetric profile $h_2^j=h_2^{-j}=h$, we have
$\sigma_2^j=1/2$. Hence the first-order condition for an interior
symmetric equilibrium reduces to \(F_2(h;s_2,A_2)=\lambda R\delta\).
Because $F_2$ is strictly decreasing, there can be at most one interior
symmetric equilibrium. The boundary cases follow from the same
strict-quasiconcavity argument applied to each firm's best response to
the other firm's boundary action. 

If \(F_2(0;s_2,A_2)\le \lambda R\delta\), then at the profile \((0,0)\)
the sign-determining term \(D(0;0)\) is nonpositive. Since
\(D(\cdot;0)\) is strictly decreasing, firm \(j\)'s marginal payoff is
negative for every \(h_2^j>0\). Hence \(0\) is the unique best response
to \(0\), so \(h_2^*(s_2)=0\) is the unique symmetric equilibrium.

If \(F_2(1;s_2,A_2)\ge \lambda R\delta\), then at the profile \((1,1)\)
the sign-determining term \(D(1;1)\) is nonnegative. Since
\(D(\cdot;1)\) is strictly decreasing, \(D(h_2^j;1)>D(1;1)\ge0\) for
every \(h_2^j<1\). Hence firm \(j\)'s payoff is increasing up to
\(h_2^j=1\), so \(1\) is the unique best response to \(1\), and
\(h_2^*(s_2)=1\) is the unique symmetric equilibrium.

Finally, \cref{lem:terminal_symmetry} rules out asymmetric
pure-strategy equilibria. Therefore the unique symmetric equilibrium
identified above is the unique pure-strategy Nash equilibrium of the
terminal-period two-firm game.
\hfill \emph{Q.E.D.}

\medskip

\begin{claim}\label{app:power_single_crossing_cutoff}
Suppose \(B(s)=W(s)=s^b\) with \(b>1\). For \(s_2\in
(\underline{s}_2,\alpha_2)\), the left-hand side of the terminal
single-crossing condition is maximized at \(h=0\) and equals
\begin{align*}
\Psi(s_2)
\equiv
(b-1)
\frac{
\lambda R\big[(1-\pi_2)\alpha_2+\pi_2s_2\big]-s_2^b
}{
[1+\gamma(1-\pi_2)]s_2-\gamma(1-\pi_2)\alpha_2
}.
\end{align*}
Moreover, under \cref{ass:wages_revenue}, \(\Psi\) is strictly decreasing,
\begin{align*}
\lim_{s_2\downarrow\underline{s}_2}\Psi(s_2)=+\infty,
\qquad
\lim_{s_2\uparrow\alpha_2}\Psi(s_2)
=
(b-1)(\lambda R-\alpha_2^{b-1}).
\end{align*}
Hence the terminal single-crossing condition holds on a nonempty
upper-tail domain if and only if
\begin{align}\label{eq:upper_tail_curvature_condition}
(b-1)(\lambda R-\alpha_2^{b-1})
<
\frac{\lambda R\delta}{\phi\pi_2+\gamma}.
\end{align}
When this condition holds, there is a unique cutoff
\(\bar{\bar{s}}_2\in(\underline{s}_2,\alpha_2)\) solving
\begin{align}\label{eq:monotone_condition_power}
\Psi(\bar{\bar{s}}_2)
=
\frac{\lambda R\delta}{\phi\pi_2+\gamma},
\end{align}
and the terminal single-crossing condition holds exactly for
\(s_2\in(\bar{\bar{s}}_2,\alpha_2)\).
\end{claim}

\noindent \emph{Proof.}
For the power wage,
\begin{align*}
\frac{B''(g(s_2,h;A_2))}{B'(g(s_2,h;A_2))}
=
\frac{b-1}{g(s_2,h;A_2)}.
\end{align*}
Thus the left-hand side of the terminal single-crossing condition is
\begin{align*}
(b-1)\max_{h\in[0,1]}
\left\{
\frac{M(s_2,h;A_2)}{g(s_2,h;A_2)}
\right\}.
\end{align*}
On \(s_2\in(\underline{s}_2,\alpha_2)\), we have
\(g(s_2,h;A_2)>0\) and \(M(s_2,h;A_2)>0\). Also,
\begin{align*}
\frac{\partial M(s_2,h;A_2)}{\partial h}
=
-\lambda R\delta(1-\pi_2)(\alpha_2-s_2)<0
\end{align*}
and
\begin{align*}
\frac{\partial g(s_2,h;A_2)}{\partial h}
=
(1-\pi_2)(\phi\pi_2+\gamma)(\alpha_2-s_2)>0.
\end{align*}
Therefore
\begin{align*}
\frac{\partial}{\partial h}
\left[
\frac{M(s_2,h;A_2)}{g(s_2,h;A_2)}
\right]
=
\frac{M_h g-Mg_h}{g^2}<0.
\end{align*}
Hence the maximum is attained at \(h=0\). Since
\begin{align*}
g(s_2,0;A_2)
=
[1+\gamma(1-\pi_2)]s_2-\gamma(1-\pi_2)\alpha_2
\end{align*}
and
\begin{align*}
M(s_2,0;A_2)
=
\lambda R\big[(1-\pi_2)\alpha_2+\pi_2s_2\big]-s_2^b,
\end{align*}
the left-hand side equals \(\Psi(s_2)\).

It remains to show \(\Psi\) is strictly decreasing. Let
\(q\equiv1-\pi_2\), \(a\equiv1+\gamma q\), and
\begin{align*}
N(s_2)\equiv \lambda R(q\alpha_2+\pi_2s_2)-s_2^b,
\qquad
D(s_2)\equiv as_2-\gamma q\alpha_2.
\end{align*}
Then \(\Psi(s_2)=(b-1)N(s_2)/D(s_2)\), and \(D(s_2)>0\) on
\((\underline{s}_2,\alpha_2)\). Thus the sign of \(\Psi'(s_2)\) is the
sign of \(N'(s_2)D(s_2)-N(s_2)D'(s_2)\).
A direct calculation gives
\begin{align*}
N'(s_2)D(s_2)-N(s_2)D'(s_2)
=
-\lambda R q\alpha_2(1+\gamma)
+
s_2^{b-1}
\left[
b\gamma q\alpha_2-a(b-1)s_2
\right].
\end{align*}
Write \(x=s_2/\alpha_2\in(0,1)\). Then
\begin{align*}
s_2^{b-1}
\left[
b\gamma q\alpha_2-a(b-1)s_2
\right]
=
\alpha_2^b
x^{b-1}
\left[
b\gamma q-a(b-1)x
\right].
\end{align*}
The function \(x^{b-1}\left[b\gamma q-a(b-1)x\right]\)
is either nonpositive or attains its maximum at
\(x=\gamma q/a\), where its value is
\begin{align*}
\gamma q\left(\frac{\gamma q}{a}\right)^{b-1}
<
q(1+\gamma).
\end{align*}
Using \(\lambda R>\alpha_2^{b-1}\) from
\cref{ass:wages_revenue} in the power-wage case, we therefore have \(N'(s_2)D(s_2)-N(s_2)D'(s_2)<0\).
Hence \(\Psi'(s_2)<0\) on
\((\underline{s}_2,\alpha_2)\).

Finally, as \(s_2\downarrow\underline{s}_2\), the denominator
\(D(s_2)=g(s_2,0;A_2)\) converges to zero while the numerator remains
positive by margin positivity. Hence \(\Psi(s_2)\to+\infty\). At the
upper endpoint,
\begin{align*}
D(\alpha_2)=\alpha_2,
\qquad
N(\alpha_2)=\lambda R\alpha_2-\alpha_2^b,
\end{align*}
so
\begin{align*}
\lim_{s_2\uparrow\alpha_2}\Psi(s_2)
=
(b-1)(\lambda R-\alpha_2^{b-1}).
\end{align*}
Strict monotonicity then gives the claimed existence and uniqueness of
\(\bar{\bar{s}}_2\).
\hfill \emph{Q.E.D.}

\medskip

\noindent \emph{Proof of \cref{cor:terminal_cutoff}.}
\label{app:proof_terminal_cutoff}
Fix \(s_2\in D_2\). Since \(D_2\subset(\underline{s}_2,\alpha_2)\), by the
definition of \(\underline{s}_2\),
\begin{align*}
g(s_2,0;A_2)
=
[1+\gamma(1-\pi_2)]s_2-\gamma(1-\pi_2)\alpha_2
>0.
\end{align*}
Since
\begin{align*}
\frac{\partial g(s_2,h;A_2)}{\partial h}
=
(1-\pi_2)(\phi\pi_2+\gamma)(\alpha_2-s_2)>0,
\end{align*}
we have \(g(s_2,h;A_2)>0\) for every \(h\in[0,1]\). Thus the zero floor in
the skill transition does not bind.

At zero engagement, under \(B(s)=W(s)=s^b\),
\begin{align*}
F_2(0;s_2,A_2)
=
\frac{\beta b(\phi\pi_2+\gamma)}{2\eta}
\bigl(g(s_2,0;A_2)\bigr)^{b-1}
M(s_2,0;A_2).
\end{align*}
Therefore, by the definition of \(\bar\eta(s_2)\) in
\cref{eq:eta_cutoff_terminal},
\begin{align*}
F_2(0;s_2,A_2)
=
\lambda R\delta\cdot \frac{\bar\eta(s_2)}{\eta}.
\end{align*}
By \cref{prop:duopoly_terminal}, the terminal-period equilibrium satisfies
\(h_2^*(s_2)=0\) if and only if \(F_2(0;s_2,A_2)\le \lambda R\delta\).
Hence \(h_2^*(s_2)>0\) if and only if \(F_2(0;s_2,A_2)>\lambda R\delta\).
Using the display above and \(\eta>0\), this is equivalent to \(\frac{\bar\eta(s_2)}{\eta}>1\),
or \(\eta<\bar\eta(s_2)\).
\hfill \emph{Q.E.D.}

\medskip

\emph{Shape of the terminal cutoff.} Although the equilibrium characterization is restricted to \(D_2\), the cutoff \(\bar\eta(s_2)\) is defined on the full smooth interval \((\underline{s}_2,\alpha_2)\). The transition \(g(s_2,0;A_2)\propto s_2-\underline{s}_2\) rises from \(0\) at the floor to \(\alpha_2\) at the benchmark, so with \(b>1\) the cutoff vanishes as \(s_2\downarrow\underline{s}_2\). The margin \(M(s_2,0;A_2)\) has slope \(\lambda R\pi_2-b\,s_2^{b-1}\), which decreases in \(s_2\): it rises far below the benchmark and turns down nearer it once the marginal wage \(b\,s_2^{b-1}\) overtakes the marginal current revenue a more skilled worker brings when AI fails, \(\lambda R\pi_2\), which may or may not happen before \(\alpha_2\). How the cutoff varies over \((\underline{s}_2,\alpha_2)\) is governed by the sign of the scalar \(\Delta_2\) defined in \cref{eq:delta2_def}. When \(\Delta_2\ge0\), \(\bar\eta(s_2)\) rises monotonically across \((\underline{s}_2,\alpha_2)\). When \(\Delta_2<0\), it is single-peaked, strictly increasing to a unique interior maximum at \(s_2^\dagger\) and strictly decreasing beyond it. Write \(\eta^U\triangleq\lim_{s_2\to\alpha_2^-}\bar\eta(s_2)\) for the cutoff at the benchmark and, in the single-peaked case, \(\eta^\dagger\triangleq\bar\eta(s_2^\dagger)\) for its peak, with \(0<\eta^U<\eta^\dagger\).

Because the equilibrium characterization applies on \(D_2\), the relevant object is the restriction of \(\bar\eta\) to \(D_2\). Define \(\eta^D\equiv\lim_{s_2\downarrow\bar{\bar{s}}_2}\bar\eta(s_2)\). When a full-domain threshold lies below \(\bar{\bar{s}}_2\), we truncate it at the admissible-domain boundary: write \(\hat{s}_2^D\equiv\max\{\hat{s}_2,\bar{\bar{s}}_2\}\) and \(\hat{s}_2^{L,D}\equiv\max\{\hat{s}_2^L,\bar{\bar{s}}_2\}\).

\begin{claim}\label{app:terminal_targeting_detail}
Suppose \(B(s)=W(s)=s^b\) with \(b>1\), and fix \(\eta\). On the admissible domain \(D_2\), the engaged set \(E_2(\eta)\equiv\{s_2\in D_2:\eta<\bar\eta(s_2)\}\) is an interval. If \(\Delta_2\ge0\), then \(E_2(\eta)\) is the high-skill band \((\hat{s}_2^D,\alpha_2)\) when \(\eta<\eta^U\), and is empty otherwise. If \(\Delta_2<0\) and \(s_2^\dagger>\bar{\bar{s}}_2\), then the full-domain peak lies inside \(D_2\), and \(E_2(\eta)\) is the high-skill band \((\hat{s}_2^D,\alpha_2)\) when \(\eta\le\eta^U\), a bounded band \((\hat{s}_2^{L,D},\hat{s}_2^H)\) containing \(s_2^\dagger\) when \(\eta^U<\eta<\eta^\dagger\), and is empty when \(\eta\ge\eta^\dagger\). If \(\Delta_2<0\) and \(s_2^\dagger\le\bar{\bar{s}}_2\), then the full-domain peak lies weakly below \(D_2\), so \(\bar\eta\) is strictly decreasing on \(D_2\), and \(E_2(\eta)=D_2\) when \(\eta\le\eta^U\), \(E_2(\eta)=(\bar{\bar{s}}_2,\tilde{s}_2)\) when \(\eta^U<\eta<\eta^D\), where \(\tilde{s}_2\in D_2\) uniquely solves \(\bar\eta(\tilde{s}_2)=\eta\), and \(E_2(\eta)\) is empty when \(\eta\ge\eta^D\).
\end{claim}
As \(\eta\) falls, a firm first engages the admissible workers with the highest cutoff and then widens the band. When \(\Delta_2\ge0\), these are the highest-skill workers in \(D_2\). When \(\Delta_2<0\) and \(s_2^\dagger\in D_2\), they are the workers around the peak \(s_2^\dagger\). When \(\Delta_2<0\) and \(s_2^\dagger\le\bar{\bar{s}}_2\), the peak is truncated away, so engagement starts just above the lower boundary \(\bar{\bar{s}}_2\). \Cref{prop:terminal_targeting} is the qualitative summary of this characterization.

\noindent \emph{Proof of \cref{prop:terminal_targeting} and \Cref{app:terminal_targeting_detail}.}\label{app:proof_terminal_targeting}
From \cref{cor:terminal_cutoff}, positive terminal-period
engagement is characterized by the cutoff $\bar\eta(s_2)$. It remains
to characterize the shape of this cutoff as a function of $s_2$.

Because the multiplicative constant in~\cref{eq:eta_cutoff_terminal} is
positive, the sign of $\bar\eta'(s_2)$ is the sign of the derivative of \(\bigl(g(s_2,0;A_2)\bigr)^{b-1}M(s_2,0;A_2)\).
Define
\begin{align*}
z_0(s_2)
\triangleq
g(s_2,0;A_2)
=
[1+\gamma(1-\pi_2)]s_2-\gamma(1-\pi_2)\alpha_2
\end{align*}
and
\begin{align*}
r_0(s_2)
\triangleq
M(s_2,0;A_2)
=
\lambda R[(1-\pi_2)\alpha_2+\pi_2s_2]-s_2^b.
\end{align*}
Then $\bar\eta(s_2)$ is proportional to
$z_0(s_2)^{b-1}r_0(s_2)$. Differentiating,
\begin{align*}
\frac{d}{ds_2}
\left[z_0(s_2)^{b-1}r_0(s_2)\right]
=
z_0(s_2)^{b-2}
\left[
(b-1)z_0'(s_2)r_0(s_2)+z_0(s_2)r_0'(s_2)
\right].
\end{align*}
Because $z_0(s_2)>0$, the sign of $\bar\eta'(s_2)$ is the sign of
\begin{align*}
H(s_2)
\triangleq
(b-1)[1+\gamma(1-\pi_2)]r_0(s_2)
+
z_0(s_2)(\lambda R\pi_2-bs_2^{b-1}).
\end{align*}
Let $a\triangleq1+\gamma(1-\pi_2)$. Because
\begin{align*}
r_0'(s_2)=\lambda R\pi_2-bs_2^{b-1},
\qquad
r_0''(s_2)=-b(b-1)s_2^{b-2},
\end{align*}
we have \(H(s_2)=(b-1)a r_0(s_2)+z_0(s_2)r_0'(s_2)\),
and hence
\begin{align*}
H'(s_2)
=
b\,a\,r_0'(s_2)+z_0(s_2)r_0''(s_2).
\end{align*}
Differentiating once more gives
\begin{align*}
H''(s_2)
=
-b(b-1)s_2^{b-3}
\left[
a(b+1)s_2+(b-2)z_0(s_2)
\right].
\end{align*}
The bracketed term is strictly positive for every
$s_2\in(\underline{s}_2,\alpha_2)$ and every $b>1$. If $b\ge2$, this is
immediate because $a>0$, $s_2>0$, and $z_0(s_2)>0$. If $1<b<2$, then
$z_0(s_2)<as_2$ and $b-2<0$, so \((b-2)z_0(s_2)>(b-2)as_2\).
Hence
\begin{align*}
a(b+1)s_2+(b-2)z_0(s_2)
>
a(b+1)s_2+(b-2)as_2
=
a(2b-1)s_2>0.
\end{align*}
Therefore $H''(s_2)<0$ on $(\underline{s}_2,\alpha_2)$, so $H$ is
strictly concave.

Moreover, as $s_2\downarrow\underline{s}_2$, we have
$z_0(s_2)\downarrow0$, and hence
\begin{align*}
\lim_{s_2\downarrow\underline{s}_2}H(s_2)
=
(b-1)[1+\gamma(1-\pi_2)]r_0(\underline{s}_2)>0,
\end{align*}
in which the inequality follows from margin positivity under
\cref{ass:wages_revenue}. At the upper endpoint,
$z_0(\alpha_2)=\alpha_2$ and
$r_0(\alpha_2)=\lambda R\alpha_2-\alpha_2^b$, so
\begin{align*}
H(\alpha_2)
=
(b-1)[1+\gamma(1-\pi_2)](\lambda R\alpha_2-\alpha_2^b)
+
\alpha_2(\lambda R\pi_2-b\alpha_2^{b-1})
=
\Delta_2.
\end{align*}

If $\Delta_2\ge0$, then strict concavity of $H$, together with the
positive lower-endpoint value, implies $H(s_2)>0$ throughout
$(\underline{s}_2,\alpha_2)$. Therefore $\bar\eta'(s_2)>0$ throughout
the domain, and $\bar\eta$ is strictly increasing.

Let \(\eta^U\triangleq \lim_{s_2\to\alpha_2^-}\bar\eta(s_2)\).
Because $z_0(s_2)\downarrow0$ as
$s_2\downarrow\underline{s}_2$ and $b>1$, whereas $r_0(s_2)$ remains
finite and positive, we have \(\lim_{s_2\downarrow\underline{s}_2}\bar\eta(s_2)=0\).
Thus, when $\Delta_2\ge0$, strict monotonicity implies if
$\eta\ge\eta^U$, then $\eta<\bar\eta(s_2)$ never holds on
$(\underline{s}_2,\alpha_2)$, so $h_2^*(s_2)=0$ for all
$s_2\in(\underline{s}_2,\alpha_2)$. If $\eta<\eta^U$, strict
monotonicity and the lower endpoint limit imply there exists a
unique threshold $\hat{s}_2\in(\underline{s}_2,\alpha_2)$ such that
$\eta<\bar\eta(s_2)$ if and only if
$s_2\in(\hat{s}_2,\alpha_2)$. By \cref{cor:terminal_cutoff},
this is equivalent to
\begin{align*}
h_2^*(s_2)>0
\quad\Longleftrightarrow\quad
s_2\in(\hat{s}_2,\alpha_2).
\end{align*}

Now suppose $\Delta_2<0$. Then continuity and the positive
lower-endpoint value of $H$ imply $H$ is positive near
$\underline{s}_2$ and negative at $\alpha_2$. Because $H$ is strictly
concave, its upper contour set \(\{s_2\in(\underline{s}_2,\alpha_2): H(s_2)>0\}\)
is an interval. Because this interval contains points arbitrarily close
to $\underline{s}_2$ but excludes points sufficiently close to
$\alpha_2$, there exists a unique
$s_2^\dagger\in(\underline{s}_2,\alpha_2)$ such that
$H(s_2^\dagger)=0$, with $H(s_2)>0$ for
$s_2<s_2^\dagger$ and $H(s_2)<0$ for
$s_2>s_2^\dagger$. Therefore $\bar\eta$ is strictly increasing on
$(\underline{s}_2,s_2^\dagger)$ and strictly decreasing on
$(s_2^\dagger,\alpha_2)$.

Let
\begin{align*}
\eta^\dagger\triangleq \bar\eta(s_2^\dagger),
\qquad
\eta^U\triangleq \lim_{s_2\to\alpha_2^-}\bar\eta(s_2).
\end{align*}
Because $\bar\eta$ is strictly decreasing on
$(s_2^\dagger,\alpha_2)$, we have $\eta^U<\eta^\dagger$. Moreover,
$\eta^U>0$ by margin positivity at the upper endpoint and
$z_0(\alpha_2)=\alpha_2>0$. Hence $0<\eta^U<\eta^\dagger$.

If $\eta\ge\eta^\dagger$, then $\eta<\bar\eta(s_2)$ never holds, so
$h_2^*(s_2)=0$ for all $s_2\in(\underline{s}_2,\alpha_2)$. If
$0<\eta\le\eta^U$, then the lower endpoint limit
$\lim_{s_2\downarrow\underline{s}_2}\bar\eta(s_2)=0$, strict increase
on $(\underline{s}_2,s_2^\dagger)$, and strict decrease to
$\eta^U$ on $(s_2^\dagger,\alpha_2)$ imply there exists a unique
threshold $\hat{s}_2\in(\underline{s}_2,s_2^\dagger)$ such that
$\eta<\bar\eta(s_2)$ if and only if
$s_2\in(\hat{s}_2,\alpha_2)$. If
$\eta^U<\eta<\eta^\dagger$, then the cutoff crosses the level $\eta$
once on each side of $s_2^\dagger$, so there exist unique thresholds
$\hat{s}_2^L\in(\underline{s}_2,s_2^\dagger)$ and
$\hat{s}_2^H\in(s_2^\dagger,\alpha_2)$ such that
$\eta<\bar\eta(s_2)$ if and only if
$s_2\in(\hat{s}_2^L,\hat{s}_2^H)$. Applying
\cref{cor:terminal_cutoff} gives the stated positive-engagement
regions.

It remains to restrict these full-domain regions to the admissible domain
\(D_2=(\bar{\bar{s}}_2,\alpha_2)\). If \(\Delta_2\ge0\), then
\(\bar\eta\) is strictly increasing on \((\underline{s}_2,\alpha_2)\) and
has upper-end limit \(\eta^U\). Hence, if \(\eta\ge\eta^U\), the
inequality \(\eta<\bar\eta(s_2)\) fails for every \(s_2\in D_2\), so
\(E_2(\eta)=\emptyset\). If \(\eta<\eta^U\), the full-domain
positive-engagement region is \((\hat{s}_2,\alpha_2)\). Intersecting
with \(D_2\) gives
\begin{align*}
E_2(\eta)
=
(\hat{s}_2,\alpha_2)\cap(\bar{\bar{s}}_2,\alpha_2)
=
(\hat{s}_2^D,\alpha_2),
\end{align*}
where \(\hat{s}_2^D=\max\{\hat{s}_2,\bar{\bar{s}}_2\}\).

Now suppose \(\Delta_2<0\). If \(s_2^\dagger>\bar{\bar{s}}_2\), the
full-domain peak lies inside \(D_2\). When \(\eta\ge\eta^\dagger\), the
inequality \(\eta<\bar\eta(s_2)\) fails everywhere, so
\(E_2(\eta)=\emptyset\). When \(\eta\le\eta^U\), the full-domain
positive-engagement region is \((\hat{s}_2,\alpha_2)\), and intersecting
with \(D_2\) gives \(E_2(\eta)=(\hat{s}_2^D,\alpha_2)\).
When \(\eta^U<\eta<\eta^\dagger\), the full-domain positive-engagement
region is the bounded band
\((\hat{s}_2^L,\hat{s}_2^H)\), with
\(\hat{s}_2^L<s_2^\dagger<\hat{s}_2^H<\alpha_2\). Since
\(s_2^\dagger>\bar{\bar{s}}_2\), the upper cutoff \(\hat{s}_2^H\) lies
inside \(D_2\), whereas the lower cutoff may lie below the admissible
boundary. Therefore
\begin{align*}
E_2(\eta)
=
(\hat{s}_2^L,\hat{s}_2^H)\cap(\bar{\bar{s}}_2,\alpha_2)
=
(\hat{s}_2^{L,D},\hat{s}_2^H),
\end{align*}
where
\(\hat{s}_2^{L,D}=\max\{\hat{s}_2^L,\bar{\bar{s}}_2\}\).

Finally, suppose \(\Delta_2<0\) and
\(s_2^\dagger\le\bar{\bar{s}}_2\). Then the peak is weakly below the
admissible domain, so \(\bar\eta\) is strictly decreasing on \(D_2\).
Its upper limit on \(D_2\) is \(\eta^D=\lim_{s_2\downarrow\bar{\bar{s}}_2}\bar\eta(s_2)\),
and its lower limit at the AI benchmark is \(\eta^U\). Hence
\(\eta^D>\eta^U\). If \(\eta\le\eta^U\), then
\(\eta<\bar\eta(s_2)\) for every \(s_2\in D_2\), so
\(E_2(\eta)=D_2\). If \(\eta^U<\eta<\eta^D\), strict monotonicity gives a
unique \(\tilde{s}_2\in D_2\) satisfying
\(\bar\eta(\tilde{s}_2)=\eta\), and \(E_2(\eta)=(\bar{\bar{s}}_2,\tilde{s}_2)\).
If \(\eta\ge\eta^D\), the inequality \(\eta<\bar\eta(s_2)\) fails
throughout \(D_2\), so \(E_2(\eta)=\emptyset\). These are exactly the
three cases stated in the proposition.
\hfill \emph{Q.E.D.}

\medskip

\noindent \emph{Proof of \cref{cor:terminal_monotone}.}
On \(D_2\subset(\underline{s}_2,\alpha_2)\), the zero floor does not bind,
and \(g(s_2,h;A_2)>0\) and \(M(s_2,h;A_2)>0\) for every
\(h\in[0,1]\). Under the power wage,
\begin{align*}
F_2(h;s_2,A_2)
=
C\,g(s_2,h;A_2)^{b-1}M(s_2,h;A_2),
\qquad
C=\frac{\beta b(\phi\pi_2+\gamma)}{2\eta}>0.
\end{align*}
We first show \(F_2\) is strictly increasing in \(s_2\) for every
fixed \(h\in[0,1]\). Differentiating,
\begin{align*}
\frac{\partial F_2(h;s_2,A_2)}{\partial s_2}
=
C\,g(s_2,h;A_2)^{b-2}
\Bigl[
(b-1)g_sM(s_2,h;A_2)+g(s_2,h;A_2)M_s
\Bigr],
\end{align*}
where
\begin{align*}
g_s
&\equiv
\frac{\partial g(s_2,h;A_2)}{\partial s_2}
=
1-(1-\pi_2)\bigl(h(\phi\pi_2+\gamma)-\gamma\bigr),\\
M_s
&\equiv
\frac{\partial M(s_2,h;A_2)}{\partial s_2}
=
\lambda R\bigl(\pi_2+\delta h(1-\pi_2)\bigr)-b s_2^{b-1}.
\end{align*}
The function \(h\mapsto g_s\) is affine and decreasing, so its minimum
on \([0,1]\) is attained at \(h=1\), where \(g_s=1-\phi\pi_2(1-\pi_2)>0\)
by \cref{ass:skill}. Hence \(g_s>0\) for all \(h\in[0,1]\). The function
\(h\mapsto M_s\) is nondecreasing, so its minimum is attained at \(h=0\),
where
\begin{align*}
M_s
=
\lambda R\pi_2-b s_2^{b-1}
>
\lambda R\pi_2-b\alpha_2^{b-1}
\ge 0,
\end{align*}
using \(s_2<\alpha_2\), \(b>1\), and
\(\lambda R\pi_2\ge b\alpha_2^{b-1}\). Thus \(M_s>0\) for all
\(h\in[0,1]\). Since \(g>0\) and \(M>0\), every term in the bracket is
positive, so
\begin{align*}
\frac{\partial F_2(h;s_2,A_2)}{\partial s_2}>0
\qquad\text{for all }h\in[0,1]\text{ and }s_2\in D_2.
\end{align*}

By the terminal single-crossing condition,
\(\partial F_2(h;s_2,A_2)/\partial h<0\) on \(D_2\). Hence
\cref{prop:duopoly_terminal} applies on \(D_2\): the equilibrium is
\(h_2^*(s_2)=0\) when \(F_2(0;s_2,A_2)\le\lambda R\delta\), is
\(h_2^*(s_2)=1\) when \(F_2(1;s_2,A_2)\ge\lambda R\delta\), and otherwise
is the unique interior root of \(F_2(h;s_2,A_2)=\lambda R\delta\).
On the interior region, implicit differentiation gives
\begin{align*}
h_2^{*\prime}(s_2)
=
-\frac{\partial F_2(h_2^*(s_2);s_2,A_2)/\partial s_2}
       {\partial F_2(h_2^*(s_2);s_2,A_2)/\partial h}
>0.
\end{align*}
On the two corner regions, \(h_2^*(s_2)\) is constant. Therefore
\(h_2^*(s_2)\) is nondecreasing on \(D_2\), and strictly increasing on
the interior region.

It remains to locate the regimes. Since
\(s_2\mapsto F_2(0;s_2,A_2)\) and
\(s_2\mapsto F_2(1;s_2,A_2)\) are strictly increasing on \(D_2\), each
crosses \(\lambda R\delta\) at most once. Let
\begin{align*}
\hat{s}_2^D
\equiv
\inf\Bigl(
\{s_2\in D_2:F_2(0;s_2,A_2)>\lambda R\delta\}
\cup\{\alpha_2\}
\Bigr),
\end{align*}
and
\begin{align*}
\overline{s}_2^D
\equiv
\inf\Bigl(
\{s_2\in D_2:F_2(1;s_2,A_2)\ge\lambda R\delta\}
\cup\{\alpha_2\}
\Bigr).
\end{align*}
These definitions give the endpoint conventions \(\bar{\bar{s}}_2\le\hat{s}_2^D\le\overline{s}_2^D\le\alpha_2\).
The ordering follows because \(F_2(1;s_2,A_2)<F_2(0;s_2,A_2)\) for every
\(s_2\in D_2\). Thus the set where full engagement is optimal lies weakly
above the set where engagement begins. By
\cref{prop:duopoly_terminal}, the three regimes are
\begin{align*}
h_2^*(s_2)=0
\quad\text{on }(\bar{\bar{s}}_2,\hat{s}_2^D],
\end{align*}
\begin{align*}
h_2^*(s_2)\in(0,1)
\quad\text{on }(\hat{s}_2^D,\overline{s}_2^D),
\end{align*}
and
\begin{align*}
h_2^*(s_2)=1
\quad\text{on }[\overline{s}_2^D,\alpha_2),
\end{align*}
with all intervals understood relative to \(D_2\). If
\(\hat{s}_2^D\in D_2\), then
\(F_2(0;\hat{s}_2^D,A_2)=\lambda R\delta\); if
\(\overline{s}_2^D\in D_2\), then
\(F_2(1;\overline{s}_2^D,A_2)=\lambda R\delta\).

Finally, consider regularity. On the two corner regions, \(h_2^*\) is
constant and hence continuously differentiable with derivative zero. On
the interior region, \((h,s_2)\mapsto F_2(h;s_2,A_2)\) is continuously
differentiable and \(\partial F_2/\partial h<0\), so the implicit
function theorem implies \(h_2^*\) is continuously differentiable
there, with \(h_2^{*\prime}(s_2)>0\).

If \(\hat{s}_2^D\in D_2\), continuity at \(\hat{s}_2^D\) follows by
the argument below. Suppose, toward a contradiction, that
\(h_2^*(s_n)\ge\varepsilon>0\) along a sequence
\(s_n\downarrow\hat{s}_2^D\) in the interior region. Then
\begin{align*}
\lambda R\delta
=
F_2(h_2^*(s_n);s_n,A_2)
\le
F_2(\varepsilon;s_n,A_2)
\to
F_2(\varepsilon;\hat{s}_2^D,A_2)
<
F_2(0;\hat{s}_2^D,A_2)
=
\lambda R\delta,
\end{align*}
a contradiction. Hence \(h_2^*(s_2)\to0\) as
\(s_2\downarrow\hat{s}_2^D\). Similarly, if
\(\overline{s}_2^D\in D_2\), suppose
\(h_2^*(s_n)\le1-\varepsilon\) along a sequence
\(s_n\uparrow\overline{s}_2^D\) in the interior region. Then
\begin{align*}
\lambda R\delta
=
F_2(h_2^*(s_n);s_n,A_2)
\ge
F_2(1-\varepsilon;s_n,A_2)
\to
F_2(1-\varepsilon;\overline{s}_2^D,A_2)
>
F_2(1;\overline{s}_2^D,A_2)
=
\lambda R\delta,
\end{align*}
again a contradiction. Hence \(h_2^*(s_2)\to1\) as
\(s_2\uparrow\overline{s}_2^D\). Therefore \(h_2^*\) is continuous on
\(D_2\) and continuously differentiable on each nonempty region.
\hfill \emph{Q.E.D.}
\medskip

\noindent \emph{Proof of \cref{prop:alpha_effect_competition}.}\label{app:proof_alpha_effect_competition}
\textbf{Part (i).}
Fix \(s_2\in D_2\) and suppose the terminal-period equilibrium is
interior, \(h_2^*(s_2)\in(0,1)\). Then \(F_2(h_2^*(s_2);s_2,A_2)=\lambda R\delta\).
For local changes in \(\alpha_2\) that keep \(s_2\) in the admissible
domain \(D_2\) and keep the equilibrium interior, implicit
differentiation gives
\begin{align*}
\frac{\partial h_2^*(s_2)}{\partial \alpha_2}
=
-\frac{\partial F_2/\partial \alpha_2}{\partial F_2/\partial h}.
\end{align*}
By \cref{prop:duopoly_terminal},
$\partial F_2/\partial h<0$. It therefore suffices to show
$\partial F_2/\partial \alpha_2>0$ at
$h=h_2^*(s_2)$.

Recall that
\begin{align*}
F_2(h;s_2,A_2)
=
\frac{\beta(\phi\pi_2+\gamma)}{2\eta}
B'\!\bigl(g(s_2,h;A_2)\bigr)
M(s_2,h;A_2).
\end{align*}
The multiplicative constant is positive, so the sign of
$\partial F_2/\partial \alpha_2$ is the sign of
\begin{align*}
B''\!\bigl(g(s_2,h;A_2)\bigr)
M(s_2,h;A_2)
\frac{\partial g(s_2,h;A_2)}{\partial \alpha_2}
+
B'\!\bigl(g(s_2,h;A_2)\bigr)
\frac{\partial M(s_2,h;A_2)}{\partial \alpha_2}.
\end{align*}
Now
\begin{align*}
\frac{\partial g(s_2,h;A_2)}{\partial \alpha_2}
=
(1-\pi_2)(\phi\pi_2+\gamma)(h-h_2^c),
\qquad
h_2^c=\frac{\gamma}{\phi\pi_2+\gamma},
\end{align*}
and
\begin{align*}
\frac{\partial M(s_2,h;A_2)}{\partial \alpha_2}
=
\lambda R(1-\pi_2)(1-\delta h).
\end{align*}
Thus, after factoring out the positive term
$(1-\pi_2)B'\!\bigl(g(s_2,h;A_2)\bigr)$, the sign of
$\partial F_2/\partial \alpha_2$ is the sign of
\begin{align*}
\frac{B''(g(s_2,h;A_2))}{B'(g(s_2,h;A_2))}
M(s_2,h;A_2)
(\phi\pi_2+\gamma)(h-h_2^c)
+
\lambda R(1-\delta h).
\end{align*}

If $h\ge h_2^c$, this expression is strictly positive because
$B'>0$, $B''>0$, $M>0$, and $1-\delta h>0$. Now consider
$h<h_2^c$. By condition~\cref{eq:monotone_condition},
\begin{align*}
\frac{B''(g(s_2,h;A_2))}{B'(g(s_2,h;A_2))}
M(s_2,h;A_2)
<
\frac{\lambda R\delta}{\phi\pi_2+\gamma}.
\end{align*}
Because $h-h_2^c<0$, multiplying both sides by
$(\phi\pi_2+\gamma)(h-h_2^c)$ reverses the inequality and gives
\begin{align*}
\frac{B''(g(s_2,h;A_2))}{B'(g(s_2,h;A_2))}
M(s_2,h;A_2)
(\phi\pi_2+\gamma)(h-h_2^c)
>
\lambda R\delta(h-h_2^c).
\end{align*}
Therefore,
\begin{align*}
\begin{aligned}
&\frac{B''(g(s_2,h;A_2))}{B'(g(s_2,h;A_2))}
M(s_2,h;A_2)
(\phi\pi_2+\gamma)(h-h_2^c)
+
\lambda R(1-\delta h) \\
&\qquad >
\lambda R\delta(h-h_2^c)+\lambda R(1-\delta h) \\
&\qquad =
\lambda R(1-\delta h_2^c)>0,
\end{aligned}
\end{align*}
because $\delta\in(0,1)$ and $h_2^c\in(0,1)$.

Hence $\partial F_2/\partial \alpha_2>0$ at
$h=h_2^*(s_2)$. Because $\partial F_2/\partial h<0$, implicit
differentiation implies \(\frac{\partial h_2^*(s_2)}{\partial \alpha_2}>0\).
\medskip
\noindent For Part~(ii), which establishes the comparative static in \(\pi_2\) under the power wage \(B(s)=W(s)=s^b\) with \(b>1\), we first define the threshold values used in the statement of \cref{prop:alpha_effect_competition}. Fix \(s_2\in D_2\), and define two functions of
\(\pi_2\),
\begin{align*}
L_+(\pi_2)
&\triangleq
\bigl[\lambda R\alpha_2-s_2^b-\lambda R\pi_2(\alpha_2-s_2)\bigr]
\,\frac{\phi}{\phi\pi_2+\gamma},\\
L_-(\pi_2)
&\triangleq
\bigl[\lambda R\alpha_2-\lambda R\delta(\alpha_2-s_2)-s_2^b
-\lambda R\pi_2(1-\delta)(\alpha_2-s_2)\bigr]\\
&\quad\times
\left[
\frac{\phi}{\phi\pi_2+\gamma}
+
(b-1)\,\frac{\max\{\gamma,\,\phi(1-2\pi_2)\}\,(\alpha_2-s_2)}
{s_2-\gamma(1-\pi_2)(\alpha_2-s_2)}
\right],
\end{align*}
both strictly decreasing in \(\pi_2\). Each function tracks the net effect of reliability on the engagement gain through its two channels: a higher \(\pi_2\) lets a unit of engagement build more skill, strengthening the skill-value channel, but it also makes AI fail more often, eroding the current margin on a below-benchmark worker (the operating-margin channel). \(L_+\) bounds this net effect from below and \(L_-\) from above, so a large \(L_+\) certifies that the skill-value channel dominates, whereas a small \(L_-\) certifies that the operating-margin channel dominates. Both fall in \(\pi_2\) because more frequent failure steadily tilts the balance from skill building toward margin erosion. Let \(\bar\pi^+\) and \(\bar\pi^-\), when they exist, denote the unique thresholds defined by \(L_+(\bar\pi^+)=\lambda R(\alpha_2-s_2)\) and \(L_-(\bar\pi^-)=\lambda R(1-\delta)(\alpha_2-s_2)\). If one equation has no solution in the relevant range, the corresponding sign region is empty.

\noindent\textbf{Part (ii).}
\label{app:proof_pi_comparative_statics}
Fix a local range of \(\pi_2\) values over which \(s_2\) remains in
\(D_2\), the terminal single-crossing condition holds, and the equilibrium
remains interior. Hold \(s_2\) and \(\alpha_2\) fixed. Under $B(s)=W(s)=s^b$,
\begin{align*}
F_2(h;s_2,A_2)
=
\frac{\beta b}{2\eta}(\phi\pi_2+\gamma)
\bigl(g(s_2,h;A_2)\bigr)^{b-1}M(s_2,h;A_2).
\end{align*}
Because the equilibrium is interior,
$F_2(h_2^*(s_2,\pi_2);s_2,A_2)=\lambda R\delta$, and implicit
differentiation gives
\begin{align*}
\frac{\partial h_2^*(s_2,\pi_2)}{\partial \pi_2}
=
-\frac{\partial F_2/\partial \pi_2}{\partial F_2/\partial h}.
\end{align*}
Because \(s_2\in D_2\) throughout this local range, the terminal
single-crossing condition in \cref{eq:monotone_condition} applies. Hence,
by \cref{prop:duopoly_terminal}, 
\(\partial F_2(h;s_2,A_2)/\partial h<0\). Therefore
\begin{align*}
\operatorname{sign}\!\left(
\frac{\partial h_2^*(s_2,\pi_2)}{\partial \pi_2}
\right)
=
\operatorname{sign}\!\left(
\frac{\partial F_2(h;s_2,A_2)}{\partial \pi_2}
\Big|_{h=h_2^*(s_2,\pi_2)}
\right).
\end{align*}

For fixed $h$, the derivatives of the skill transition and margin with
respect to $\pi_2$ are
\begin{align*}
\frac{\partial g(s_2,h;A_2)}{\partial \pi_2}
=
(\alpha_2-s_2)\{\gamma+h(\phi-\gamma-2\phi\pi_2)\},
\qquad
\frac{\partial M(s_2,h;A_2)}{\partial \pi_2}
=
-\lambda R(\alpha_2-s_2)(1-\delta h).
\end{align*}
Taking the logarithmic derivative of $F_2$ with respect to $\pi_2$
therefore gives
\begin{align*}
\mathcal A(s_2,\pi_2,h)
&\triangleq
\frac{1}{F_2(h;s_2,A_2)}
\frac{\partial F_2(h;s_2,A_2)}{\partial \pi_2} \\
&=
\frac{\phi}{\phi\pi_2+\gamma}
+
(b-1)
\frac{(\alpha_2-s_2)\{\gamma+h(\phi-\gamma-2\phi\pi_2)\}}
     {g(s_2,h;A_2)}
-
\frac{\lambda R(\alpha_2-s_2)(1-\delta h)}
     {M(s_2,h;A_2)}.
\end{align*}
Because $F_2(h;s_2,A_2)>0$, the sign of $\mathcal A(s_2,\pi_2,h)$ is the
sign of $\partial F_2/\partial\pi_2$ at that same $h$. Combined with the
display above, $\partial h_2^*(s_2,\pi_2)/\partial\pi_2$ has the sign of
$\mathcal A$ at the equilibrium point $h=h_2^*(s_2,\pi_2)$. Because
$h_2^*(s_2,\pi_2)$ has no closed form, we instead establish a sign for
$\mathcal A(s_2,\pi_2,h)$ that holds uniformly over all $h\in[0,1]$; such
a sign holds in particular at $h=h_2^*(s_2,\pi_2)$, and the conclusion
follows.

We first show $\partial h_2^*(s_2,\pi_2)/\partial \pi_2>0$ when
$\pi_2<\bar\pi^+$. Suppose $\pi_2\le1/2$. Then
\begin{align*}
\gamma+h(\phi-\gamma-2\phi\pi_2)\ge0
\qquad\text{for all } h\in[0,1].
\end{align*}
The middle term in $\mathcal A$ is therefore nonnegative, and the margin ratio strictly decreases in $h$, because
\begin{align*}
\frac{\partial}{\partial h}
\left[
\frac{\lambda R(\alpha_2-s_2)(1-\delta h)}
     {M(s_2,h;A_2)}
\right]
=
\lambda R(\alpha_2-s_2)
\frac{\delta s_2(s_2^{b-1}-\lambda R)}
     {M(s_2,h;A_2)^2}<0,
\end{align*}
in which the inequality follows from \cref{ass:wages_revenue},
which implies $\lambda R>\alpha_2^{b-1}>s_2^{b-1}$. Therefore this ratio is maximized
at $h=0$, and for every $h\in[0,1]$,
\begin{align*}
\mathcal A(s_2,\pi_2,h)
\ge
\frac{\phi}{\phi\pi_2+\gamma}
-
\frac{\lambda R(\alpha_2-s_2)}
     {M(s_2,0;A_2)}
=
\frac{L_+(\pi_2)-\lambda R(\alpha_2-s_2)}{M(s_2,0;A_2)}.
\end{align*}
Because $M(s_2,0;A_2)>0$ and $L_+$ is strictly decreasing in $\pi_2$, the
right-hand side is strictly positive precisely when $\pi_2<\bar\pi^+$.
Thus $\mathcal A(s_2,\pi_2,h)>0$ for all $h\in[0,1]$, and in particular
at $h=h_2^*(s_2,\pi_2)$, so \(\frac{\partial h_2^*(s_2,\pi_2)}{\partial \pi_2}>0\).
We next show $\partial h_2^*(s_2,\pi_2)/\partial \pi_2<0$ when
$\pi_2>\bar\pi^-$, by bounding $\mathcal A$ from above. First,
\begin{align*}
\gamma+h(\phi-\gamma-2\phi\pi_2)
\le
\max\{\gamma,\,\phi(1-2\pi_2)\}
\qquad\text{for all } h\in[0,1],
\end{align*}
and $g(s_2,h;A_2)\ge g(s_2,0;A_2)$, so the middle term in $\mathcal A$ is
bounded above by
$(b-1)(\alpha_2-s_2)\max\{\gamma,\phi(1-2\pi_2)\}/g(s_2,0;A_2)$. Second,
the margin ratio is decreasing in $h$, so it is minimized at $h=1$, and
the negative margin term is bounded above by
$-\lambda R(\alpha_2-s_2)(1-\delta)/M(s_2,1;A_2)$. Combining these bounds,
for every $h\in[0,1]$,
\begin{align*}
\mathcal A(s_2,\pi_2,h)
&\le
\frac{\phi}{\phi\pi_2+\gamma}
+
(b-1)
\frac{(\alpha_2-s_2)\max\{\gamma,\,\phi(1-2\pi_2)\}}
     {g(s_2,0;A_2)}
-
\frac{\lambda R(\alpha_2-s_2)(1-\delta)}
     {M(s_2,1;A_2)}
\\
&=
\frac{L_-(\pi_2)-\lambda R(1-\delta)(\alpha_2-s_2)}{M(s_2,1;A_2)}.
\end{align*}
Because $M(s_2,1;A_2)>0$ and $L_-$ is strictly decreasing in $\pi_2$, the
right-hand side is strictly negative precisely when $\pi_2>\bar\pi^-$.
Thus $\mathcal A(s_2,\pi_2,h)<0$ for all $h\in[0,1]$, and in particular
at $h=h_2^*(s_2,\pi_2)$, so \(\frac{\partial h_2^*(s_2,\pi_2)}{\partial \pi_2}<0\).
\hfill \emph{Q.E.D.}

\medskip

\noindent \emph{Proof of \cref{lem:period1_candidates}.}\label{app:proof_period1_candidates}
Fix $s_1$ as stated and suppose $(h_1^*,h_1^*)$ is a symmetric
equilibrium. For brevity, write
\[
\sigma_1^j=\sigma_1^j(h_1^j,h_1^{-j};s_1,A_1,A_2).
\]
Then $x\mapsto\Phi_1^j(x,h_1^*;s_1,A_1,A_2)$ attains its
maximum over $[0,1]$ at $x=h_1^*$. Differentiating
\cref{eq:period1_objective} in $h_1^j$,
\begin{align*}
\frac{\partial\Phi_1^j}{\partial h_1^j}
=
\frac{\partial\sigma_1^j}{\partial h_1^j}
\Bigl[M(s_1,h_1^j;A_1)
+\beta\Pi_2^*\!\bigl(g(s_1,h_1^j;A_1),A_2\bigr)
-\beta\Pi_2^*\!\bigl(g(s_1,h_1^{-j};A_1),A_2\bigr)\Bigr]\\
+\sigma_1^j\!\left[
\frac{\partial M(s_1,h_1^j;A_1)}{\partial h_1^j}
+\beta\,\frac{d\,\Pi_2^*\!\bigl(g(s_1,h_1^j;A_1),A_2\bigr)}{dh_1}
\right].
\end{align*}
At the symmetric point $h_1^j=h_1^{-j}=h_1^*$, the two continuation
terms in the first bracket cancel, $\sigma_1^j=1/2$, and the logit
derivative is
$\partial\sigma_1^j/\partial h_1^j\big|_{\mathrm{sym}}
=(\beta/4\eta)\,d\bar V_2\!\bigl(g(s_1,h_1^*;A_1),A_2\bigr)/dh_1$,
so the own-action derivative at the symmetric profile equals
$F_1(h_1^*;s_1,A_1,A_2)$. (By the differentiability assumption in the lemma, \(F_1\) is well defined at \(h_1^*\).)
Optimality of $h_1^*$ on $[0,1]$ requires the directional derivative
toward any feasible $h_1\in[0,1]$ to be nonpositive:
$F_1(h_1^*;s_1,A_1,A_2)(h_1-h_1^*)\le0$ for all $h_1\in[0,1]$. If
$h_1^*\in(0,1)$, both directions are feasible, forcing
$F_1(h_1^*;s_1,A_1,A_2)=0$; if $h_1^*=0$, only $h_1>0$ is feasible,
giving $F_1(0;s_1,A_1,A_2)\le0$; if $h_1^*=1$, only $h_1<1$ is
feasible, giving $F_1(1;s_1,A_1,A_2)\ge0$.
\hfill \emph{Q.E.D.}

\medskip

\noindent \emph{Proof of \cref{prop:period1_regimes}.}
\label{app:proof_period1_regimes}
Fix \(s_1\in D_1\) satisfying either \(s_1<s_1^D\) or
\(s_1>s_1^F\). By \cref{ass:skill},
\begin{align*}
g(s_1,1;A_1)
=
\bigl[1-\phi\pi_1(1-\pi_1)\bigr]s_1
+\phi\pi_1(1-\pi_1)\alpha_1
\end{align*}
and
\begin{align*}
g(s_1,0;A_1)
=
\bigl[1+\gamma(1-\pi_1)\bigr]s_1
-\gamma(1-\pi_1)\alpha_1
\end{align*}
are strictly increasing in \(s_1\). By the definitions in
\cref{eq:s1_regime_thresholds},
\begin{align*}
g(s_1^D,1;A_1)=\hat{s}_2^D,
\qquad
g(s_1^F,0;A_1)=\overline{s}_2^D.
\end{align*}
Thus \(s_1<s_1^D\) implies \(g(s_1,1;A_1)<\hat{s}_2^D\),
and \(s_1>s_1^F\) implies \(g(s_1,0;A_1)>\overline{s}_2^D\).
Because \(g(s_1,\cdot\,;A_1)\) is increasing, the reachable continuation
set \(I\triangleq[g(s_1,0;A_1),g(s_1,1;A_1)]\)
is a compact interval. Since \(s_1\in D_1\), we have \(I\subset D_2\).
If \(s_1<s_1^D\), then \(I\subset(\bar{\bar{s}}_2,\hat{s}_2^D)\),
so \cref{cor:terminal_monotone} gives \(h_2^*(s_2)=0\) for all
\(s_2\in I\). If \(s_1>s_1^F\), then \(I\subset(\overline{s}_2^D,\alpha_2)\),
so \cref{cor:terminal_monotone} gives \(h_2^*(s_2)=1\) for all
\(s_2\in I\). In either case, the inequalities are strict and \(I\) is
compact, so \(I\) is bounded away from the relevant terminal cutoff.
Hence the terminal policy is constant on an open interval containing
\(I\).

In the no-terminal-engagement case \(s_1<s_1^D\), on this open interval,
\begin{align*}
\Pi_2^*(s_2,A_2)=\frac12 M(s_2,0;A_2),
\qquad
\bar V_2(s_2,A_2)
=
s_2^b+\beta\bigl(g(s_2,0;A_2)\bigr)^b+\eta\log2.
\end{align*}
In the full-terminal-engagement case \(s_1>s_1^F\), on this open
interval,
\begin{align*}
\Pi_2^*(s_2,A_2)=\frac12 M(s_2,1;A_2),
\qquad
\bar V_2(s_2,A_2)
=
s_2^b+\beta\bigl(g(s_2,1;A_2)\bigr)^b+\eta\log2.
\end{align*}
In both cases, \(\Pi_2^*(\cdot,A_2)\) and \(\bar V_2(\cdot,A_2)\) are
\(C^2\) on the open interval containing \(I\), because
\(g(s_2,0;A_2)>0\) for \(s_2>\underline{s}_2\) and
\(g(s_2,1;A_2)>0\) throughout.

In the no-terminal-engagement case,
\begin{align*}
\Pi_2^{*\prime}(s_2,A_2)
=
\frac12\bigl[\lambda R\pi_2-b\,s_2^{b-1}\bigr],
\qquad
\Pi_2^{*\prime\prime}(s_2,A_2)
=
-\frac{b(b-1)}{2}s_2^{b-2}<0,
\end{align*}
and
\begin{align*}
\bar V_2'(s_2,A_2)
=
b\,s_2^{b-1}
+
\beta b\bigl[1+\gamma(1-\pi_2)\bigr]
\bigl(g(s_2,0;A_2)\bigr)^{b-1}>0.
\end{align*}
In the full-terminal-engagement case,
\begin{align*}
\Pi_2^{*\prime}(s_2,A_2)
=
\frac12\bigl[
\lambda R\pi_2+\lambda R\delta(1-\pi_2)-b\,s_2^{b-1}
\bigr],
\qquad
\Pi_2^{*\prime\prime}(s_2,A_2)
=
-\frac{b(b-1)}{2}s_2^{b-2}<0,
\end{align*}
and
\begin{align*}
\bar V_2'(s_2,A_2)
=
b\,s_2^{b-1}
+
\beta b\bigl[1-\phi\pi_2(1-\pi_2)\bigr]
\bigl(g(s_2,1;A_2)\bigr)^{b-1}>0.
\end{align*}

Now write
\begin{align*}
g_1\triangleq g(s_1,h_1;A_1),
\qquad
g_h\triangleq
\frac{\partial g(s_1,h_1;A_1)}{\partial h_1}
=
(1-\pi_1)(\phi\pi_1+\gamma)(\alpha_1-s_1)>0.
\end{align*}
The derivative \(g_h\) is constant in \(h_1\). Therefore, on the
terminal-corner region under consideration, the marginal gain can be
written as
\begin{align*}
F_1(h_1;s_1,A_1,A_2)
=
\frac{1}{\eta}A(h_1)+C(h_1),
\end{align*}
where
\begin{align*}
A(h_1)
\triangleq
\frac{\beta}{4}\,
\bar V_2'(g_1,A_2)\,g_h\,M(s_1,h_1;A_1),
\end{align*}
and
\begin{align*}
C(h_1)
\triangleq
-\frac12\lambda R\delta(1-\pi_1)(\alpha_1-s_1)
+
\frac{\beta}{2}\Pi_2^{*\prime}(g_1,A_2)g_h.
\end{align*}
Both \(A\) and \(C\) are \(C^1\) on \([0,1]\). Moreover,
\begin{align*}
C'(h_1)
=
\frac{\beta}{2}\Pi_2^{*\prime\prime}(g_1,A_2)g_h^2
\le -c<0,
\end{align*}
where
\begin{align*}
c
\triangleq
\frac{\beta b(b-1)}{4}g_h^2
\min_{s_2\in I}s_2^{b-2}>0.
\end{align*}
Let \(K\triangleq\max_{h_1\in[0,1]}|A'(h_1)|<\infty\).
Then, for every \(\eta\ge 2K/c\),
\begin{align*}
F_1'(h_1;s_1,A_1,A_2)
=
\frac{1}{\eta}A'(h_1)+C'(h_1)
\le
\frac{K}{\eta}-c
\le
-\frac{c}{2}<0
\end{align*}
for every \(h_1\in[0,1]\). Hence \(F_1\) is strictly decreasing on
\([0,1]\).
Strict monotonicity of \(F_1\) gives uniqueness of the symmetric
candidate selected by the variational inequality. To show this
candidate is an equilibrium, however, the first-order condition must be
sufficient for a global best response. This is why we next establish
strict concavity of \(x\mapsto\Phi_1^j(x,y;s_1,A_1,A_2)\) for each fixed
other-firm action \(y\). Under this concavity, for any candidate \(h\) and any
deviation \(z\in[0,1]\),
\begin{align*}
\Phi_1^j(z,h)-\Phi_1^j(h,h)
\le
\frac{\partial \Phi_1^j}{\partial x}(h,h)(z-h)
=
F_1(h;s_1,A_1,A_2)(z-h).
\end{align*}
Thus any \(h\) satisfying the variational inequality is a best response
to itself.
It remains to verify, for sufficiently large \(\eta\), each firm's
objective is strictly concave in its own action. Fix any other-firm action
\(y\in[0,1]\), write \(x\) for the firm's own action, and define \(G(x)\triangleq g(s_1,x;A_1)\).
Let \(\sigma=\sigma_1^j(x,y;s_1,A_1,A_2)\)
and
\begin{align*}
\Delta(x,y)
\triangleq
M(s_1,x;A_1)
+\beta\Pi_2^*(G(x),A_2)
-\beta\Pi_2^*(G(y),A_2).
\end{align*}
Then
\begin{align*}
\Phi_1^j(x,y;s_1,A_1,A_2)
=
\sigma\,\Delta(x,y)
+
\beta\Pi_2^*(G(y),A_2).
\end{align*}
Differentiating the logit gives
\begin{align*}
\sigma_x
=
\sigma(1-\sigma)\frac{\beta}{\eta}
\bar V_2'(G(x),A_2)g_h
\end{align*}
and
\begin{align*}
\sigma_{xx}
=
\sigma(1-\sigma)(1-2\sigma)
\left(\frac{\beta}{\eta}\right)^2
\bigl(\bar V_2'(G(x),A_2)g_h\bigr)^2
+
\sigma(1-\sigma)\frac{\beta}{\eta}
\bar V_2''(G(x),A_2)g_h^2.
\end{align*}
Since
\begin{align*}
\Delta_x(x,y)
=
\frac{\partial M(s_1,x;A_1)}{\partial h_1}
+
\beta\Pi_2^{*\prime}(G(x),A_2)g_h
\end{align*}
and
\begin{align*}
\Delta_{xx}(x,y)
=
\beta\Pi_2^{*\prime\prime}(G(x),A_2)g_h^2,
\end{align*}
we obtain
\begin{align*}
\frac{\partial^2\Phi_1^j}{\partial x^2}
=
\sigma\Biggl\{&
(1-\sigma)
\left[
(1-2\sigma)\frac{\beta^2}{\eta^2}
\bigl(\bar V_2'(G(x),A_2)g_h\bigr)^2
+
\frac{\beta}{\eta}\bar V_2''(G(x),A_2)g_h^2
\right]\Delta(x,y)\\
&+
2(1-\sigma)\frac{\beta}{\eta}
\bar V_2'(G(x),A_2)g_h
\left[
\frac{\partial M(s_1,x;A_1)}{\partial h_1}
+
\beta\Pi_2^{*\prime}(G(x),A_2)g_h
\right]\\
&+
\beta\Pi_2^{*\prime\prime}(G(x),A_2)g_h^2
\Biggr\}.
\end{align*}
For \(x,y\in[0,1]\), we have \(G(x),G(y)\in I\). Since
\(\bar V_2\) and \(\Pi_2^*\) are \(C^2\) on a neighborhood of
\(I\), the quantities
\begin{align*}
\bar V_2'(G(x),A_2),\quad
|\bar V_2''(G(x),A_2)|,\quad
|\Delta(x,y)|,
\quad\text{and}\quad
\left|
\frac{\partial M(s_1,x;A_1)}{\partial h_1}
+
\beta\Pi_2^{*\prime}(G(x),A_2)g_h
\right|
\end{align*}
are bounded uniformly on \([0,1]^2\). Also, \((1-\sigma)\le1\), \(|1-2\sigma|\le1\), and the final term inside braces satisfies \(\beta\Pi_2^{*\prime\prime}(G(x),A_2)g_h^2\le -2c\).
Hence there exist finite constants \(K_1\) and \(K_2\), independent of
\(\eta\), such that the expression inside braces is at most \(\frac{K_1}{\eta}+\frac{K_2}{\eta^2}-2c\).
For \(\eta\ge \max\left\{1,\frac{K_1+K_2}{c}\right\}\),
this upper bound is at most \(-c<0\). Since \(\sigma>0\), it follows that \(\frac{\partial^2\Phi_1^j}{\partial x^2}<0\)
for every \(x,y\in[0,1]\). Thus each firm's objective is strictly
concave in its own action, uniformly in the other firm's action.

The bounds above depend only on the primitives, \(s_1\), \(A_1\), \(A_2\),
and the terminal corner under consideration, but not on \(\eta\). If both
terminal-corner cases are possible for different values of the cutoffs,
take \(\eta_1(s_1)\) to be the maximum of the corresponding bounds across
the no-terminal-engagement and full-terminal-engagement cases. Thus, for
every \(\eta\ge\eta_1(s_1)\), whenever \(s_1<s_1^D\) or \(s_1>s_1^F\),
the function \(F_1\) is strictly decreasing on \([0,1]\), and each firm's
objective is strictly concave in its own action.

It remains to characterize the symmetric equilibrium. Fix
\(\eta\ge\eta_1(s_1)\). By strict concavity, a firm has no profitable
deviation from own action \(h\) against the other firm's action \(h\) if and only if
the directional derivative toward every feasible action is nonpositive.
By the computation in the proof of \cref{lem:period1_candidates}, the
own-action derivative at the symmetric profile \((h,h)\) equals
\(F_1(h;s_1,A_1,A_2)\). Hence \((h,h)\) is a symmetric equilibrium if and
only if
\begin{align*}
F_1(h;s_1,A_1,A_2)(z-h)\le0
\qquad
\forall z\in[0,1].
\end{align*}
Because \(F_1\) is strictly decreasing, exactly one \(h\) satisfies this
condition. If \(F_1(0;s_1,A_1,A_2)\le0\),
then \(h=0\) satisfies the variational inequality, and every \(h>0\) has
\(F_1(h;s_1,A_1,A_2)<0\), so no interior point or upper corner can
satisfy it. If \(F_1(1;s_1,A_1,A_2)\ge0\),
then \(h=1\) satisfies the variational inequality, and every \(h<1\) has
\(F_1(h;s_1,A_1,A_2)>0\), so no other point can satisfy it. Otherwise, \(F_1(0;s_1,A_1,A_2)>0>F_1(1;s_1,A_1,A_2)\),
and continuity plus strict monotonicity gives a unique interior root of \(F_1(h;s_1,A_1,A_2)=0\).
This root satisfies the variational inequality, while both corners fail.
This proves the stated characterization of the unique symmetric
equilibrium \(h_1^*(s_1)\).

Finally, the realized continuation state
\(g(s_1,h_1^*(s_1);A_1)\) lies in \(I\). Therefore, by the
terminal-corner argument above, \(h_2^*(g(s_1,h_1^*(s_1);A_1))=0\)
when \(s_1<s_1^D\), and \(h_2^*(g(s_1,h_1^*(s_1);A_1))=1\)
when \(s_1>s_1^F\).
\hfill \emph{Q.E.D.}
\medskip

\noindent \emph{Proof of \cref{cor:mobility_vs_monopoly}.}
\label{app:proof_mobility_vs_monopoly}
Fix \(s_1\in D_1\), and let
\(I=[g(s_1,0;A_1),g(s_1,1;A_1)]\) be the set of period-2 skills reachable
from \(s_1\). Since \(s_1\in D_1\), we have \(I\subset D_2\). On the
region \(s_1<s_1^D\), \cref{prop:period1_regimes} implies
\(h_2^*(g(s_1,h_1;A_1))=0\) for all \(h_1\in[0,1]\). Equivalently, on
\(I\),
\(\Pi_2^*(s_2,A_2)=\frac12 M(s_2,0;A_2)\), and hence
\begin{align*}
\Pi_2^{*\prime}(s_2,A_2)
=
\frac12\bigl[\lambda R\pi_2-b\,s_2^{b-1}\bigr].
\end{align*}
Moreover, on this terminal corner,
\(\bar V_2(s_2,A_2)=s_2^b+\beta(g(s_2,0;A_2))^b+\eta\log2\), so
\(\bar V_2'(s_2,A_2)\) does not depend on \(\eta\).

Write \(g_1\equiv g(s_1,h_1;A_1)\) and
\(g_h\equiv(1-\pi_1)(\phi\pi_1+\gamma)(\alpha_1-s_1)>0\). The single
firm's period-1 marginal payoff is
\begin{align*}
F_1^{\mathrm{sf}}(h_1)
=
\frac{\partial M(s_1,h_1;A_1)}{\partial h_1}
+
\beta\bigl[\lambda R\pi_2-b\,g_1^{b-1}\bigr]g_h.
\end{align*}
By \cref{prop:monopoly_period1}, \(h_1^{\mathrm{sf}}(s_1)\) is characterized by the
Karush--Kuhn--Tucker conditions for this strictly concave single-firm
problem.

Using the display for \(\Pi_2^{*\prime}\) above in the definition of
\(F_1(h_1;s_1,A_1,A_2)\), we obtain
\begin{align*}
F_1(h_1;s_1,A_1,A_2)
=
\frac{1}{\eta}
\left[
\frac{\beta}{4}\bar V_2'(g_1,A_2)g_h M(s_1,h_1;A_1)
\right]
+\frac14 F_1^{\mathrm{sf}}(h_1)
+\frac14\frac{\partial M(s_1,h_1;A_1)}{\partial h_1}.
\end{align*}
The term in square brackets is continuous on \([0,1]\), nonnegative, and
independent of \(\eta\). Let
\begin{align*}
K_R
\triangleq
\max_{h_1\in[0,1]}
\frac{\beta}{4}\bar V_2'(g(s_1,h_1;A_1),A_2)g_hM(s_1,h_1;A_1)
<\infty.
\end{align*}
Because \(I\subset D_2\) is compact and \(\bar\eta(s_2)\) is continuous
on \(D_2\), \(\max_{s_2\in I}\bar\eta(s_2)<\infty\). Define
\begin{align*}
\eta_2(s_1)
\triangleq
\max\left\{
\eta_1(s_1),\;
2\max_{s_2\in I}\bar\eta(s_2),\;
\frac{8K_R}{\lambda R\delta(1-\pi_1)(\alpha_1-s_1)}
\right\}.
\end{align*}
Then \(\eta_2(s_1)\ge\eta_1(s_1)\). For every
\(\eta\ge\eta_2(s_1)\), the inequality
\(\eta>\bar\eta(s_2)\) holds for every \(s_2\in I\), so terminal
engagement is zero throughout the reachable continuation set. Thus the
no-terminal-engagement case of \cref{prop:period1_regimes} applies.

For such \(\eta\), since
\begin{align*}
\frac{\partial M(s_1,h_1;A_1)}{\partial h_1}
=
-\lambda R\delta(1-\pi_1)(\alpha_1-s_1),
\end{align*}
we have, for every \(h_1\in[0,1]\),
\begin{align*}
\frac{K_R}{\eta}
+
\frac14\frac{\partial M(s_1,h_1;A_1)}{\partial h_1}
\le
\frac{\lambda R\delta(1-\pi_1)(\alpha_1-s_1)}{8}
-
\frac{\lambda R\delta(1-\pi_1)(\alpha_1-s_1)}{4}
<0.
\end{align*}
Therefore
\begin{align*}
F_1(h_1;s_1,A_1,A_2)
<
\frac14 F_1^{\mathrm{sf}}(h_1)
\qquad \forall h_1\in[0,1].
\end{align*}

We now compare the equilibrium conditions. If
\(h_1^{\mathrm{sf}}(s_1)\in(0,1)\), then \(F_1^{\mathrm{sf}}(h_1^{\mathrm{sf}}(s_1))=0\), so
\(F_1(h_1^{\mathrm{sf}}(s_1);s_1,A_1,A_2)<0\). Since
\cref{prop:period1_regimes} gives that \(F_1(h_1;s_1,A_1,A_2)\) is
strictly decreasing in \(h_1\), the symmetric equilibrium under mobility
must lie strictly below \(h_1^{\mathrm{sf}}(s_1)\). Indeed, \(h_1^*(s_1)=1\) is
impossible because \(F_1(1;s_1,A_1,A_2)<0\), and any interior root of
\(F_1(h_1;s_1,A_1,A_2)=0\) must occur to the left of \(h_1^{\mathrm{sf}}(s_1)\).

If \(h_1^{\mathrm{sf}}(s_1)=0\), the single firm's KKT condition gives
\(F_1^{\mathrm{sf}}(0)\le0\). Hence \(F_1(0;s_1,A_1,A_2)<0\), and
\cref{prop:period1_regimes} implies \(h_1^*(s_1)=0=h_1^{\mathrm{sf}}(s_1)\). If
\(h_1^{\mathrm{sf}}(s_1)=1\), then \(h_1^*(s_1)\le1=h_1^{\mathrm{sf}}(s_1)\) trivially. Thus,
for all \(\eta\ge\eta_2(s_1)\),
\(h_1^*(s_1)\le h_1^{\mathrm{sf}}(s_1)\), with strict inequality whenever
\(h_1^{\mathrm{sf}}(s_1)\in(0,1)\).
\hfill \emph{Q.E.D.}

\medskip

\noindent \emph{Derivation of the cross-partial in
\cref{sec:asymmetry_duopoly}.}
Write $x$ and $y$ for the own and the other firm's actions,
$P(x)\triangleq\Pi_2^*(g(s_1,x;A_1),A_2)$,
$\Delta(x,y)\triangleq M(s_1,x;A_1)+\beta P(x)-\beta P(y)$, and
$u(x,y)\triangleq\beta[\bar V_2(g(s_1,y;A_1),A_2)
-\bar V_2(g(s_1,x;A_1),A_2)]/\eta$, so that
$\sigma_1^j=(1+e^u)^{-1}$ and
$\Phi_1^j=\sigma_1^j\Delta+\beta P(y)$. Then
$\partial\Phi_1^j/\partial x
=\sigma_x\Delta+\sigma_1^j[\partial M/\partial h_1+\beta P'(x)]$, and
\begin{align*}
\frac{\partial^2\Phi_1^j}{\partial x\,\partial y}
=
\sigma_{xy}\,\Delta
-\sigma_x\,\beta P'(y)
+\sigma_y\,\Bigl[\frac{\partial M(s_1,x;A_1)}{\partial h_1}
+\beta P'(x)\Bigr].
\end{align*}
Because $u$ is additively separable in $(x,y)$, $u_{xy}=0$, and at a
symmetric profile $u=0$, where the logistic satisfies
$\sigma_1^j=\tfrac12$ and $d^2\sigma/du^2=0$; hence $\sigma_{xy}=0$
there. Moreover,
$\sigma_x=-\sigma_y
=(\beta/4\eta)\bar V_2'\!\bigl(g(s_1,h_1;A_1)\bigr)g_h$ at the
symmetric profile, with
$g_h=(1-\pi_1)(\phi\pi_1+\gamma)(\alpha_1-s_1)$, and
$P'=\Pi_2^{*\prime}\!\bigl(g(s_1,h_1;A_1),A_2\bigr)g_h$. Substituting,
\begin{align*}
\frac{\partial^2\Phi_1^j}{\partial x\,\partial y}
\bigg|_{x=y=h_1}
=
-\frac{\beta}{4\eta}\bar V_2'\!\bigl(g(s_1,h_1;A_1)\bigr)g_h
\left[\frac{\partial M(s_1,h_1;A_1)}{\partial h_1}
+2\beta\,\Pi_2^{*\prime}\!\bigl(g(s_1,h_1;A_1),A_2\bigr)g_h\right],
\end{align*}
and inserting
$\partial M/\partial h_1=-\lambda R\delta(1-\pi_1)(\alpha_1-s_1)$ and
the value of $g_h$ yields the expression displayed in
\cref{sec:asymmetry_duopoly}.
\hfill \emph{Q.E.D.}

\medskip

\section{Skill-Transition Policy Without Worker Mobility}\label{app:skill_transition_table}

The single firm's optimal period-1 engagement policy, characterized in \cref{prop:monopoly_period1}, takes the firm's target period-2 skill $s^*$ and engages each worker just enough to reach it. \cref{tab:skill_transition} restates that policy as a partition of the initial-skill range by the thresholds $s_L$ and $s_H$ of \cref{eq:s_L_s_H}: workers below $s_L$ are engaged fully, workers above $s_H$ are left alone, and workers in between are held exactly at $s^*$.

\begin{table}[ht]
\centering
\caption{Skill transition under optimal period-1 engagement without worker mobility}
\label{tab:skill_transition}
\begin{threeparttable}
\begin{tabular}{@{}lccc@{}}
\toprule
Targeting region & Initial skill $s_1$ & Engagement $h_1^{\mathrm{sf}}(s_1)$ & Period-2 skill $s_2$ \\
\midrule
Full engagement    & $s_1 \le s_L$     & $1$         & $g(s_1,1;A_1) \le s^*$ \\
Partial engagement & $s_L < s_1 < s_H$ & $\in (0,1)$ & $s^*$ \\
No engagement      & $s_1 \ge s_H$     & $0$         & $g(s_1,0;A_1) \ge s^*$ \\
\bottomrule
\end{tabular}
\begin{tablenotes}
\footnotesize
\item \textit{Notes.} $s^{*}$ is the target period-2 skill; $s_L \le s_H$ are the initial-skill thresholds defined in \cref{eq:s_L_s_H}; and $g(s_1,h_1;A_1)$ is the skill-transition map at period-1 AI state~$A_1$. Equality in the last column holds at the thresholds $s_1=s_L$ and $s_1=s_H$.
\end{tablenotes}
\end{threeparttable}
\end{table}

\section{Coverage of the Terminal Single-Crossing Domain}
\label{app:d2_size}

The power-wage analysis in the main text shows that the terminal
single-crossing condition holds on an upper-tail domain
$D_2=(\bar{\bar{s}}_2,\alpha_2)$. To gauge the size of this restriction,
we report two measures. The first is the share of the full below-benchmark
interval covered by $D_2$,
\[
\frac{\alpha_2-\bar{\bar{s}}_2}{\alpha_2}.
\]
The second is the share of the smooth terminal domain covered by $D_2$,
\[
\frac{\alpha_2-\bar{\bar{s}}_2}{\alpha_2-\underline{s}_2},
\qquad
\underline{s}_2=
\frac{\gamma(1-\pi_2)\alpha_2}{1+\gamma(1-\pi_2)}.
\]
The first measure is the intuitive coverage of below-benchmark workers;
the second isolates the additional restriction imposed by the
single-crossing condition, after the positivity requirement has already
removed the interval below $\underline{s}_2$.

\begin{table}[ht]
\centering
\caption{Size of the terminal domain $D_2$}
\label{tab:d2_size}
\begin{tabular}{cccc}
\toprule
AI capability $\alpha_2$
& Cutoff $\bar{\bar{s}}_2$
& Share of $(0,\alpha_2)$
& Share of $(\underline{s}_2,\alpha_2)$ \\
\midrule
$0.50$ & $0.136$ & $72.8\%$ & $87.0\%$ \\
$0.60$ & $0.163$ & $72.8\%$ & $87.0\%$ \\
$0.70$ & $0.190$ & $72.9\%$ & $87.1\%$ \\
$0.80$ & $0.217$ & $72.9\%$ & $87.1\%$ \\
$0.90$ & $0.244$ & $72.9\%$ & $87.1\%$ \\
$0.95$ & $0.257$ & $72.9\%$ & $87.1\%$ \\
\bottomrule
\end{tabular}
\begin{flushleft}
\footnotesize
\textit{Notes.} The table uses $B(s)=W(s)=s^{1.2}$,
$\pi_2=0.35$, $\phi=0.5$, $\gamma=0.3$, $\delta=0.5$, and
$\lambda R=3.0$. For each value of $\alpha_2$, the cutoff
$\bar{\bar{s}}_2$ solves \cref{eq:monotone_condition_power} at equality.
\end{flushleft}
\end{table}

Across these capability levels, the admissible terminal domain covers
about three-quarters of the full below-benchmark interval and about
seven-eighths of the smooth terminal domain. Thus the single-crossing
condition does not restrict attention to a narrow slice near the AI
frontier; it removes the lowest-skill region where the curvature term is
largest, while leaving a broad upper band of below-benchmark workers.

\medskip

\section{Coverage of the Period-1 Skill Domains}
\label{app:period1_domain_size}

This appendix quantifies the period-1 skill restrictions used in
\cref{sec:period1_duopoly}. The main text restricts attention to
initial skills \(s_1\in D_1\) so that every continuation state reachable
under feasible period-1 engagement lies in the terminal domain \(D_2\).
It then further studies the regions \(s_1<s_1^D\) and \(s_1>s_1^F\),
where terminal engagement is fixed along the whole period-1 skill
trajectory. The purpose of these restrictions is to avoid terminal-policy
kinks in the period-1 objective, not to generate the mobility mechanism.

We first record the explicit cutoffs. The terminal-period equilibrium
characterization applies on \(D_2=(\bar{\bar{s}}_2,\alpha_2)\). Because
\(g(s_1,h_1;A_1)\) is increasing in \(h_1\), and because
\cref{ass:skill} implies \(g(s_1,h_1;A_1)<\alpha_1\le\alpha_2\) for every
\(h_1\in[0,1]\), every reachable continuation state lies in \(D_2\) as
soon as \(g(s_1,0;A_1)>\bar{\bar{s}}_2\). Define
\begin{align*}
\bar{\bar{s}}_1
\triangleq
\frac{\bar{\bar{s}}_2+\gamma(1-\pi_1)\alpha_1}
{1+\gamma(1-\pi_1)}
\qquad\text{and}\qquad
D_1
\triangleq
\bigl(\max\{\underline{s},\bar{\bar{s}}_1\},\alpha_1\bigr).
\end{align*}
For \(s_1\in D_1\) we then have \(g(s_1,h_1;A_1)\in D_2\) for every
\(h_1\in[0,1]\); only the reachable subset of \(D_2\) is used, because
\cref{ass:skill} implies \(g(s_1,h_1;A_1)<\alpha_1\le\alpha_2\), so the
portion of the terminal domain above \(\alpha_1\) need not be reachable
from period~1.

The two fixed-terminal thresholds pull the domain-restricted terminal
cutoffs \(\hat{s}_2^D\) and \(\overline{s}_2^D\) of
\cref{cor:terminal_monotone} back to period~1:
\begin{align}\label{eq:s1_regime_thresholds}
s_1^D
\triangleq
\frac{\hat{s}_2^D-\phi\pi_1(1-\pi_1)\alpha_1}
{1-\phi\pi_1(1-\pi_1)},
\qquad
s_1^F
\triangleq
\frac{\overline{s}_2^D+\gamma(1-\pi_1)\alpha_1}
{1+\gamma(1-\pi_1)}.
\end{align}
By construction, \(g(s_1^D,1;A_1)=\hat{s}_2^D\) and
\(g(s_1^F,0;A_1)=\overline{s}_2^D\). Hence, if \(s_1<s_1^D\), then even
full period-1 engagement leaves the worker below the terminal
no-engagement cutoff: \(g(s_1,h_1;A_1)\le g(s_1,1;A_1)<\hat{s}_2^D\) for
all \(h_1\in[0,1]\), so \(h_2^*(g(s_1,h_1;A_1))=0\) for all
\(h_1\in[0,1]\). Similarly, if \(s_1>s_1^F\), then even zero period-1
engagement puts the worker above the terminal full-engagement cutoff:
\(g(s_1,h_1;A_1)\ge g(s_1,0;A_1)>\overline{s}_2^D\) for all
\(h_1\in[0,1]\), so \(h_2^*(g(s_1,h_1;A_1))=1\) for all
\(h_1\in[0,1]\). After intersection with \(D_1\), the
no-terminal-engagement region is nonempty only if
\(s_1^D>\max\{\underline{s},\bar{\bar{s}}_1\}\) and the
full-terminal-engagement region only if \(s_1^F<\alpha_1\); if either
fails, the corresponding region is vacuous rather than inconsistent.

We measure the size of the period-1 terminal-domain restriction by
\[
\rho_1
\equiv
\frac{|D_1|}{\alpha_1-\underline{s}}
=
\frac{\alpha_1-\max\{\underline{s},\bar{\bar{s}}_1\}}
{\alpha_1-\underline{s}},
\]
whenever \(D_1\) is nonempty. Thus \(\rho_1\) is the share of the feasible
period-1 skill interval \((\underline{s},\alpha_1)\) covered by \(D_1\).

We also measure the size of the two fixed-terminal regions within \(D_1\):
\[
\rho_1^D
\equiv
\frac{|D_1\cap(-\infty,s_1^D)|}{|D_1|},
\qquad
\rho_1^F
\equiv
\frac{|D_1\cap(s_1^F,\infty)|}{|D_1|}.
\]
The first is the share of \(D_1\) for which every reachable terminal state
falls in the no-terminal-engagement region; the second is the share for
which every reachable terminal state falls in the full-terminal-engagement
region. The remaining share,
\(1-\rho_1^D-\rho_1^F\), is the part of \(D_1\) where reachable continuation
states may cross the terminal interior-engagement band.

We compute these shares for the power-wage case \(B(s)=W(s)=s^b\). The
grid is
\[
b\in\{1.03,1.05,1.08,1.12\},\quad
\alpha_1\in\{0.55,0.65,0.75\},\quad
\alpha_2\in\{0.70,0.82,0.90\},
\]
\[
\pi_2\in\{0.25,0.35,0.45,0.55\},\quad
\pi_1\in\{0.45,0.55,0.65,0.75\},
\]
\[
\phi\in\{0.5,1.2,2.0\},\quad
\gamma\in\{0.1,0.2,0.3\},\quad
\delta\in\{0.05,0.10,0.20,0.50\},
\]
\[
\lambda R\in\{3,5,10,20\},\quad
\eta\in\{0.2,0.5,1,2,5,10\},
\qquad \beta=0.95.
\]
We retain parameter vectors satisfying the model assumptions, the AI path
restrictions \(\alpha_2\ge\alpha_1\) and \(\pi_2\le\pi_1\), the terminal
upper-tail condition in \cref{eq:upper_tail_curvature_condition}, and the
monotone-terminal-policy condition \(\lambda R\pi_2\ge b\alpha_2^{b-1}\).

\begin{table}[ht]
\centering
\caption{Size of the period-1 domain \(D_1\)}
\label{tab:period1_D1_size}
\begin{tabular}{lccccc}
\toprule
Grid & Nonempty \(D_1\) & Mean \(\rho_1\) & 10th pct. & Median & 90th pct. \\
\midrule
Power-wage grid & \(95.3\%\) & \(72.9\%\) & \(36.0\%\) & \(81.0\%\) & \(95.5\%\) \\
\bottomrule
\end{tabular}
\begin{flushleft}
\footnotesize
\textit{Notes.} \(\rho_1=|D_1|/(\alpha_1-\underline{s})\) is the share of
the feasible period-1 skill interval covered by \(D_1\). The table reports
statistics across admissible grid points.
\end{flushleft}
\end{table}

\begin{table}[ht]
\centering
\caption{Coverage of the fixed-terminal regions within \(D_1\)}
\label{tab:period1_corner_regions}
\begin{tabular}{lcccc}
\toprule
Sorting friction \(\eta\)
& Mean \(\rho_1^D\)
& Mean \(\rho_1^F\)
& Mean \(\rho_1^D+\rho_1^F\)
& Median \(\rho_1^D+\rho_1^F\) \\
\midrule
\(0.2\)  & \(3.6\%\)  & \(92.1\%\) & \(95.7\%\) & \(100.0\%\) \\
\(0.5\)  & \(23.9\%\) & \(69.9\%\) & \(93.8\%\) & \(100.0\%\) \\
\(1.0\)  & \(46.9\%\) & \(50.2\%\) & \(97.1\%\) & \(100.0\%\) \\
\(2.0\)  & \(73.8\%\) & \(24.6\%\) & \(98.4\%\) & \(100.0\%\) \\
\(5.0\)  & \(97.0\%\) & \(2.8\%\)  & \(99.8\%\) & \(100.0\%\) \\
\(10.0\) & \(100.0\%\) & \(0.0\%\) & \(100.0\%\) & \(100.0\%\) \\
\bottomrule
\end{tabular}
\begin{flushleft}
\footnotesize
\textit{Notes.} The entries are average shares within \(D_1\), averaged
over admissible grid points with \(\eta\) held at the row value. The no-terminal
region is \(D_1\cap(-\infty,s_1^D)\); the full-terminal region is
\(D_1\cap(s_1^F,\infty)\). The fixed-terminal coverage is the share of
\(D_1\) on which the terminal policy is constant along every feasible
period-1 skill trajectory.
\end{flushleft}
\end{table}

The period-1 restrictions leave a sizable skill
domain. In the grid, \(D_1\) is nonempty in nearly all admissible
calibrations and covers a median of \(81.0\%\) of the feasible period-1
skill range. Within \(D_1\), the two fixed-terminal regions together cover
at least \(93.8\%\) of the domain on average for every value of \(\eta\)
reported in the table. Thus the restrictions mainly exclude the small set
of initial skills whose reachable continuation states cross a terminal
engagement cutoff; they do not confine the analysis to a narrow slice of
workers.

\medskip

\section{Robustness of the Reliability Pattern to Alternative Learning Profiles}
\label{app:learning_profile_robustness}

The reliability comparative static in \cref{sec:ai_dimensions_duopoly}
is the result most directly tied to the learning profile in the skill
transition. The baseline model assumes that learning from engagement is
proportional to $\pi(1-\pi)$. This appendix repeats the numerical
exercise in \cref{fig:h_star_vs_pi} under alternative learning profiles.

We replace the learning component in the skill transition with
$\phi h_t\ell(\pi_t)(\alpha_t-s_t)$, so that for $s_t<\alpha_t$,
\[
g_\ell(s_t,h_t;A_t)
=
s_t+
\left[
\phi h_t\ell(\pi_t)
-
\gamma(1-h_t)(1-\pi_t)
\right](\alpha_t-s_t).
\]
The terminal-period marginal engagement gain at a symmetric profile is
then governed by
\[
F_{2,\ell}(h;s_2,\pi_2)
=
\frac{\beta}{2\eta}
\left[
\gamma+\frac{\phi\ell(\pi_2)}{1-\pi_2}
\right]
B'\!\left(g_\ell(s_2,h;A_2)\right)
M(s_2,h;A_2),
\]
which reduces to the baseline expression when
$\ell(\pi)=\pi(1-\pi)$. For each learning profile and each value of
$\pi_2$, we solve the terminal two-firm game directly. When the
equilibrium is interior, it solves
\[
F_{2,\ell}(h;s_2,\pi_2)=\lambda R\delta.
\]

The numerical exercise uses the same parameters as \cref{fig:h_star_vs_pi}:
\[
B(s)=W(s)=s^{1.2},\quad s_2=0.30,\quad \alpha_2=0.90,\quad
\phi=0.5,\quad \gamma=0.3,\quad \delta=0.5,
\]
\[
\lambda R=3.0,\quad \beta=0.95,\quad \eta=0.21.
\]
To isolate the shape of the learning profile from the scale of learning,
each profile below is normalized so that
\[
\max_{\pi\in[0,1]}\ell(\pi)=\frac14,
\]
the maximum of the baseline product $\pi(1-\pi)$. Hence all profiles
satisfy the analogous no-leapfrogging restriction
$\phi\ell(\pi)<1$.

\begin{table}[ht]
\centering
\caption{Reliability pattern under alternative learning profiles}
\label{tab:learning_profile_robustness}
\begin{tabular}{lccc}
\toprule
Learning profile $\ell(\pi)$
& Peak $\pi_2$
& Peak $h_2^*$
& Positive-engagement range \\
\midrule
$\pi(1-\pi)$
& $0.500$ & $0.674$ & $(0.132,0.793)$ \\
$\pi^{1/2}(1-\pi)$
& $0.253$ & $0.766$ & $(0.020,0.634)$ \\
$\pi^{3/2}(1-\pi)$
& $0.753$ & $0.832$ & $(0.289,0.935)$ \\
$\pi(1-\pi)^2$
& $0.296$ & $0.753$ & $(0.069,0.551)$ \\
$\pi^2(1-\pi)^2$
& $0.500$ & $0.674$ & $(0.250,0.685)$ \\
$(0.05+\pi)(1-\pi)$
& $0.450$ & $0.678$ & $(0.069,0.770)$ \\
\midrule
$\pi^2(1-\pi)$
& $0.714$ & $1.000$ & $(0.419,0.950)$ \\
$\pi$
& $0.802$ & $1.000$ & $(0.725,0.950)$ \\
$1-\pi$
& $0.001$ & $0.972$ & $(0.001,0.355)$ \\
$0.05+\pi(1-\pi)$
& $0.917$ & $1.000$ & $(0.080,0.950)$ \\
\bottomrule
\end{tabular}
\begin{flushleft}
\footnotesize
\textit{Notes.} Each profile is normalized to have maximum value $1/4$.
The table reports the peak of the equilibrium engagement curve
$h_2^*(s_2;\pi_2)$ and the range of $\pi_2$ values over which engagement is
positive, computed on a grid $\pi_2\in[0.001,0.950]$. The first six rows
are exposure--practice profiles that keep learning limited when AI failures
are either very rare or very frequent; all produce a hump-shaped engagement
curve. The last four rows are extreme profiles. When learning remains strong at
high failure rates, as in $\pi^2(1-\pi)$, $\pi$, or $0.05+\pi(1-\pi)$,
the high-unreliability decline can disappear. When learning is driven only by
AI exposure, as in $1-\pi$, engagement is highest at very high reliability
rather than at intermediate reliability.
\end{flushleft}
\end{table}

The hump-shaped reliability pattern does not depend on
the literal product $\pi(1-\pi)$. It persists across a range of
exposure--practice learning profiles, including asymmetric profiles that
shift the peak toward higher or lower unreliability. What matters is the
economic structure: a small increase in unreliability must initially raise the
skill built by engagement, while very frequent failures must eventually make
the current operating margin too weak to justify engagement.

The extreme profiles also clarify the boundary of the result. If learning is
purely failure-practice based, or if learning remains available even when AI
almost never functions, then the skill-value channel can remain strong at high
$\pi_2$, and the high-unreliability decline may disappear. If learning is
purely AI-exposure based, engagement is strongest when AI is most reliable,
and the intermediate-reliability hump is lost. Thus the nonmonotone
reliability result is robust to the exact functional form of learning, but it
does rely on the exposure--practice complementarity that motivates the skill
transition.

\medskip

\section{Robustness of the Timing Reversal to Alternative Learning Profiles}
\label{app:timing_reversal_learning}

\cref{ex:reversals} shows that mobility can reverse the timing of
engagement: the worker receives no engagement in period~1 but full
engagement in period~2. We show that this timing reversal is robust to the choice of learning profile, and does not hinge on the exact product $\pi(1-\pi)$.

We replace the learning component in the skill transition with
$\phi h_t\ell(\pi_t)(\alpha_t-s_t)$, so that for $s_t<\alpha_t$,
\[
g_\ell(s_t,h_t;A_t)
=
s_t+
\left[
\phi h_t\ell(\pi_t)
-
\gamma(1-h_t)(1-\pi_t)
\right](\alpha_t-s_t).
\]
The erosion term is unchanged. To isolate the shape of the learning
profile from the scale of learning, each alternative profile is normalized
so that $\max_{\pi\in[0,1]}\ell(\pi)=1/4$, the maximum of the baseline
profile $\pi(1-\pi)$. With the parameters in \cref{ex:reversals},
this also preserves no-leapfrogging, since $\phi\max_\pi\ell(\pi)=0.30<1$.

For each profile, we directly solve the two-period mobility game using the
same primitive parameters as in \cref{ex:reversals}:
\[
B(s)=W(s)=s^{1.05},\quad
\alpha_1=0.65,\quad
\alpha_2=0.82,\quad
\pi_1=0.45,\quad
\pi_2=0.35,
\]
\[
\phi=1.20,\quad
\gamma=0.20,\quad
\beta=0.95,\quad
\delta=0.05,\quad
\lambda R=1.5,\quad
\eta=2.0,\quad
s_1=0.371.
\]
The table reports the symmetric period-1 equilibrium and the terminal
equilibrium reached from the resulting continuation skill.

\begin{table}[ht]
\centering
\caption{Timing reversal under alternative learning profiles}
\label{tab:timing_reversal_learning}
\begin{tabular}{lccccc}
\toprule
Learning profile $\ell(\pi)$
& $\ell(\pi_1)$
& $\ell(\pi_2)$
& $h_1^*$
& $h_2^*$
& Timing reversal \\
\midrule
$\pi(1-\pi)$
& $0.248$ & $0.227$ & $0$ & $1$ & Yes \\
$\pi^{1/2}(1-\pi)$
& $0.240$ & $0.250$ & $0$ & $1$ & Yes \\
$\pi^{3/2}(1-\pi)$
& $0.223$ & $0.181$ & $0$ & $1$ & Yes \\
$\pi(1-\pi)^2$
& $0.230$ & $0.250$ & $0$ & $1$ & Yes \\
$\pi^2(1-\pi)^2$
& $0.245$ & $0.207$ & $0$ & $1$ & Yes \\
$(0.05+\pi)(1-\pi)$
& $0.249$ & $0.236$ & $0$ & $1$ & Yes \\
$0.05+\pi(1-\pi)$
& $0.248$ & $0.231$ & $0$ & $1$ & Yes \\
\midrule
$\pi^2(1-\pi)$
& $0.188$ & $0.134$ & $0$ & $0$ & No \\
$\pi$
& $0.113$ & $0.088$ & $0$ & $0$ & No \\
$1-\pi$
& $0.138$ & $0.163$ & $0$ & $1$ & Yes \\
\bottomrule
\end{tabular}
\begin{flushleft}
\footnotesize
\textit{Notes.} Each nonbaseline profile is normalized so that
$\max_{\pi\in[0,1]}\ell(\pi)=1/4$. In every row, the symmetric period-1
equilibrium is $h_1^*=0$, so the worker enters period~2 with
$s_2=g_\ell(s_1,0;A_1)=0.340$. The terminal equilibrium $h_2^*$ is then
computed at this continuation skill. A timing reversal occurs when
$h_1^*=0$ and $h_2^*=1$.
\end{flushleft}
\end{table}

The timing reversal is not tied to the literal
product $\pi(1-\pi)$. The period-1 part of the reversal is especially
robust: across all profiles in the table, the symmetric period-1
equilibrium remains $h_1^*=0$. The terminal part depends on whether
engagement builds enough skill at the terminal failure probability
$\pi_2=0.35$. Balanced exposure--practice profiles preserve full terminal
engagement and hence preserve the timing reversal. The reversal disappears
in extreme profiles, such as normalized $\ell(\pi)=\pi$ or
$\ell(\pi)=\pi^2(1-\pi)$, that assign too little learning to the terminal
reliability level. Thus the example relies on terminal skill-building being
strong enough to attract workers, but not on the exact baseline product
form.

\medskip

\section{Robustness to a Pay Bonus}
\label{app:bonus}

In the baseline model, firms attract workers only through the skill trajectory their jobs build, because the wage schedule $W(\cdot)$ is set by the market and is the same at both firms. Firms might also compete on pay. We show that adding a simple cash instrument leaves our main conclusions intact. In each period $t\in\{1,2\}$, each firm $j$ now makes one extra choice alongside its engagement $h_t^j\in[0,1]$: a binary \emph{bonus} decision $z_t^j\in\{0,1\}$: whether to pay a bonus $\omega\ge 0$ on top of the market wage. The bonus is pure cash and does not affect how skill evolves, so the skill dynamics of \cref{sec:model} are unchanged. It matters in only two ways: it makes the firm more attractive to workers, and it lowers the firm's per-worker margin by $\omega$. Setting $\omega=0$ returns the model of \cref{sec:competition}.

In the terminal period, a worker's value from joining firm $j$ becomes
\begin{equation*}
  V_2^j\!\left(s_2,h_2^j,z_2^j\right)
  = W(s_2) + \omega\,z_2^j + \beta\,B\!\left(g(s_2,h_2^j;A_2)\right),
\end{equation*}
and, because the common wage cancels when the worker compares the two firms, firm $j$'s share is
\begin{equation*}
  \sigma_2^j
  = \left[\,1+\exp\!\Big(\big\{[\omega z_2^{-j}+\beta B(g(s_2,h_2^{-j};A_2))]
     -[\omega z_2^{j}+\beta B(g(s_2,h_2^{j};A_2))]\big\}/\eta\Big)\right]^{-1}.
\end{equation*}
Paying the bonus lowers the margin to $M(s_2,h_2^j;A_2)-\omega z_2^j$, so the firm's terminal payoff is $\Pi_2^j=\sigma_2^j\,[\,M(s_2,h_2^j;A_2)-\omega z_2^j\,]$. In the first period, let \(\bar V_2(\cdot,A_2)\) and \(\bar\Pi_2(\cdot,A_2)\) denote the worker value and common per-firm profit in the selected terminal-period equilibrium, which already reflect the terminal bonus decision. The worker's first-period value is $V_1^j
=
W(s_1)+\omega z_1^j
+\beta\,\bar V_2\!\left(g(s_1,h_1^j;A_1),A_2\right)$, the shares $\sigma_1^j$ follow the same logit, and firm $j$ chooses $(h_1^j,z_1^j)$ to maximize
\begin{equation*}
  \Phi_1^j
  = \sigma_1^j\Big[\,M(s_1,h_1^j;A_1)-\omega z_1^j
      + \beta\,\bar\Pi_2\!\left(g(s_1,h_1^j;A_1),A_2\right)\Big]
  + \big(1-\sigma_1^j\big)\,\beta\,\bar\Pi_2\!\left(g(s_1,h_1^{-j};A_1),A_2\right).
\end{equation*}
The last term is the free-riding channel from
\cref{sec:period1_duopoly,sec:asymmetry_duopoly}: if the worker joins the
other firm, firm \(j\) earns nothing today but must still attract the worker
next period, against the skill the other firm built.

The numerical extension preserves the main qualitative patterns. The
period-1 comparison with the single-firm benchmark is unchanged in the
calibration in \cref{fig:bonus-robust}: first-period engagement as a
function of \(\eta\) is identical for \(\omega\in\{0,0.05,0.10\}\), so
the crossing point \(\eta^\ast\) is unaffected. The terminal comparative
statics in \cref{sec:ai_dimensions_duopoly} are also robust in the
numerical exercises: greater AI capability still raises engagement, and
the hump-shaped response to reliability remains for moderate bonuses.
The free-riding result in \cref{sec:asymmetry_duopoly} also persists:
the asymmetric equilibrium in \cref{fig:asym_eq}, in which one firm
builds skill and the other free-rides, remains an equilibrium even for a
bonus worth about \(44\%\) of the relevant per-worker margin; in that
equilibrium neither firm pays the bonus.

\begin{figure}[t]
  \centering
  \includegraphics[width=0.62\textwidth]{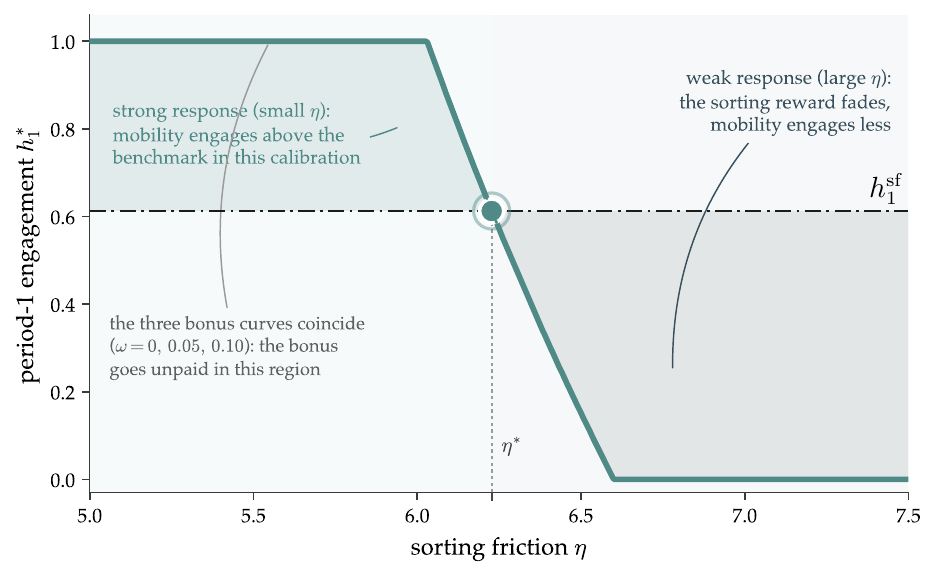}
\caption{The bonus leaves the engagement margin unchanged. First-period
mobility engagement \(h_1^\ast\) as a function of the worker-response
friction \(\eta\), for bonus levels \(\omega\in\{0,0.05,0.10\}\); the
three curves coincide because the equilibrium bonus is not paid in this
region. Engagement crosses the single-firm benchmark \(h_1^{\mathrm{sf}}\) (constant
in \(\eta\)) at \(\eta^\ast\): when workers respond weakly (large
\(\eta\)) mobility engages below the benchmark, as
\cref{cor:mobility_vs_monopoly} establishes; when they respond strongly
(small \(\eta\)) it engages above the benchmark in this calibration, a
numerical instance of the direction the analysis does not prove. The
crossing is unaffected by the bonus. Parameter values:
\(B(s)=W(s)=s^{1.2}\), \(\alpha_1=0.5\), \(\alpha_2=0.6\),
\(\pi_1=0.45\), \(\pi_2=0.35\), \(\phi=0.5\), \(\gamma=0.3\),
\(\delta=0.05\), \(\lambda R=3.65\), \(\beta=0.95\), \(s_1=0.25\), giving
\(h_1^{\mathrm{sf}}\approx0.61\) and \(\eta^\ast\approx6.23\).}
  \label{fig:bonus-robust}
\end{figure}

Whether the bonus is worth paying depends on a single tension. A firm that offers the bonus must pay it to every worker it employs, yet it only attracts the few extra workers it tips its way. When workers respond weakly to what firms offer (large $\eta$), too few are tipped to justify paying everyone, so firms leave the bonus off; when workers respond strongly (small $\eta$), even a small bonus draws in many, so firms pay it. Free-riding and specialization arise only when workers respond weakly, exactly when the bonus goes unused. The two levers therefore operate in separate regions and never interact. Consistent with this separation, our numerical search finds asymmetric
outcomes only through engagement-based free-riding, as in
\cref{sec:asymmetry_duopoly}; in those equilibria both firms make the same bonus choice. We do not find equilibria in which one firm keeps workers with cash while the other builds skill.

When workers respond strongly enough for the bonus to be used, both firms pay it. Because both firms pay, neither gains a relative sorting advantage, so the bonus is largely passed through to workers as extra compensation. In the numerical exercises, the bonus can substitute for some costly engagement of below-benchmark workers, but it does not alter
the core mechanisms: engagement remains concentrated where worker sorting is most valuable, AI capability and reliability have the same qualitative effects, and the free-riding equilibrium persists. Thus allowing firms to vary pay changes the division between cash and engagement in some
regions, but it does not overturn the paper's main engagement logic.

\end{document}